%% file: TR.tex
\documentclass{article} 
\usepackage{graphicx}
\usepackage[english]{babel}
\usepackage[latin1]{inputenc}
\usepackage{amsmath,amssymb,amsthm}
\usepackage{mathrsfs}
\usepackage{color}
\usepackage{lineno}
\usepackage{epic}
\usepackage{eepic}
\usepackage{alltt}

\newtheorem{theorem}{Theorem}[section]
\newtheorem{cor}[theorem]{Corollary}
\newtheorem{lemma}[theorem]{Lemma}
\newtheorem{proposition}[theorem]{Proposition}

\newtheorem{defn}[theorem]{Definition}
\newtheorem{example}[theorem]{Example}
\newtheorem{claim}[theorem]{Claim}
\newtheorem{remark}[theorem]{Remark}

\newcommand{\size}[1]{\parallel\!\! #1\!\!\parallel }
\newcommand{\Pol}{\mbox{\rm P}}
\newcommand{\NP}{\mbox{\rm NP}}
\newcommand{\LCFL}{\mbox{\rm LOGCFL}}

\newcommand{\U}{\mathcal{U}}
\newcommand{\anymaxref}{{\sc Max}}
\newcommand{\nextref}{{\sc Next}}
\newcommand{\topk}{{\sc Top-}\mbox{$\it K$}}

\newcommand{\PHI}{{\mathrm{\Phi}}}
\newcommand{\DB}{{\mbox{\rm DB}}}
\newcommand{\DBW}{{\mbox{\rm DB'}}}
\newcommand{\onDB}{{\mbox{\rm \tiny DB}}}
\newcommand{\vars}{\mathit{vars}}
\newcommand{\atoms}{\mathit{atoms}}

\newcommand{\ranks}{\mathit{wfs}}

\newcommand{\onDBW}{{\mbox{\rm \tiny DB'}}}

\newcommand{\onDBWa}{{\mbox{\rm \tiny DB'}_a}}
\newcommand{\Q}{{\mathrm{\Phi}}}
\newcommand{\rel}{\mathit{rel}}
\newcommand{\algMax}{\mbox{\tt Compute-Max}}
\newcommand{\dom}{{\it dom}}
\newcommand{\homEquiv}{\equiv_{\tt h}}
\newcommand{\atom}{\mathit{atom}}
\newcommand{\newsep}{\mathit{sep}}

\newcommand{\HG}{{\cal H}}
\newcommand{\nodes}{\mathit{nodes}}
\newcommand{\edges}{\mathit{edges}}
\newcommand{\JT}{\mathit{JT}}

\newcommand{\W}{\mathcal{W}}
\newcommand{\V}{\mathcal{V}}

\newcommand{\cJ}{\mathcal{I}}

\newcommand{\bp}{p}
\newcommand{\bq}{q}
\newcommand{\bs}{s}
\newcommand{\bl}{l}

\newcommand{\relp}{Rel_{\!\mbox{\tiny T'}}}
\newcommand{\ret}{R_{\!\mbox{\tiny T}}}
\newcommand{\retp}{R_{\!\mbox{\tiny T'}}}
\newcommand{\iet}{I_{\!\mbox{\tiny T}}}
\newcommand{\oet}{O_{\!\mbox{\tiny T}}}
\newcommand{\oetp}{O_{\!\mbox{\tiny T'}}}
\newcommand{\ietp}{I_{\!\mbox{\tiny T'}}}
\newcommand{\sch}{\it S\!cope_{\mbox{\tiny T}}}
\newcommand{\schp}{\it S\!cope_{{T'}}}
\newcommand{\join}{\bowtie}

\newcommand{\ACQ}{\mbox{\rm ACSP}}
\newcommand{\val}{{\it val}}

\newcommand{\overbar}[1]{\mkern 9mu\overline{\mkern-9mu#1\mkern-4mu}\mkern 4mu}
\newcommand{\bJT}{\overbar{\mathit{JT}}}

\newenvironment{myitemize}{\begin{list}{$\bullet$}{\setlength{\leftmargin}{13pt}
\setlength{\itemindent}{0.6\labelwidth}}} {\end{list}}

\newenvironment{myitemize2}{\begin{list}{$\bullet$}{\setlength{\leftmargin}{13pt}
\setlength{\itemindent}{0.2\labelwidth}}} {\end{list}}

\newcommand{\tuple}[1]{\langle#1\rangle}

\begin{document}

\title{Tree Projections and Constraint Optimization Problems:  Fixed-Parameter Tractability and Parallel Algorithms}

\author{Georg Gottlob$^1$, Gianluigi Greco$^2$ and Francesco Scarcello$^3$\\
\ \\
\small $^1$Department of Computer Science, University of Oxford, UK\\
\small $^2$Department of Mathematics and Computer Science, University of Calabria, Italy\\
\small $^3$DIMES, University of Calabria, Italy\\
\small  {\tt gottlob@cs.ox.ac.uk, ggreco@mat.unical.it, scarcello@deis.unical.it }}

\date{}

\maketitle

\begin{abstract}
Tree projections provide a unifying framework to deal with most structural decomposition methods of constraint satisfaction problems~(CSPs).
Within this framework, a CSP instance is decomposed into a number of sub-problems, called views, whose solutions are either already available
or can be computed efficiently. The goal is to arrange portions of these views in a tree-like structure, called tree projection, which
determines an efficiently solvable CSP instance equivalent to the original one. However, deciding whether a tree projection exists is
$\NP$-hard. Solution methods have therefore been proposed in the literature that do not require a tree projection to be given, and that either
correctly decide whether the given CSP instance is satisfiable, or return that a tree projection actually does not exist. These approaches had
not been generalized so far to deal with CSP extensions tailored for optimization problems, where the goal is to compute a solution of maximum
value/minimum cost.
The paper fills the gap, by exhibiting a fixed-parameter polynomial-time algorithm that either disproves the existence of tree projections or
computes an optimal solution, with the parameter being the size of the expression of the objective function to be optimized over all possible
solutions (and not the size of the whole constraint formula, used in related works). Tractability results are also established for the problem
of returning the best $K$ solutions. Finally, parallel algorithms for such optimization problems are proposed and analyzed.

Given that the classes of acyclic hypergraphs, hypergraphs of bounded treewidth, and hypergraphs of bounded generalized hypertree width are all
covered as special cases of the tree projection framework, the results in this paper directly apply to these classes. These classes are
extensively considered in the CSP setting, as well as in conjunctive database query evaluation and optimization.
\end{abstract}


\raggedbottom

\section{Introduction}\label{sec:intro}

\subsection{Optimization in Constraint Satisfaction Problems}

Constraint satisfaction is a central topic of research in Artificial Intelligence, and has a wide spectrum of concrete applications ranging
from configuration to scheduling, plan design, temporal reasoning, and machine learning, just to name a few.

Formally, a constraint satisfaction problem (for short: CSP) instance is a triple $\mathcal{I}=\tuple {{\it Var},\U,\mathcal{C}}$, where ${\it
Var}$ is a finite set of variables, $\U$ is a finite domain of values, and $\mathcal{C}=\{C_1,C_2,...,C_q\}$ is a finite set of constraints
(see, e.g.,~\cite{D03}). Each constraint $C_v$, with $v\in \{1,...,q\}$, is a pair $(S_v,r_v)$, where $S_v\subseteq {\it Var}$ is a set of
variables called the {\em constraint scope}, and $r_v$ is a set of assignments from variables in $S_v$ to values in $\U$ indicating the allowed
combinations of values for the variables in $S_v$.
A (partial) assignment from a set of variables $\W\subseteq {\it Var}$ to $\U$ is explicitly represented by the set of pairs of the form $X/u$,
where $u\in\U$ is the value to which $X\in \W$ is mapped.
%
An assignment $\theta$ \emph{satisfies} a constraint $C_v$ if its restriction to $S_v$, i.e., the set of pairs $X/u\in \theta$ such that $X\in
S_v$, occurs in $r_v$.
A \emph{solution} to $\mathcal{I}$ is a (total) assignment $\theta:{\it Var}\mapsto \U$ for which $q$ satisfying assignments $\theta_1\in
r_1,...,\theta_q\in r_q$ exist such that $\theta=\theta_1\cup...\cup \theta_q$. Therefore, a solution is a total assignment that satisfies all
the constraints in $\mathcal{I}$.

By \emph{solving a CSP instance} we usually just mean finding any arbitrary solution. However, when assignments are associated with weights
because of the semantics of the underlying application domain, we might instead be interested in the corresponding \emph{optimization problem}
of finding the solution of maximum or minimum weight (short: \anymaxref\ and {\sc Min} problems), whose modeling is possible in several
variants of the basic CSP framework, such as the \emph{valued} and \emph{semiring-based CSPs}~\cite{BFMRSV96}. Moreover, we might be interested
in the \topk\ problem of enumerating the best (w.r.t.~\anymaxref\ or {\sc Min}) $K$ solutions in form of a ranked list (see, e.g.,
\cite{FD10,BDGM12}),\footnote{Related results on graphical models, conjunctive query evaluation, and computing homomorphisms on relational
structures are transparently recalled hereinafter in the context of constraint satisfaction.} or even in the \nextref\ problem of computing the
next solution (w.r.t.~such an ordering) following one that is at hand~\cite{BRSVW10}.

CSP instances, as well as their extensions tailored to model optimization problems, are computationally intractable. Indeed, even just deciding
whether a given instance admits a solution is a well-known $\NP$-hard problem, which calls for practically effective algorithms and heuristics,
and for the identification of specific subclasses, called ``islands of tractability'', over which the problem can be solved efficiently. In
this paper, we consider the latter perspective to attack CSP instances, by looking at \emph{structural properties} of constraint scopes.

\subsection{Structural Decomposition Methods and Tree Projections}

The avenue of research looking for islands of tractability based on structural properties
originated from the observation that constraint satisfaction is tractable on \emph{acyclic} instances (cf.~\cite{MONTANARI197495,Y81}), i.e.,
on instances whose associated hypergraph (whose hyperedges correspond one-to-one to the sets of variables in the given constraints) is
acyclic.\footnote{There are different notions of hypergraph acyclicity. In the paper, we consider $\alpha$-acyclicity, which is the most
liberal one~\cite{F83}.}

Motivated by this result, \emph{structural decomposition methods} have been proposed in the literature as approaches to transform any given
cyclic CSP into an equivalent acyclic one by organizing its constraints or variables into a polynomial number of clusters and by arranging
these clusters as a tree, called \emph{decomposition tree}. The satisfiability of the original instance can be then checked by exploiting this
tree, with a cost that is exponential in the cardinality of the largest cluster, also called {\em width} of the decomposition, and polynomial
if the width is bounded by a constant (see~\cite{GGLS16} and the references therein).
Similarly, by exploiting this tree, solutions can be computed even to CSP extensions tailored for optimization problems, again with a cost that
is polynomial over bounded-width instances. For instance, we know that (in certain natural optimization settings) \anymaxref\ is feasible in
polynomial time over instances whose underlying hypergraphs are acyclic~\cite{KS06}, have bounded \emph{treewidth}~\cite{FD10}, or have bounded
\emph{hypertree width}~\cite{GGS09,GS11}.

Despite their different technical definitions, there is a simple framework encompassing all structural decomposition methods,\footnote{The
notion of \emph{submodular} width~\cite{M13} does not fit this framework, as it is not purely structural.} which is the framework of the
\emph{tree projections}~\cite{GS84}.
The basic idea of these methods is indeed to ``cover'' all the given constraints via a polynomial number of clusters of variables and to
arrange these clusters as a tree, in such a way that the \emph{connectedness condition} holds, i.e., for each variable $X\in {\it Var}$, the
subgraph induced by the clusters containing $X$ is a tree. In particular, any cluster identifies a subproblem of the original instance, and it
is required that all solutions to this subproblem can either be computed efficiently, or are already available (e.g., from previous
computations).
A tree built from the available clusters and covering all constraints is called a \emph{tree projection}~\cite{GS84,SS93,GMS09,GrecoSIC17}. In
particular, whenever such clusters are required to satisfy additional conditions, tree projections reduce to specific decomposition methods.
For instance, if we consider candidate clusters given by all subproblems over $k+1$ variables at most (resp., over any set of variables
contained in the union of $k$ constraints at most), then tree projections correspond to \emph{tree decompositions}~\cite{RS86,D03} (resp.,
\emph{generalized hypertree decompositions}~\cite{GLS02,GMS09}), and $k$ is their associated {width}.

Deciding whether a tree projection exists is $\NP$-hard in general, that is, when a set of arbitrary clusters/subproblems is
given~\cite{GMS09}. Moreover, the problem remains intractable is some specific settings, such as (bounded width) generalized hypertree
decompositions~\cite{GMS09}.
Therefore, designing tractable algorithms within the framework of tree projections is not an easy task. Ideally we would like to efficiently
solve the instances without requiring that a tree projection be explicitly computed (or provided as part of the input).
For standard CSP instances, algorithms of this kind have already been exhibited~\cite{SS93,GS84,CD05,GS10}. These algorithms are based on
\emph{enforcing pairwise-consistency}~\cite{BFMY83}, also known in the CSP community as {\em relational arc consistency} (or arc consistency on
the dual graph)~\cite{D03}, {\em 2-wise consistency}~\cite{G86}, and $R(*,2)C$~\cite{KWRCB10}. Note that these algorithms are mostly used in
heuristics for constraint solving algorithms.
The idea is to repeatedly take---until a fixpoint is reached---any two constraints $C_i=(S_i,r_i)$ and $C_j=(S_j,r_j)$ and to remove from $r_i$
all assignments $\theta_i$ that cannot be extended over the variables in $S_j$, i.e., for which there is no assignment $\theta_j\in r_j$ such
that the restrictions of $\theta_i$ and $\theta_j$ over the variables in $S_i\cap S_j$ coincide.
Here, the crucial observation is that the order according to which pairs of constraints are processed is immaterial, so that this procedure is
equivalent to Yannakakis' algorithm~\cite{Y81}, which identifies a correct processing order based on the knowledge of a tree
projection.\footnote{The algorithm has been originally proposed for acyclic instances. For its application within the tree projection setting,
the reader is referred to~\cite{GS10}.}
Actually, it is even unnecessary to know that a tree projection exists at all, because any candidate solution can be certified in polynomial
time.
Indeed, these algorithms are designed in a way that, whenever some assignment is computed that is subsequently found not to be a solution, then
the (promised) existence of a tree projection is disproved. We define these algorithms computing certified solutions as \emph{promise-free},
with respect to the existence of a tree projection (cf.~\cite{CD05,GS10}). We note that, so far, this kind of solution approach has not been
generalized in the literature to deal with CSP extensions tailored for optimization problems.

\subsection{Contributions}

All previous algorithms proposed in the literature for computing the best CSP solutions in polynomial
time~\cite{FD10,GGS09,KDLD05,BMR97,TJ03,GG13,GS11,DBLP:journals/corr/JoglekarPR15,AboKhamis:2016:FQA:2902251.2902280}
(or, more generally, for optimizing functions in different application domains---see, e.g.,~\cite{LS90}) require the knowledge of some suitable
tree projection, which provides at each node a list of potentially good partial evaluations with their associated values to be propagated
within a dynamic programming scheme.
The main conceptual contribution of the present paper is to show that this knowledge is not necessary, since promise-free algorithms can be
exhibited in the tree projection framework even when dealing with optimization problems.

More formally, we consider a setting where the given CSP instance $\mathcal{I}$ is equipped with a valuation function $\mathcal{F}$ to be
maximized over the feasible solutions. The function is built from basic weight functions defined on subsets of variables occurring in
constraint scopes, combined via some binary operator $\oplus$.\footnote{In fact, our results are designed to hold in a more general setting
where different binary operators may be used together in the definition of more complex valuation functions. However, for the sake of
presentation, we shall mainly focus on a single operator, in the spirit of the (standard) valued and semiring-based CSP settings.} Moreover, we
assume that a set of subproblems $\V$ is given together with their respective solutions.
Then, within this setting,

\vspace{-2mm}
\begin{itemize}
\item[$\rhd$] We provide a {\em fixed-parameter polynomial-time} algorithm~\cite{down-fell-99} for {\anymaxref} that either computes a
    solution (if one exists) having the best weight according to $\mathcal{F}$, or says that no tree projection can be built by using the
    available subproblems in $\V$. In any case, the algorithm does not output any wrong answer, because the computed solutions are
    certified. More precisely, the algorithm runs in time $f(\kappa)\times n^c$, where the parameter $\kappa$ is the number of basic
    functions occurring in $\mathcal{F}$, $n$ is the size of the input, and $c$ is a fixed natural number. Thus, the running time has no
    exponential dependency on the input, but possibly on the fixed parameter $\kappa$.

\item[$\rhd$] We show that the  \topk\ problem of returning the best $K$ solutions over all  possible solutions is fixed-parameter
    tractable, too. As we may have an exponential number of solutions (w.r.t.~$n$), tractability means here having a promise-free algorithm
    that computes the desired output {\em with fixed-parameter polynomial delay}: The first solution is computed  in fixed-parameter
    polynomial-time, and any other solution is computed within fixed-parameter polynomial-time after the previous one.

\item[$\rhd$] Moreover, we complement the above research results, by studying the setting where a tree projection is given at hand. In this
    case, we show that the task of computing the best solutions over a set of output variables is not only feasible in polynomial time (as
    we already know from the literature pointed out above), but it is even possible to define parallel algorithms that can exploit the
    availability of machines with multiple processors.
\end{itemize}

Concerning our main technical contributions, we stress here that different kinds of fixed-parameter polynomial-time algorithms can be defined
for the problems of interests when varying the underlying parameter of interest. For instance, a trivial choice would be to consider the
overall number of constraints involved in the CSP at hand. In fact, our parameter is very often much smaller, so that our algorithms can be
useful in all those applications where the optimization function consists of few basic functions, while the number of constraints is large
(which makes infeasible computing any tree projection).

\paragraph{Organization} The rest of the paper is organized as follows.
Section~\ref{sec:preliminaries} illustrates some basic notions about CSPs and their structural properties. The formal framework for equipping
CSP instances with optimization functions is introduced in Section~\ref{sec:framework}.
Our fixed-parameter tractability results are illustrated in Section~\ref{sec:main}.
Parallel algorithms are presented in Section~\ref{sec:parallel}.
Relevant related works are discussed in Section~\ref{sec:related}, and concluding remarks are drawn in Section~\ref{sec:conclusion}.

\section{Preliminaries}\label{sec:preliminaries}

\paragraph{Logic-Based Modeling of Constraint Satisfaction} Let $\mathcal{I}=\tuple
{{\it Var},\U,\mathcal{C}}$ be a CSP instance, with $\mathcal{C}=\{(S_1,r_1),(S_2,r_2),...,(S_q,r_q)\}$.
Following~\cite{KV00}, we shall exploit throughout the paper the logic-based characterization of $\mathcal{I}$ as a pair $(\PHI,\DB)$, which
simplifies the illustration of structural tractability results.
In particular, $\PHI$ is the \emph{constraint formula} (associated with $\mathcal{I}$), i.e., a conjunction of atoms of the form $r_{1}({\bf
u_1})\wedge\cdots\wedge r_{q}({\bf u_q})$ where ${\bf u_v}$, for each $v\in \{1,...,q\}$, is obtained by listing all the variables in the scope
$S_v$. The set of variables in $\PHI$ is denoted by  $\vars(\PHI)$, while the set of atoms occurring in $\PHI$ is denoted by $\atoms(\PHI)$.
Moreover, $\DB$ is the \emph{constraint database}, i.e., a set of ground atoms encoding the allowed tuples of values for each constraint, built
as follows. For each constraint index $v\in \{1,...,q\}$ and for each assignment $\theta_v\in r_v$, $\DB$ contains the ground atom
$r_v(a_1,...,a_{|S_v|})$ where if $X_i$ is the $i$-th variable in the list ${\bf u_v}$, then $a_i=\theta_v(X_i)$ holds for each
$i\in\{1,...,|S_v|\}$. No further ground atom is in $\DB$.

In the following, for any set $\W$ of variables and any assignment $\theta$, $\theta[\W]$ denotes the partial assignment obtained by
restricting $\theta$ to the variables in $\W$. Therefore, a (total) substitution $\theta$ is a solution to $(\PHI,\DB)$ if $\theta[S_v]\in r_v$
holds for each $v\in \{1,...,q\}$. The set of all solutions to the CSP instance $(\PHI,\DB)$ is denoted by $\PHI^\onDB$. Moreover, for any set
$\W$ of variables, $\PHI^\onDB[\W]$ denotes the set $\{ \theta[\W] \mid \theta\in \PHI^\onDB \}$.

\medskip

\paragraph{Structural Properties of CSP Instances}
The structure of a constraint formula $\PHI$ is best represented by its associated hypergraph $\HG_\PHI=(N,H)$, where $N=\vars(\PHI)$, i.e.,
variables are viewed as nodes, and where $H=\{S_1,...,S_v\}$, i.e., for each atom in $\PHI$, $H$ contains a hyperedge including all its
variables.
For any hypergraph $\HG$, we denote the sets of its nodes and of its hyperedges by $\nodes(\HG)$ and $\edges(\HG)$, respectively.

A hypergraph $\HG$ is {\em acyclic} if it has a {\em join tree}~\cite{BG81}. A {\em join tree} $\JT$ of $\HG$ is a labeled tree $(V,E,\chi)$,
where for each vertex $v\in V$, it holds that $\chi(v)\in \edges(\HG)$, and where the following conditions are satisfied:
\begin{description}
  \item[{Covering Condition}:] $\forall h\in \edges(\HG)$, for some vertex $v$ of $\JT$, $h=\chi(v)$ holds;

  \item[{Connectedness Condition}:] for each pair of vertices $v_1,v_2$ in $V$ such that $\chi(v_1)\cap\chi(v_2)=\W\neq\emptyset$, $v_1$
      and $v_2$ are connected in $\JT$ (via edges from $E$) and $\W\subseteq\chi(v)$, for every vertex $v$ in the unique path linking $v_1$
      and $v_2$ in $\JT$.
\end{description}

Note that this definition is apparently more liberal than the traditional one (in~\cite{BG81}), where there is a one-to-one correspondence
between hyperedges and vertices of the join tree. We find it convenient to allow multiple occurrences of the same hyperedge in the $\chi$
labels of different vertices of $\JT$, but it is straightforward to show that a standard join tree may be obtained from $\JT$ by repeatedly
contracting edges of the form $\{v_1,v_2\}$, where $\chi(v_1)\subseteq \chi(v_2)$ (until such a one-to-one correspondence is met).

\medskip

\paragraph{Decomposition Methods}
Structural decomposition methods have been proposed in the literature in order to provide a measure of the degree of acyclicity of hypergraphs,
and in order to generalize positive computational results from acyclic hypergraphs to nearly-acyclic ones. Despite their different technical
definitions, there is a simple framework encompassing all known (purely structural) decomposition methods. The framework is based on the
concept of \emph{tree projection}~\cite{GS84}, which is recalled below.

\begin{defn}\em
For two hypergraphs $\HG_1$ and $\HG_2$, we say that $\HG_2$ \emph{covers} $\HG_1$, denoted by $\HG_1\leq \HG_2$, if each hyperedge of $\HG_1$
is contained in at least one hyperedge of $\HG_2$. Let $\HG_1\leq \HG_2$. Then, a \emph{tree projection} of $\HG_1$ with respect to $\HG_2$ is
an acyclic hypergraph $\HG_a$ such that $\HG_1\leq \HG_a \leq \HG_2$. Whenever such a hypergraph $\HG_a$ exists, we say that the pair
$(\HG_1,\HG_2)$ has a tree projection. \hfill $\Box$
\end{defn}

\begin{figure*}[t]
  \centering
  \includegraphics[width=\textwidth]{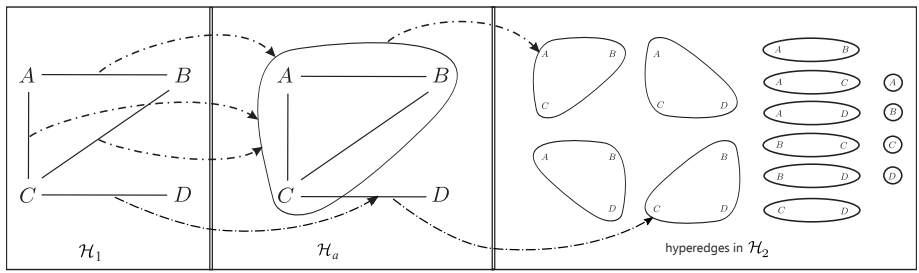}
  \caption{Structures discussed in Example~\ref{ex:treeDecomposition}.}\label{fig:Example-TW}
\end{figure*}

\begin{example}\em
Consider the hypergraph $\HG_1$ depicted in Figure~\ref{fig:Example-TW}, and the hypergraph $\HG_2$ whose hyperedges are listed on the right of
the same figure. Note that $\HG_1$ is (just) a graph and it contains a cycle over the nodes $A$, $B$, and $C$.

The acyclic hypergraph $\HG_a$ shown in the middle is a tree projection of $\HG_1$ w.r.t.~$\HG_2$. For instance, note that the cycle is
``absorbed'' by the hyperedge $\{A,B,C\}$, which is in its turn trivially contained in a hyperedge of $\HG_2$. \hfill $\lhd$
\end{example}

 Following~\cite{GS10}, tree projections can be used to solve any CSP instance $(\PHI,\DB)$ whenever we have (or we can build) an additional
 pair $(\V,\DBW)$ such that:
\begin{itemize}
\item  $\V$ is a set of atoms (hence, corresponding to a set of constraint scopes). Each atom in $\V$ clusters together the variables of a
    subproblem whose solutions are assumed to be available in the \emph{constraint database} $\DBW$ and that can be exploited in order
      to answer the original CSP instance $(\PHI,\DB)$.
 Atoms in $\V$ will be called \emph{views}, and $\V$ will be called \emph{view set}. It is required that, for each atom $q\in\atoms(\PHI)$,
     $\V$ contains a \emph{base view} $w_q$ with the same list of variables as $q$.

  \item $\DBW$ is a constraint database that satisfies the following conditions:
  \begin{itemize}
\item[(i)] $w_q^\onDBW \subseteq q^\onDB$ holds for each base view $w_q\in\V$; that is, base views should be at least as restrictive as
    atoms in the constraint formula;

\item[(ii)] $w^\onDBW\supseteq \PHI^\onDB[w]$ holds for each $w\in \V$; that is, any view cannot be more restrictive than the
    constraint formula, otherwise correct solutions may be deleted by performing operations involving such views.
\end{itemize}
Such a database $\DBW$ is said \emph{legal} for $\V$ w.r.t.~$\Q$ and $\DB$.
\end{itemize}

The pair $(\V,\DBW)$ is used as follows. Let $\HG_\V$ denote the view hypergraph precisely containing, for each view $w$ in $\V$, one hyperedge
over the variables in $w$.
We look for a \emph{sandwich formula of $\PHI$ w.r.t.~$\V$}, that is, a constraint formula $\PHI_a$ such that $\atoms(\PHI_a)$ includes all
base views and $\HG_{\PHI_a}$ is a tree projection of $\HG_{\PHI}$ w.r.t.~$\HG_\V$. By exploiting the sandwich formula $\PHI_a$, solving $\PHI$
can be reduced to answering an acyclic instance, hence to a task which is feasible in polynomial time.
Indeed, by projecting the assignments of any legal database $\DBW$ over the (portions of the) views used in $\PHI_a$, a novel database $\DBW_a$
can be obtained such that $\PHI_a^\onDBWa=\PHI^\onDB$~\cite{GS84}.


Most structural decomposition methods of constraint satisfaction problems can be viewed as special instances of this approach, where the
peculiarities of each method lead to different ways of building the additional view set $\V$, with its associated database $\DBW$. For
instance, the methods based on {\em generalized hypertree decompositions}~\cite{GLS02,GMS09} and \emph{tree decompositions}~\cite{RS86}, for a
constant \emph{width} $k$, fit into the framework as follows:
\begin{description}
  \item[$k$-width generalized hypertree decompositions:] The method uses a set $h\mbox{-}\V_k$ of views including, for each subformula
      $\PHI'$ of $\PHI$ with $\atoms(\PHI')\subseteq \atoms(\PHI)$ and $|\atoms(\PHI')|\leq k$, a view that is built over the set of all
      variables on which these atoms are defined (hence, base views are obtained for $k$=1) and whose assignments in the corresponding
      constraint database $h\mbox{-}\DBW_k$ are all solutions to $(\PHI',\DB)$.

  \item[$k$-width tree decompositions:] The method uses the set $\V_k$ of views consisting of the base views plus all the views that can be
      built over all possible sets of at most $k+1$ variables. In the associated constraint relations in $\DBW_k$, base views consist of
      the assignments in the corresponding atoms in $\DB$, whereas each of the remaining views contains all possible assignments that can
      be built over them, hence $|U|^{k+1}$ assignments at most, where $U$ is the size of the largest domain over the selected $k+1$
      variables.
\end{description}

\begin{example}\label{ex:treeDecomposition}\em
Consider a CSP instance $(\PHI,\DB)$ such that $\PHI=r_1(A,B)\wedge r_2(B,C)\wedge r_3(A,C) \wedge r_4(C,D)$  and where $r_i(0,0)$ and
$r_i(1,1)$ are the only two ground atoms in $\DB$, for each $i\in\{1,2,3,4\}$.
The constraint hypergraph associated with $\PHI$ is precisely the hypergraph $\HG_1$ illustrated in Figure~\ref{fig:Example-TW}. Since
$\HG_\PHI=\HG_1$ is not acyclic, our goal is to apply a structural decomposition method for transforming the original instance $\PHI$ into a
novel acyclic one $\PHI_a$ that ``covers'' all constraints in $\PHI$ and is equivalent to it.
To this end, let us consider the application on $(\PHI,\DB)$ of the {tree decomposition method} with $k$ being the associated width, resulting
in the pair $(\V_k,\DBW_k)$.
%
%
For instance, the base view $w_{r_4(C,D)}(C,D)$ is in $\V_k$ and its associated tuples in $\DBW_k$ are $w_{r_4(C,D)}(0,0)$ and
$w_{r_4(C,D)}(1,1)$. Moreover, for each natural number $k>1$, a view having the form $w_{A,B,C}(A,B,C)$ is in $\V_k$ and the associated tuples
in $\DBW_k$ are $w_{A,B,C}(\alpha,\beta,\gamma)$, with $\alpha,\beta,\gamma\in \{0,1\}$.
In particular, for $k=2$, the hypergraph $\HG_{\V_2}$ precisely coincides with the hypergraph $\HG_2$ (whose hyperedges are) illustrated in
Figure~\ref{fig:Example-TW}. Consider then the constraint formula

{\small
$$\PHI_a=w_{A,B,C}(A,B,C)\wedge w_{r_1(A,B)}(A,B)\wedge w_{r_2(B,C)}(B,C)\wedge w_{r_3(A,C)}(A,C) \wedge w_{r_4(C,D)}(C,D),$$
}

\vspace{-2mm}\noindent and note that $\HG_{\PHI_a}$ coincides with the acyclic hypergraph $\HG_a$, which is a tree projection\footnote{The fact
that $\HG_a$ is a tree projection of $\HG_\PHI$ w.r.t.~$\HG_{\V_2}$ witnesses that the treewidth of $\HG_\PHI$ is 2. In general, a tree
projection of $\HG_\PHI$ w.r.t.~$\HG_{\V_k}$ exists if and only if $\HG_\PHI$ has treewidth $k$ at most (see~\cite{GS10,GS13}).} of $\HG_\PHI$
w.r.t.~$\HG_{\V_2}$. Then, solving $(\PHI,\DB)$ is equivalent to solving $(\PHI_a,\DBW_a)$ where $\DBW_a$ is just the restriction of $\DBW_2$
over the atoms in $\atoms(\PHI_a)$.\hfill $\lhd$
\end{example}

\section{Valuation Functions and Basic Results}\label{sec:framework}

In this section, we illustrate a formal framework for equipping constraint formulas with valuation functions suited to express a variety of
optimization problems. Moreover, we introduce and analyze a notion of embedding as a way to represent and study the interactions between
constraint scopes and valuation functions.

In the following we assume that a domain $\U$ of values, a constraint formula $\PHI$, and a set $\mathbb{D}$ of weights totally ordered by a
relation~$\geq$ are given. Moreover, on the set $\mathbb{D}$, we define $\max$ and $\min$ as the operations returning any $\geq$-maximum and
the $\geq$-minimum weight, respectively, over a given set of weights.

\subsection{Formal Framework}

Let $\W\subseteq \vars(\PHI)$ be a set of variables. Then, a function $f$ associating each assignment $\theta: \W \mapsto \U$ with a weight
$f(\theta)\in \mathbb{D}$ is called a \emph{weight function} (for $\PHI$), and we denote by $\vars(f)$ the set $\W$ on which it is defined. If
an assignment $\theta': \W'\mapsto \U$ with $\W'\supseteq \vars(f)$ is given, then we write $f(\theta')$ as a shorthand for
$f(\theta'[\vars(f)])$.

\begin{defn}\label{def:valuationfunction}\em
Let $\oplus$ be a closed, commutative, and associative binary operator over $\mathbb{D}$ being, moreover, distributive over $\max$.
A \emph{valuation function} $\mathcal{F}$ (for $\Q$ over $\oplus$) is an expression of the form $f^1_{{\bf u}_1}\oplus\dots\oplus f^m_{{\bf
u}_m}$, with $m\geq 1$. The set of all weight functions occurring in $\mathcal{F}$ is denoted by $\ranks(\mathcal{F})$.
For an assignment $\theta:\vars(Q)\mapsto \U$, $\mathcal{F}(\theta)$ is the weight $f^1_{{\bf u}_1}(\theta)\oplus\dots\oplus f^m_{{\bf
u}_m}(\theta)$.\hfill $\Box$
\end{defn}

As an example, note that valuation functions built for $\oplus=\mbox{`}+\mbox{'}$ basically\footnote{For more information on weighted CSPs, see
Section~\ref{sec:related}.} correspond to those arising in the classical setting of {\em weighted} CSPs, where combination of values in the
constraints come associated with a cost and the goal is to find a solution minimizing the sum of the costs over the constraints.

A constraint formula $\PHI$ equipped with a valuation function $\mathcal{F}$ is called a \emph{(constraint) optimization formula}, and is
denoted by $\PHI_\mathcal{F}$.
For an optimization formula $\Q_\mathcal{F}$ and a constraint database $\DB$, we define the total order $\succeq_\mathcal{F}$ over the
assignments in $\Q^\onDB$ such that for each pair $\theta_1$ and $\theta_2$ in $\Q^\onDB$, $\theta_1 \succeq_\mathcal{F} \theta_2$ if and only
if ${\mathcal{F}}(\theta_1)\geq {\mathcal{F}}(\theta_2)$.

For a set $O\subseteq \vars(\Q)$, we also define $\succeq_{\mathcal{F},O}$ as the total order over the assignments in $\Q^\onDB[O]$ as follows.
For each pair $\theta_1$ and $\theta_2$ of assignments in $\Q^\onDB[O]$, $\theta_1 \succeq_{\mathcal{F},O} \theta_2$ if and only if there is an
assignment $\theta_1'\in \Q^\onDB$ with $\theta_1= \theta_1'[O]$ such that $\theta_1' \succeq_\mathcal{F} \theta'_2$ holds, for each assignment
$\theta'_2\in \Q^\onDB$ such that $\theta_2=\theta_2'[O]$.
Note that $\succeq_{\mathcal{F},O}$ reflects a descending order over real numbers. To have an ascending order, we can consider operators
distributing over $\min$ (rather than over $\max$), and define the order $\succeq_{\mathcal{F},O}^{asc}$ such that $\theta_1
\succeq_{\mathcal{F},O}^{asc} \theta_2$ if and only if $\theta_2 \succeq_{\mathcal{F},O} \theta_1$. Our results are presented by focusing,
w.l.o.g., on $\succeq_{\mathcal{F},O}$ only.

Two problems that naturally arise with constraint optimization formulas are stated next. The two problems receive as input an optimization
formula $\Q_\mathcal{F}$, a set $O\subseteq \vars(\Q)$ of distinguished variables, and a constraint database $\DB$:

\begin{description}
\item[\underline{\anymaxref}{($\PHI_\mathcal{F},O,\mathrm{DB}$)}:] Compute an assignment $\theta\in \PHI^\onDB[O]$ such that there is no
    assignment $\theta'\in \PHI^\onDB[O]$ with $\theta' \succ_{\mathcal{F},O} \theta$;\footnote{As usual, the fact that $\theta_1\succeq
    \theta_2$ and $\theta_2\not\succeq \theta_1$ hold is denoted by $\theta_1 \succ \theta_2$.} Answer \texttt{NO SOLUTION}, if
    $\PHI^\onDB=\emptyset$.

\item[\underline{\topk}{($\PHI_\mathcal{F},O,\mathrm{DB}$)}:] Compute a list $(\theta_1,...,\theta_{K'})$ of distinct assignments from
    $\PHI^\onDB[O]$, where $K'=\min \{K,|\PHI^\onDB[O]|\}$, and where for each $j\in\{1,...,K'\}$, there is no assignment $\theta'\in
    \PHI^\onDB[O]\setminus \bigcup_{i=1}^{j-1}\{\theta_i\}$ with $\theta'\succ_{\mathcal{F},O} \theta_j$. Note that the parameter $K$ is an
    additional input of the problem \topk\ and, as usual, is assumed to be given in binary notation. This means, in particular, that the
    answers to this problem can be exponentially many when compared to the input size.
\end{description}

\subsubsection{Structured Valuation Functions}

Our results are actually given in the more general setting of the \emph{structured} valuation functions, where different binary operators can
be used in the same constraint formula, and the order according to which the basic weight functions have to be processed is syntactically
guided by the use of parentheses (in this case the evaluation order of operators does matter, in general).

Formally, a \emph{structured} valuation function $\mathcal{F}$ is either a weight function or an expression of the form
$(\mathcal{F}_\ell\oplus_b \mathcal{F}_r)$, where $\mathcal{F}_\ell$ and $\mathcal{F}_r$ are in turn structured valuation functions, and
$\oplus_b$ is a binary operator with the same properties as the operator in Definition~\ref{def:valuationfunction}.

Clearly enough, any structured valuation function built with one binary operator can be transparently viewed as a standard valuation function
by just omitting the parenthesis. On the other hand, given a valuation function $\mathcal{F}$, we can easily built the set
$\mathit{svf}(\mathcal{F})$ of all possible equivalent structured valuation functions, by just considering all possible legal ways of adding
parenthesis to $\mathcal{F}$.

Working on the elements of $\mathit{svf}(\mathcal{F})$ appears to be easier in the algorithms we shall illustrate in the following sections
and, accordingly, our presentation will be focused on structured valuation functions. We will discuss in Section~\ref{sec:aggregation} how to
move from structured valuation functions to equivalent standard valuation functions.

\smallskip

\begin{figure*}[t]
 \centering
    \includegraphics[width=0.99\textwidth]{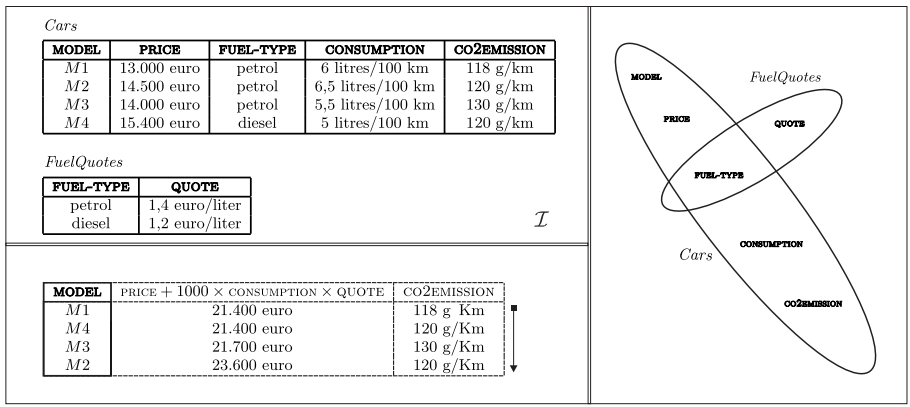}\vspace{-3mm}
    \caption{Illustration of Example~\ref{ex:intro}.}\label{fig:intro}\vspace{-3mm}
\end{figure*}

\begin{example}\label{ex:intro}\em
Consider a simple configuration scenario defined in terms of the CSP instance $(\PHI,\DB)$, where
$$
\begin{array}{lll}
\PHI & = & \mbox{\textit{Cars}\sc (model,price,fuel-type,consumption,co2emission)}\ \wedge\\
     &   & \mbox{\textit{FuelQuotes}\textsc{(fuel-type,quote)},}
\end{array}
$$
\noindent and where the atoms in $\DB$ are those shown in Figure~\ref{fig:intro} using an intuitive graphical notation.
Note that the instance is trivially satisfiable.

In fact, we are usually not interested in finding just \emph{any} solution in this setting, but would rather like to single out one that
matches as much as possible our preferences over the possible configurations. For instance, we might be interested in computing (the solution
corresponding to) a car minimizing the sum of its price plus the cost that is expected to be paid, for the given quotation of the fuel, to
cover 100.000 kilometers. Moreover, for cars that are equally ranked w.r.t.~this first criterion, we might want to give preference to cars
minimizing the emission of $CO_2$. For a sufficiently large constant $B$ (which can be treated in a symbolic way), this requirement can be
modeled via the function $\mathcal{F}=((B\times \mathcal{F}_1)+\mathcal{F}_2)$ such that
$$
\begin{array}{l}
\mathcal{F}_1  = ( f_{(\mbox{\small \sc price})} + (1000\times(f_{(\mbox{\small \sc consumption})}\times f_{(\mbox{\small \sc quote})})) \\
\mathcal{F}_2  = f_{(\mbox{\small \sc co2emission})},
\end{array}
$$

\noindent where $f_{(X)}(\theta)=\theta[X]$ is the identity weight function on each variable $X$ and where any real number is viewed as a
constant weight function. \hfill$\lhd$
\end{example}

\subsection{Structured Valuation Functions and Embeddings}

It is easily seen that structured valuation functions introduce further dependencies among the variables, which are not reflected in the basic
hypergraph-based representation of the underlying CSP instances.
Therefore, when looking at islands of tractability for constraint optimization formulas, this observation motivates the definition of a novel
form of structural representation where the interplay between functions and constraint scopes is made explicit.
In order to formalize this structural representation, we introduce the concept of \emph{parse tree} of a structured valuation function.

\begin{defn}\label{def:treeFO} \em
Let $\mathcal{F}$ be a structured valuation function. Then, the \emph{parse tree} of $\mathcal{F}$ is a labeled rooted tree
$PT(\mathcal{F})=(V,E,\tau)$, where $\tau$ maps vertices either to variables or to binary operators, defined inductively as follows:
\begin{itemize}
  \item If $\mathcal{F}$ is a weight function $f$, then $PT(\mathcal{F})=(\{f\},\emptyset,\tau_f)$ where $\tau_f(f)=\vars(f)$. That is,
      $PT(\mathcal{F})$ has no edges and a unique (root) node $f$, labeled by the variables occurring in $f$.

  \item Assume that $\mathcal{F}=(\mathcal{F}_\ell \oplus \mathcal{F}_r)$ with $PT(\mathcal{F}_\ell)=(V_\ell,E_\ell,\tau_\ell)$ and
      $PT(\mathcal{F}_r)=(V_r,E_r,\tau_r)$, and with $p_\ell$ and $p_r$ being the root nodes of  $PT(\mathcal{F}_\ell)$ and
      $PT(\mathcal{F}_r)$, respectively. Then, $PT(\mathcal{F})=(V_\ell \cup V_r\cup \{p\},E_\ell \cup E_r \cup
      \{\{p,p_\ell\},\{p,p_r\}\},\tau_p)$ is rooted at a fresh node $p$, with the labeling function $\tau_p$ such that $\tau_p(p)=\oplus$
      and its restrictions over $V_\ell$ and $V_r$ coincide with $\tau_\ell$ and $\tau_r$, respectively.
\end{itemize}

Let $O$ be a set of variables, and let $p$ be the root of $PT(\mathcal{F})$. Then, the \emph{output-aware} parse tree of $\mathcal{F}$
w.r.t.~$O$ is the labeled tree ${\it tree}(\mathcal{F},O) = (V\cup \{O\},E\cup \{ \{O,p\} \},\tau_O)$ rooted at a fresh node $O$, and where
$\tau_O$ is such that $\tau_O(O)=O$ and its restriction over $V$ coincides with $\tau$.\hfill $\Box$
\end{defn}

Now, we define the concept of embedding as a way to characterize how the parse tree of a structured function interacts with the constraints of
an acyclic constraint formula.

\begin{defn}\label{def:embedding}\em
Let $\mathcal{F}$ be a structured valuation function for a constraint formula $\PHI$, let $\HG_a$ be an acyclic hypergraph with
$\vars(\PHI)\subseteq \nodes(\HG_a)$, let $O\subseteq \vars(\PHI)$ be a set of variables, and let ${\it tree}(\mathcal{F},O) =
(V_O,E_O,\tau_O)$ be the associated output-aware parse tree.

We say that the pair $(\mathcal{F},O)$ can be \emph{embedded in $\HG_a$} if there is a join tree $\JT=(V,E,\chi)$ of $\HG_{a}$ and an injective
mapping $\xi : V_O \mapsto V$, such that every vertex $p\in V_O$ is associated with a vertex $\xi(p)$ of $\JT$, called \emph{$p$-separator},
which satisfies the following conditions: \vspace{-2mm}
\begin{itemize}
\item[(1)] $\tau_O(p)\cap\nodes(\HG_a) \subseteq \chi(\xi(p))$, i.e., the variables occurring in $p$ occur in the labeling of $\xi(p)$,
    too; and

\item[(2)] there is no pair $q$, $q'$ of vertices adjacent to $p$ in ${\it tree}(\mathcal{F},O)$ such that their images $\xi(q)$ and
    $\xi(q')$ are not separated by $\xi(p)$, i.e., such that they occur together in some connected component of the forest $(V \setminus
    \xi(p), \{ e\in E \mid e\cap \xi(p)=\emptyset\})$.
\end{itemize}

\noindent The mapping $\xi$ is called an \emph{embedding of $(\mathcal{F},O)$ in $\HG_{a}$}, and $\JT$ is its witness.\hfill $\Box$
\end{defn}

Intuitively, condition (1) states that any leaf node $p$ of the parse tree (i.e., with $\tau_O(p)\cap\nodes(\HG_a)\neq \emptyset$) is mapped
into a $p$-separator whose $\chi$-labeling covers the variables involved in the domain of the underlying weight function. Moreover, it requires
that the root node is mapped into a node whose $\chi$-labeling covers the output variables in $O$. On the other hand, condition (2) guarantees
that the structure of the parse tree is ``preserved'' by the embedding. This is explained by the example below.

\begin{figure}[t]
 \centering
    \includegraphics[width=0.8\textwidth]{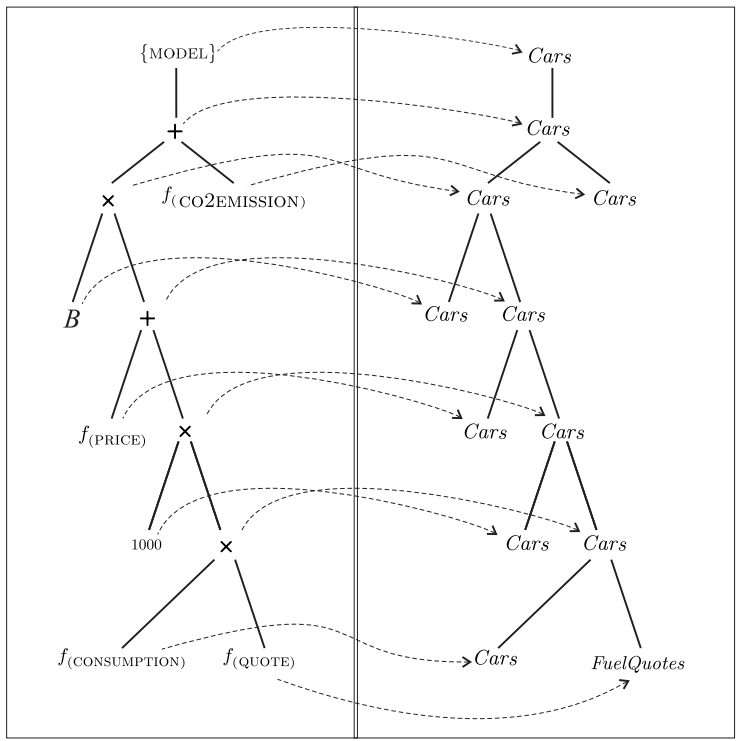}\vspace{-2mm}
    \caption{The join tree (right) and the parse tree (left) in the setting of Example~\ref{ex:introPT}.}\label{fig:introHG}
\end{figure}

\begin{example}\label{ex:introPT}\em Recall the setting of Example~\ref{ex:intro} and the output-aware parse tree
shown on the left of Figure~\ref{fig:introHG}. Observe that $\HG_\PHI$ is acyclic, as it is witnessed by the join tree depicted on the right.
Moreover, note that the figure actually shows that there is an embedding of $(\mathcal{F},\{\mbox{\sc  \small MODEL}\})$ in $\HG_{\PHI}$ that
maps each node to (the hyperedge containing the variables of) the atom $\it Cars$, except for the leaf $f_{(\mbox{\small \sc quote})}$ which is
mapped to $\it FuelQuotes$.
Note also that the root is mapped to $\it Cars$, which indeed covers the output variable $\mbox{\sc  \small MODEL}$, and that mapping constant
functions is immaterial. \hfill $\lhd$
\end{example}

\subsection{Properties of Embeddings for Structured Valuation Functions}

We shall now analyze some relevant properties of embeddings, which are useful for providing further intuitions on this notion and will be used
in our subsequent explanations.

Let $\mathcal{F}$ be a structured valuation function for a constraint formula $\PHI$, let $\HG_a$ be an acyclic hypergraph with
$\vars(\PHI)\subseteq \nodes(\HG_a)$, and let $O\subseteq \vars(\PHI)$ be a set of variables.
Let $\xi : V_O \mapsto V$ be any injective mapping and denote by $V_\xi\subseteq V$  its image.
Thus, for each vertex $v\in V_\xi$, its inverse $\xi^{-1}(v)$ is the vertex in the parse tree whose image under $\xi$ is precisely $v$. In the
following, let us view $\JT$ as a tree rooted at the vertex $\xi(O)$, where $O$ is the root of ${\it tree}(\mathcal{F},O)$. Moreover, in any
rooted tree, we say that a vertex $v$ is a descendant of $v'$, if either $v$ is a child of $v'$, or $v$ is a descendant of some child of $v'$.

\begin{theorem}\label{thm:property}
Assume that $\xi$ is an embedding of $(\mathcal{F},O)$ in $\HG_{a}$, with $\JT=(V,E,\chi)$ being its witness. Let $v$ and $v'$ be two distinct
vertices in $V_\xi$. Then, $v$ is a descendant of $v'$ in $\JT$ if and only if $\xi^{-1}(v)$ is a descendant of $\xi^{-1}(v')$ in ${\it
tree}(\mathcal{F},O)$.
\end{theorem}
\begin{proof}
We prove the property by structural induction, from the root to the leaves of $\JT$. In the base case, $v'$ is the root of $\JT$ and thus
$\xi^{-1}(v')$ is the root of ${\it tree}(\mathcal{F},O)$. In this case, the result is trivially seen to hold.
Now assume that the property holds on any vertex $v''\in V_\xi$ in the path connecting the root and a vertex $v'$. That is, $v$ is a descendant
of $v''$ in $\JT$ if and only if $\xi^{-1}(v)$ is a descendant of $\xi^{-1}(v'')$ in ${\it tree}(\mathcal{F},O)$. We show that the property
holds on $v'$, too. To this end, we first claim the following.

\begin{claim}\label{claim:path}
Let $q_1,...,q_m$ be a path in ${\it tree}(\mathcal{F},O)$ such that: {(i)} $q_i$ is a child of $q_{i-1}$, for each $i\in\{2,...,m\}$; and
{(ii)} $\xi(q_i)$ is a descendant in $\JT$ of $\xi(q_1)$, for each $i\in\{2,...,m-1\}$. Then, $\xi(q_i)$ is a descendant of $\xi(q_{i-1})$, for
each $i\in\{2,...,m\}$.
\end{claim}
\begin{itemize}
  \item[] \emph{Proof.} We prove the property by induction. Consider first the case where $i=2$. The fact that $\xi(q_2)$ is a descendant
      of $\xi(q_1)$ is immediate by \emph{(ii)}. Then, assume that the property holds up to an index $i$, with $i\in\{2,...,m-1\}$. We show
      that it holds on $i+1$, too. Indeed, by inductive hypothesis, we know that $\xi(q_i)$ is a descendant of $\xi(q_{i-1})$, which is in
      turn a descendant of $\xi(q_{1})$ by \emph{(ii)}. Consider the vertex $q_i$, and recall that since $\xi$ is an embedding, $\xi(q_i)$
      disconnects $\xi(q_{i+1})$ from $\xi(q_{i-1})$. This means that $\xi(q_{i+1})$ is a descendant of $\xi(q_i)$. \hfill $\diamond$
\end{itemize}

\noindent We now resume the main proof.

\begin{description}
  \item[(only-if)] Assume that $v$ is a descendant of $v'$. Hence, $v$ is a descendant of some vertex $v''\in V_\xi$ and we can apply the
      inductive hypothesis to derive that $\xi^{-1}(v)$ and $\xi^{-1}(v')$ are both descendant of $\xi^{-1}(v'')$ in ${\it
      tree}(\mathcal{F},O)$.
      Assume, for the sake of contradiction, that $\xi^{-1}(v)$ is not a descendant of $\xi^{-1}(v')$. We distinguish two cases.

      In the first case $\xi^{-1}(v')$ is a descendant of $\xi^{-1}(v)$. This means that there is a path
      $\xi^{-1}(v'')=q_1,...,\xi^{-1}(v),...,q_m=\xi^{-1}(v')$. Note that on this path we can apply the inductive hypothesis in order to
      conclude that $\xi(q_i)$ is a descendant of $\xi(q_1)$, for each $i\in\{2,...,m\}$. Therefore, we are in the position to apply
      Claim~\ref{claim:path}, and we conclude that $\xi(q_i)$ is a descendant of $\xi(q_{i-1})$, for each $i\in\{2,...,m\}$. In particular,
      by transitivity, we get that $\xi(q_m)$ is a descendant of $\xi(\xi^{-1}(v))$. That is, $v'$ is a descendant of $v$. Contradiction.

      The only remaining possibility is that there are two distinct vertices $q$ and $q'$ that are children of a vertex $\bar q$, which is
      a descendant of $\xi^{-1}(v'')$ or is precisely $\xi^{-1}(v'')$, and such that $\xi^{-1}(v)$ (resp., $\xi^{-1}(v')$) is a descendant
      of $q$ (resp., $q'$) or coincides with $q$ (resp., $q'$) itself. Consider then the paths $\xi^{-1}(v''),...,\bar q,q,...,v$ and
      $\xi^{-1}(v''),...,\bar q,q',...,v'$. Similarly to the case discussed above, note that on each of them we can apply
      Claim~\ref{claim:path}. Thus, we get that $\xi(q)$ and $\xi(q')$ are descendant of $\xi(\bar q)$. Moreover, either $v$ is a
      descendant of $\xi(q)$, or $\xi(q)$ coincides with $v$. Similarly,  either $v'$ is a descendant of $\xi(q')$, or $\xi(q')$ coincides
      with $v'$. Finally, recall that $\xi$ is an embedding, and hence $\xi(\bar q)$ must disconnect $\xi(q)$ and $\xi(q')$. Therefore, it
      disconnects $v$ and $v'$, too. Contradiction with the fact that $v$ is a descendant of $v'$.

  \item[(if)] Assume that $\xi^{-1}(v)$ is a descendant of $\xi^{-1}(v')$ in ${\it tree}(\mathcal{F},O)$. Assume, for the sake of
      contradiction, that $v$ is not a descendant of $v'$. Because of the only-if part, we are guaranteed that $v'$ is in any case not a
      descendant of $v$. Therefore, there is a vertex  $\bar v$ disconnecting $v$ and $v'$ and such that $v$ and $v'$ are both descendant
      of $\bar v$. We can now apply the inductive hypothesis on the vertex $v''\in V_\xi$ in the path connecting $\bar v$ and the root, and
      that is the closest to $\bar v$ (possibly coinciding with it). Therefore, we know that $\xi^{-1}(v)$ and $\xi^{-1}(v')$ are both
      descendant of $\xi^{-1}(v'')$. Hence, $\xi^{-1}(v')$ occurs in the path connecting $\xi^{-1}(v'')$ and $\xi^{-1}(v)$. Consider then
      the path $\xi^{-1}(v'')=q_1,...,\xi^{-1}(v'),...,q_m=\xi^{-1}(v)$. Because of the inductive hypothesis, we know that $\xi(q_i)$ is a
      descendant of $\xi(q_1)$, for each $i\in\{2,...,m\}$. Therefore, we are in the position of applying Claim~\ref{claim:path} and, by
      transitivity, we derive that $\xi(q_m)$ is a descendant of $\xi(\xi^{-1}(v'))$. That is, $v$ is a descendant of $v'$. Contradiction.
\end{description}

\vspace{-10mm}\ \\
\end{proof}

In words, the above result tells us that embeddings preserve the descendant relationship. In fact, preserving this relationship suffices for an
embedding to exist.

\begin{theorem}\label{thm:injective}
Let $\xi : V_O \mapsto V$ be an injective function satisfying condition (1) in Definition~\ref{def:embedding} for ${\it tree}(\mathcal{F},O) =
(V_O,E_O,\tau_O)$ and for a join tree $\JT=(V,E,\chi)$ of $\HG_{a}$. Assume that for each pair $v,v'$ of distinct vertices in $V_\xi$, $v$ is a
descendant of $v'$ in $\JT$ if and only if $\xi^{-1}(v)$ is a descendant of $\xi^{-1}(v')$ in ${\it tree}(\mathcal{F},O)$. Then, there is an
embedding $\xi'$ of $(\mathcal{F},O)$ in $\HG_{a}$ (with witness $\JT$), which can be built in polynomial time from $\xi$.
\end{theorem}
\begin{proof}
Based on $\xi : V_O \mapsto V$, we build a function $\xi': V_O \mapsto V$ as follows. Let $x$ be any node in $V_O$. If $x$ is a leaf node in
${\it tree}(\mathcal{F},O)$ or it is the root, then we set $\xi'(x)=\xi(x)$. Otherwise, i.e., if $x$ is an internal node with children $y$ and
$z$, then we first observe that, by hypothesis, $\xi(y)$ and $\xi(z)$ are both descendant of $\xi(x)$. Moreover, $\xi(y)$ (resp., $\xi(z)$) is
not a descendant of $\xi(z)$ (resp., $\xi(y)$), and therefore there is a vertex $\bar v$ in $V$ possibly coinciding with $\xi(x)$ such that
$\xi(y)$ and $\xi(z)$ occur in different components of $(V \setminus \bar v, \{ e\in E \mid e\cap \bar v=\emptyset\})$ as descendants of $\bar
v$---in particular, whenever $\bar v\neq \xi(x)$, we have that $\bar v$ is a descendant of $\xi(x)$. For the node $x$, we now define
$\xi'(x)=\bar v$.

Note that $\xi'$ trivially satisfies condition (1) in Definition~\ref{def:embedding}, as $\xi'$ differs from $\xi$ only over non-leaf nodes
different from the root. We claim that $\xi'$ satisfies condition (2), too.

Recall that, by Definition~\ref{def:treeFO}, the output-aware parse tree is binary, and its root is the only vertex having one child. Consider
next any vertex $p\in V_O$ with parent $s$ and children $q$ and $q'$. Indeed, if $p$ is the root or a leaf, then condition (2) in
Definition~\ref{def:embedding} trivially holds.
By the above construction (setting $x=p$), we know that $\xi(q)$ and $\xi(q')$ occur in different components of $(V \setminus \xi'(p), \{ e\in
E \mid e\cap \xi'(p)=\emptyset\})$ as descendants of $\xi'(p)$. Moreover, whenever $\xi'(q)\neq \xi(q)$, we are guaranteed that $\xi'(q)$ is a
descendant of $\xi(q)$ (again by the above construction, this time setting $x=q$). Similarly, either $\xi'(q')=\xi(q')$, or $\xi'(q')$ is a
descendant of $\xi(q')$. Hence, $\xi'(q)$ and $\xi'(q')$ occur in different components of $(V \setminus \xi'(p), \{ e\in E \mid e\cap
\xi'(p)=\emptyset\})$ as descendants of $\xi'(p)$.

Consider now the parent $s$ and the child $q$---the same line of reasoning applies to the child $q'$. By hypothesis, we know that $\xi(p)$
occurs in the path connecting $\xi(s)$ and $\xi(q)$, with $\xi(q)$ being a descendant of $\xi(s)$. Moreover, by construction of $\xi'$ (over
$s$, $p$, and $q$), we have that: either $\xi'(s)$ occurs in the path connecting $\xi(s)$ and $\xi(p)$, or $\xi'(s)=\xi(s)$; either $\xi'(p)$
occurs in the path connecting $\xi(p)$ and $\xi(q)$, or $\xi'(p)=\xi(p)$; and $\xi'(q)$ is either a descendant of $\xi(q)$, or coincides with
$\xi(q)$. Therefore, $\xi'(p)$ occurs in the path connecting $\xi'(s)$ and $\xi'(q)$. That is, $\xi'(s)$ and $\xi'(q)$ occur in different
components of $(V \setminus \xi'(p), \{ e\in E \mid e\cap \xi'(p)=\emptyset\})$.

By putting all together, we have shown that condition (2) in Definition~\ref{def:embedding} holds on any vertex $p\in V_O$.
\end{proof}

\section{Structural Tractability in the Tree Projection Setting}\label{sec:main}

The concept of embedding has been introduced in Section~\ref{sec:framework} as a way to analyze the interactions of valuation functions with
acyclic instances. However, the concept can be easily coupled with the tree projections framework in order to be applied to instances that are
not precisely acyclic. This coupling is formalized below.

\begin{defn}\em
Let $\mathcal{F}$ be a structured valuation function for a constraint formula $\PHI$, let $O\subseteq \vars(\PHI)$ be a set of variables, and
let $\V$ be a view set for $\PHI$.
We say that $(\mathcal{F},O)$ can be \emph{embedded in $(\Q,\V)$} if there is a sandwich formula $\Q_a$ of $\Q$ w.r.t.~$\V$ such that
$(\mathcal{F},O)$ can be embedded in the acyclic hypergraph $\HG_{\Q_a}$. If $O=\emptyset$, then we just say that $\mathcal{F}$ can be embedded
in $(\Q,\V)$.\hfill $\Box$
\end{defn}

\begin{example}\em
Recall the setting of Example~\ref{ex:treeDecomposition} and the valuation function $\mathcal{F}=f_a$, where $f_a(\theta)=\theta(A)$, for each
solution $\theta$. It is immediate to check that $\mathcal{F}$ can be embedded in $(\PHI,\V_2)$. Indeed, consider the sandwich formula $\PHI_a$
depicted in Figure~\ref{fig:Example-TW}, and note that $(\mathcal{F},\emptyset)$ can be embedded in $\HG_{\PHI_a}$. In fact, as $\mathcal{F}$
consists of a weight function only, this is witnessed by any join tree $\JT$ of $\HG_{\PHI_a}$ because we can always build an embedding that
maps $f_a$ to a vertex $q$ of $\JT$ such that $\chi(q)\supseteq \{A\}$.
Note that, for this kind of valuation functions, checking the existence of an embedding always reduces to checking the existence of a tree
projection.\hfill $\lhd$
\end{example}

Recall that deciding whether a pair of hypergraphs has a tree projection is an $\NP$-complete problem~\cite{GMS09}, so that the notion of
embedding can hardly be exploited in a constructive way when combined with tree projections. However, we show in this section that the
knowledge of an embedding (and of a tree projection) is not necessary to compute the desired answers. Indeed, a \emph{promise-free} algorithm
can be exhibited that is capable of returning a solution to \anymaxref\ (and \topk) or to check that the given instance is not embeddable.
The algorithm is in fact rather elaborated, and we start by illustrating some useful properties that can help the intuition.

Hereinafter, let $\Q$ be a constraint formula, $\DB$ a constraint database, $O$ a set of variables, and $\mathcal{F}$ a {structured} valuation
function, all of them being provided as input to our reasoning problems. Moreover, to deal with the setting of tree projections, we assume that
a view set $\V$ for $\Q$ plus a constraint database $\DB\mbox{\rm '}$ that is legal for $\V$ w.r.t.~$\Q$ and $\DB$ are provided.
Accordingly, to emphasize the role played by these structures, the problems of interest will be denoted as
{\anymaxref}{($\PHI_\mathcal{F},O,\mathrm{DB},\V,\DBW$)} and  {\topk}{($\PHI_\mathcal{F},O,\mathrm{DB},\V,\DBW$)}.

\subsection{Useful Properties of Embeddings}

For a constraint database $\DB$ and an atom $a$, the set $a^\onDB$ will be also denoted by $\rel(a,\DB)$. Substitutions in $\rel(a,\DB)$ will
be also viewed as the ground atoms in $\DB$ to which they are unambiguously associated.
If $q$ is an atom, $\Q^\onDB[q]$ denotes $\Q^\onDB[\vars(q)]$.

Without loss of generality, assume that, for each weight function $f_i\in\ranks(\mathcal{F})$, $\V$ contains a {\em function view} $w_{f_i}$
over the variables in $\vars(f_i)$. Indeed, if $w_{f_i}\not\in \V$, then we can just define $w_{f_i}$ as the projection over $\vars(f_i)$ of
any view $w\in \V$ such that $\vars(w)\supseteq \vars(f_i)$. In particular, if such a view $w$ does not exist, then we can immediately conclude
that $\mathcal{F}$ cannot be embedded in $(\Q,\V)$.

Let us define $\overline \DB$ as the constraint database obtained (in polynomial time) by \emph{enforcing pairwise consistency} on $\DBW$
w.r.t.~$\V$~\cite{BFMY83}. The method consists of repeatedly applying, till a fixpoint is reached, the following constraint propagation
procedure: Take any pair $w$ and $w'$ of views in $\V$, and delete from $\DB$ any (ground atom associated with an) assignment $\theta$ in
$\rel(w,\DB)$ for which no assignment $\theta'\in \rel(w',\DB)[\vars(w)\cap \vars(w')]$ exists with $\theta'\subseteq \theta$. In words, the
procedure removes, for each view $w$, all its associated assignments $\theta$ that cannot be extended to some assignment in each of the
remaining views.
In the database terminology, this is called a \emph{semijoin} operation over $w$ and $w'$.

The crucial property enjoyed by the database $\overline \DB$, which we shall intensively use in our elaborations, is recalled below.

\begin{proposition}[\cite{GS10}]
Assume there exists a tree projection $\HG_a$ of $\HG_\PHI$ with respect to $\HG_\V$. Then, {$w^{\overline \onDB}[h]=\Q^\onDB[h]$} holds,
for every $w\in \V$ and $h\subseteq \vars(w)$ such that there is a hyperedge $h_a$ of $\HG_a$ with $h\subseteq h_a$.
\end{proposition}

For any partial assignment $\theta: \mathcal{W} \mapsto \U$, where $\mathcal{W}\subseteq \vars(\Q)$, and for any constraint optimization
formula $\PHI_\mathcal{F}$, denote by $\max_{\mathcal{F}}(\theta)$ the maximum weight that any assignment $\theta'\in \Q^\onDB$  with
$\theta'[\mathcal{W}]=\theta$ can get according to $\mathcal{F}$, that is, $\max_{\mathcal{F}}(\theta)=\max (\{ \mathcal{F}(\theta') \mid
\theta'\in \Q^\onDB \mbox{ and } \theta'[\mathcal{W}]=\theta\}\cup\{\bot\})$, where $\bot$ denotes the minimum weight in the codomain of the
valuation function.

Let $p$ be any vertex of $\mathit{tree}(\mathcal{F},O)$ occurring in the parse tree $PT(\mathcal{F})$. Let $\mathcal{F}_p$ denote the
subexpression of $\mathcal{F}$ whose parse tree is the subtree rooted at $p$.
Note that if $\mathcal{F}_p=f_i$ holds for some weight function $f_i$, then $\forall \theta\in w_{f_i}^{\overline{\onDB}},
f_i(\theta)=\max_{\mathcal{F}_p}(\theta)$ holds by construction.
Assume now that $p$ is a (non-leaf) vertex labeled by $\oplus$, and let $\ell$ and $r$ be its children. Then, the maximum weight of
$\mathcal{F}_p=({\mathcal{F}_\ell\oplus\mathcal{F}_r})$ is bounded by the aggregation via $\oplus$ of the maximum weights that can be achieved
over its $\mathcal{F}_\ell$ and $\mathcal{F}_r$.

\begin{lemma}\label{lem:bound}$\max_{\mathcal{F}_\ell\oplus\mathcal{F}_r}(\theta)$$\leq$$\max_{\mathcal{F}_\ell}(\theta) \oplus
\max_{\mathcal{F}_r}(\theta)$ holds for each partial assignment $\theta$.
\end{lemma}
\begin{proof}
Let $\bar \theta$ be an assignment such that $\max_{\mathcal{F}_\ell\oplus\mathcal{F}_r}(\theta)=\mathcal{F}_\ell(\bar
\theta)\oplus\mathcal{F}_r(\bar \theta)$. Then, the result follows by the properties of $\oplus$ and since $\mathcal{F}_\ell(\bar \theta)\leq
\max_{\mathcal{F}_\ell}(\theta)$ and $\mathcal{F}_r(\bar \theta)\leq \max_{\mathcal{F}_r}(\theta)$.
\end{proof}

Assume now that there exists a tree projection $\HG_a$ of $\HG_\PHI$ with respect to $\HG_\V$, and that the pair $(\mathcal{F},O)$ can be
\emph{embedded in $\HG_a$}. Let  $\JT=(V,E,\chi)$ be the join tree of $\HG_{a}$ and $\xi : V_O \mapsto V$ be the injective mapping of
Definition~\ref{def:embedding}. Then, we show that the inequality in Lemma~\ref{lem:bound} is tight on separators, thus the operation of
choosing the best sets of partial assignments computed in the subtrees distributes over $\oplus$.

\begin{lemma}\label{lem:distribution} $\max_{\mathcal{F}_\ell\oplus\mathcal{F}_r}(\theta) = \max_{\mathcal{F}_\ell}(\theta) \oplus
\max_{\mathcal{F}_r}(\theta)$ holds for each {\em $\theta\in w^{\overline \onDB}[q_a]$}, where $q_a$ is a $p$-separator for some $p\in V_O$
and $w\in \V$ is a view with $\vars(w)\supseteq \vars(q_a)$.
\end{lemma}
\begin{proof}
Recall that $\xi(p)$ separates its children $\ell$ and $r$ in the join tree $\JT$ of $\HG_{\Q_a}$. Observe that, by Lemma~\ref{thm:property},
the images of all vertices of the subtree rooted at $\ell$ (resp., $r$) belong to the same connected component $C_\ell$ (resp., $C_r$) of
$\JT\setminus\xi(p)$. In particular, $C_\ell\neq C_r$. Thus, from the connectedness condition of join trees, every variable that $C_\ell$ and
$C_r$ have in common must be included in $\chi(\xi(p))$. In fact, any partial assignment $\theta\in w^{\overline \onDB}[q_a]$ provides a value
for all these variables. Thus, all possible extensions of $\theta$ to $C_\ell$ are independent of their extensions to $C_r$, so that they can
be freely combined. Hence, we can safely obtain $\max_{\mathcal{F}_\ell\oplus\mathcal{F}_r}(\theta)$ by computing $\oplus$ over the maximum
weights obtained for $\theta$ looking at $\mathcal{F}_\ell$ and $\mathcal{F}_r$ in a separate way.
\end{proof}

In the light of Lemma~\ref{lem:bound} and Lemma~\ref{lem:distribution}, it is not difficult to define a bottom-up algorithm that, given the
tree projection $\HG_a$ and the embedding $\xi$, processes from the leaves to the root each vertex $p$ of $(\mathcal{F},O)$ computing the
maximum weights that can be achieved by combining the results coming from its children, and using some view covering the $p$-separator.

However, because deciding whether there is a tree projection (and compute one, if one exists) is $\NP$-hard~\cite{GMS09}, such a na\"ive
approach to solve \anymaxref\ and \topk\ is impractical for large constraint formulas. We need a method that is able to perform the computation
even when an embedding is not given.

\subsection{Algorithm \algMax}

\begin{figure*}[t]
\centering \centering \fbox{
\parbox{0.82\textwidth}{
\hspace{-1mm}
\begin{tabular}{l}
  \hspace{-2mm}\textbf{Input}: A constraint formula $\Q$; a structured valuation function $\mathcal{F}$;\\
  \hspace{-2mm}\phantom{\textbf{Input}: }a view set $\V$ for $\Q$;\\
  \hspace{-2mm}\phantom{\textbf{Input}: }a constraint database $\DB\mbox{\rm '}$ legal for $\V$ w.r.t.~$\Q$ and $\DB$; and,\\
  \hspace{-2mm}\phantom{\textbf{Input}: }a set of variables $O\subseteq \vars(\Q)$;\\
  \hspace{-2mm}\textbf{Output}: An assignment, \texttt{NO SOLUTION}, or \texttt{FAIL}; \hspace{-20mm}\\
  \hline\\
  \vspace{-6mm}\ \hspace{-80mm} \\

  \hspace{-2mm}\textbf{begin} \\
  \hspace{-4mm}\ \ \ \ $\DB_1 :=$ {\it EnforcePairwiseConsistency}$(\V,\DBW)$;\\
  \hspace{-4mm}\ \ \ \ \textbf{if} some database relation is empty \textbf{then} 
  \textbf{Output} \texttt{NO SOLUTION};\\
  \hspace{-4mm}\ \ \ \ Let $p_1,\dots,p_s$ be a topological ordering of the vertices of $\mathit{tree}(\mathcal{F},O)$;\\
  \hspace{-4mm}\ \ \ \ Initialize the sets of candidate separators $\newsep_1,\dots,\newsep_s$,\\
  \hspace{-4mm}\ \ \ \ \ \ \ \ and let $\DB_1$ be the resulting constraint database;\\
  \hspace{-4mm}\ \ \ \ \textbf{for} i:=1 \textbf{to} s-1 \textbf{do}\\
  \hspace{-4mm}\ \ \ \ \ \ \ \ $\DB_i' := {\it evaluate}(\newsep_i,\DB_i)$;\\
  \hspace{-4mm}\ \ \ \ \ \ \ \ \textbf{if} $\newsep_i$ is empty \textbf{then} 
  \textbf{Output} \texttt{FAIL};\\
  \hspace{-4mm}\ \ \ \ \ \ \ \ $\DB_{i+1}:= {\it propagate}(\newsep_i,\DB_i')$\\
  \hspace{-4mm}\ \ \ \ $\DB_{s+1} :={\it evaluate}(\newsep_s,\DB_s)$;\\
  \hspace{-4mm}\ \ \ \ \textbf{if} $\newsep_s$ is empty \textbf{then} 
  \textbf{Output} \texttt{FAIL};\\
  \hspace{-4mm}\ \ \ \ \textbf{else}  \\
  \hspace{-4mm}\ \ \ \ \ \ \ \ \textbf{let} ${\it rel}= [w X]^{\onDB_{s+1}}$, where $\newsep_s=\{[w X] \}$;\\
  \hspace{-4mm}\ \ \ \ \ \ \ \ \textbf{Output} any assignment from $\{ \theta\in {\it rel} \mid \theta[X]=\max \{{\it rel}[X]\} \}$;\\
  \hspace{-2mm}\textbf{end}.\\
\end{tabular}
}} \caption{\textbf{Algorithm} \algMax.}\label{fig:algoritmoPromise}
\end{figure*}

Our approach to solve  \anymaxref\ and \topk\ even without the knowledge of (a sandwich formula and of) an embedding is based on the Algorithm
\algMax\ shown in Figure~\ref{fig:algoritmoPromise}.

Assume that the vertices $p_1,\dots,p_s$ of $\mathit{tree}(\mathcal{F},O)$ are numbered according to some topological ordering of this tree
(from leaves to root).
As no tree projection and no embedding are known, we miss the relevant information about which views behave as separators. This is dealt with
in \algMax\ by maintaining, for each vertex $p_i$, a set $\newsep_i$ of views that are \emph{candidates} to be $p_i$-separators in some tree
projections and w.r.t.~to some embedding. These sets are managed as follows.

\medskip \noindent \textbf{Initialization.}
Define $\DB_1$ as the database obtained by enforcing pairwise-consistency on $\DBW$ w.r.t.~$\V$. Let $p_i$ be a vertex of
$\mathit{tree}(\mathcal{F},O)$, and consider three cases:

\begin{myitemize}
\item[$\mbox{\it \sc Leaf node:}$] {$p_i$ is a leaf associated with the weight function~$f_i$.} Then, $\newsep_i$ contains only an
    ``augmented'' view $[w_{f_i} X^{(w_{f_i})}_i]$ over a fresh relation symbol, and over all variables in $\vars(w_{f_i})$ plus the fresh
    variable $X^{(w_{f_i})}_i$. Accordingly, $\DB_1$ is enlarged to contain the relation:
$$
\small \rel([w_{f_i} X^{(w_{f_i})}_i],\DB_1)=\{ \theta\cup \{X^{(w_{f_i})}/f_i(\theta)\} \mid \theta\in w_{f_i}^{\onDB_1} \}.
$$
Thus, the auxiliary variable $X^{(w_{f_i})}_i$ is meant to store the weight of the function $f_i=\mathcal{F}_{p_i}$ for each assignment of the
view $w_{f_i}$. Note that the sample set $\newsep_i$ includes just one (augmented) function view, as  $w_{f_i}$ can always be used as a
$p_i$-separator.

\item[$\mbox{\it \sc Internal node:}$] {$p_i$ is a non-leaf vertex having two children named $p_r$ and $p_t$ in the given ordering.} Then, for
    each $w\in \V$, $w'\in\newsep_r$, and $w''\in\newsep_t$, the set $\newsep_i$ includes the augmented view $[w X^{(w')}_{r} X^{(w'')}_{t}]$,
    over the variables in $\vars(w)$ plus the fresh variables $X^{(w')}_{r}$ and $X^{(w'')}_t$. Accordingly, $\DB_1$ is enlarged to contain the
    relation:
$$
\small \rel([w X^{(w')}_{r} X^{(w'')}_{t}],\DB_1)=\{ \theta\cup \{X^{(w')}_{r}/\mathit{noval},X^{(w'')}_{t}/\mathit{noval}\} \mid \theta\in w^{\onDB_1} \}.
$$
Intuitively, augmented views store the weights derived during the computation for functions $\mathcal{F}_{p_r}$ and $\mathcal{F}_{p_t}$.
Initially, we consider the constant $\mathit{noval}$, meaning that no weight is currently available. Note that for internal nodes we need to
keep all the possible views in $\V$ as candidates for being $p_i$-separators.

\item[$\mbox{\it \sc Root:}$] {$p_i=p_s$ is the root whose only child is $p_{s-1}$.} Then, let us chose any view $w_O\in \V$ such that
    $O\subseteq \vars(w_o)$, which exists for otherwise there would be no embedding. For each view $w\in \newsep_{s-1}$, the set $\newsep_s$
    includes the augmented view $[w_O X^{(w)}_{s-1}]$, whose relations in $\DB_1$ are:
$$
\small \rel([w_O X^{(w)}_{s-1}],\DB_1)=\{ \theta\cup\{X^{(w)}_{s-1}/\mathit{noval}\} \mid  \theta\in w_O^{\onDB_1}[O]\}.
$$
\end{myitemize}

\noindent During the initialization, $\DB_1$ is modified so as to include augmented views. However, the projection of each augmented view over
the variables occurring in $\Q$ gives precisely the original underlying view. Thus, Lemma~\ref{lem:bound} and Lemma~\ref{lem:distribution} hold
over (the modified) $\DB_1$ and the augmented views. During the computation, a sequence of such constraint databases $\DB_2,...,\DB_s$ is
constructed. For each of them, the equivalence with $\DB_1$ when considering projections over the variables in the original views is
guaranteed.

\medskip \noindent \textbf{Main Loop.}
After their initialization, views in $\newsep_i$ are incrementally processed, from $i=1$ to $i=s$, via the functions \emph{evaluate} and
\emph{propagate}. Both functions receive as input a candidate separator and a current constraint database, and produce as output a novel
constraint database. The functions

\medskip \noindent \emph{Step} $\it \DB_i':=\mathit{evaluate}(\newsep_i,\DB_i)$.
The goal of this step is to evaluate functions over the candidates in $\newsep_i$ and to filter out those that cannot be $p_i$-separators. When
invoked according to the topological ordering, it will be guaranteed that the active domain in $\DB_i$ of any variable of any augmented view in
$\newsep_i$ does not include $\mathit{noval}$, because functions associated with the children have been previously evaluated.
We distinguish three cases:

\begin{myitemize}
\item[$\mbox{\it \sc Leaf node:}$] In this case, no operation is required, as $\newsep_i$ contains one good augmented view, by initialization.

\item[$\mbox{\it \sc Internal node:}$] For every $w\in \V$, recall that $\newsep_i$ contains the view $[w X^{(w')}_{r} X^{(w'')}_{t}]$, for
    each pair $w'\in \newsep_r$ and $w''\in \newsep_t$. Let $\oplus$ be the label of $p_i$, and for any assignment $\bar \theta\in
    w^{{\onDB_1}}$, let $\mathit{marg}(\bar \theta,w,w',w'')$ be the maximum of $\theta[X^{(w')}_{r}]\oplus \theta[X^{(w'')}_{t}]$ taken over
    all the assignments $\theta\in [w X^{(w')}_{r} X^{(w'')}_{t}]^{\onDB_i}$ such that $\bar \theta=\theta[w]$. This is often called
    marginalization of $[w X^{(w')}_{r} X^{(w'')}_{t}]$ w.r.t.~$X^{(w')}_{r}\oplus X^{(w'')}_{t}$. Let $\mathit{best}(\bar \theta,w)$ denote
the minimum weight of $\mathit{marg}(\bar \theta,w,w',w'')$ over all the possible pairs of views $w'$ and $w''$, and define for $w$ the
augmented view $[w X^{(w)}_{i}]$, whose associated relation in $\DB_i$ is:
$$
\small \rel([w X^{(w)}_{i}],\DB_i)=\{ \theta\cup\{ X^{(w)}_{i}/\mathit{best}(\bar \theta,w)\} \mid \bar \theta\in w^{\onDB_1}\}.
$$

Then, $\newsep_i$ is modified by including only all augmented views of the form $[w X^{(w)}_{i}]$.

The rationale of this step can be understood by first recalling that every assignment in a $p_i$-separator is associated with the largest
weight over its possible extensions to full answers according to $\mathcal{F}_{p_i}$ (cf. Lemma~\ref{lem:bound} and
Lemma~\ref{lem:distribution}). In fact, when analyzing the algorithm, we shall show that good candidates to act as $p_i$-separators are those
having the minimum marginalized weights for each one of their assignments, which therefore motivates the definition of the term
$\mathit{best}(\bar \theta,w)$. Based on this fact, we actually delete from $\newsep_i$ every view whose maximum weight over all its
assignments is not the minimum over the maximum weights of all other views. This way $|\newsep_i|\leq |\V|$, and the maximum weight stored
somewhere in the constraint database is bounded by its real maximum over the answers of the given constraint formula. Therefore, no space
explosion may occur, neither in terms of number of samples nor in terms of size of the weights stored in the views.

\item[$\mbox{\it \sc Root:}$] We drop all views from $\newsep_s$,  but one view (if any) of the form $[w_O X^{(w)}_{s-1}]$ such that for each
    $\theta\in [w_O X^{(w)}_{s-1}]^{{\onDB_s}}$, $\theta[X^{(w)}_{s-1}]\leq \theta'[X^{(w')}_{s-1}]$ holds over any $[w_O X^{(w')}_{s-1}]\in
    \newsep_s$, and any $\theta'\in [w_O X^{(w')}_{s-1}]^{{\onDB_s}}$ with $\theta'[w_O]=\theta[w_O]$.
\end{myitemize}

\medskip \noindent \emph{Step} $\it \DB_{i+1}:=\mathit{propagate}(\newsep_i,\DB_i')$. Let $p_j$ be the parent of $p_i$. In this step, we propagate
the information of the views in $\newsep_i$ into $\newsep_j$. For any variable $X$, let $\dom(X)$ denote its active domain in $\DB_{i}'$.
    Then, for each view $[w X^{(w)}_i]\in \newsep_i$, propagation is implemented via the following steps (1)---(6):
\begin{myitemize2}
\item[(1)] initialize a set $\V_i = \{[w X^{(w)}_i] \}$;

\item[(2)] add to $\V_i$ all augmented views $[w_b X^{(w)}_i]$ for each $w_b\in\V$, and to $\DB_{i}'$ their corresponding relations of the form
    $\rel([w_b X^{(w)}_i],\DB'_i)=w_b^{\onDB_i'}\times \dom(X^{(w)}_i)$; that is, these views are not restrictive w.r.t.~$X_i$ because all its
    possible weights are considered;

\item[(3)] add to $\V_i$ all the views of the form $[w' X^{(w)}_i X^{(w'')}_r]$ which are stored in $\newsep_j$.
Denote by $R$ the  projection  of $\rel([w' X^{(w)}_i X^{(w'')}_r],\DB_i')$ over all variables but $X^{(w)}_i$, and update $\rel([w' X^{(w)}_i
X^{(w'')}_r],\DB_i')$ to be the relation containing all assignments $\theta$ such that $\theta[w' X^{(w'')}_r]\in R$ and
$\theta[X^{(w)}_i]\in\dom(X^{(w)}_i)$. Repeat the step for the symmetrical case of those views having the form $[w' X^{(w'')}_r X^{(w)}_i ]$.

\item[(4)] update $\DB_{i+1}'$ with the result of {\it EnforcePairWiseConsistency}$(\V_i,\DB_i')$;

\item[(5)] remove from $\DB_{i+1}'$ the relations added at step~(2).

\item[(6)] replace each view $[w' X^{(w)}_i X^{(w'')}_r]\in \newsep_j$ by its marginalization w.r.t.~$X^{(w)}_i$, that is, remove from
    $\rel([w' X^{(w)}_i X^{(w'')}_r],\DB_{i+1}')$ any assignment $\theta$ for which there is an assignment
$\theta'$ in the same relation with $\theta[w' X^{(w'')}_r]=\theta'[w' X^{(w'')}_r]$ and $\theta[X^{(w)}_i]<\theta'[X^{(w)}_i]$. Repeat for all
views of the form $[w' X^{(w'')}_r X^{(w)}_i ]$.
\end{myitemize2}

Note that the goal of  steps (1)---(4) above is to filter views of the form $[w' X^{(w)}_i X^{(w'')}_r]$ that are stored in $\newsep_j$, by
keeping the assignments that agree with the weights stored in the view $[w X^{(w)}_i]$. This is done by enforcing local consistency via the
augmented views $[w_b X^{(w)}_i]$ added at step~(2). In particular, such augmented views (as well as the target view $[w' X^{(w)}_i
X^{(w'')}_r]$) do not constrain the weights for the variable $X^{(w)}_i$, but just propagate the information in $[w X^{(w)}_i]$, as they are
initialized by associating the whole active domain $\dom(X^{(w)}_i)$ with each assignment in the corresponding original views of $\V$. In
particular, as their role is just to propagate the information from sample $[w X^{(w)}_i]$ to sample $[w' X^{(w)}_i X^{(w'')}_r]$, they are
eventually removed in step (5) from the constraint database. Finally, note that the same assignment can be propagated with different associated
weights. The final ingredient (used at the end of each ``evaluate step'') is to retain the assignment with the minimum associated weight.

\medskip \noindent \textbf{Concluding Step.} If after the last invocation of the evaluation step $\newsep_s$ contains one view, then
output any of its assignments having the maximum associated weight.

\subsection{``Depromisization'' and Analysis Overview}\label{sec:deprom}

The analysis of \algMax\ is rather technical, and its details are deferred to the Appendix. Observe that algorithm \algMax\ is a \emph{promise}
algorithm, in that it is guaranteed  to correctly return a solution to \anymaxref\  under the hypothesis that some embedding exists (the
``promise''). Actually, we can show that its correctness does not require that the constraint optimization formula can be embedded in a tree
projection of the entire CSP instance. Indeed, we can show that it suffices that an embedding exists for some \emph{homomorphically equivalent}
subformula.

\begin{remark}\em
We consider the usual computational setting where each mathematical operation costs $1$ time unit. However, all the algorithms described in the
following are such that the total size of the weights computed during their execution is polynomially bounded w.r.t.~the combined size of the
input and the size of the value of any optimal solution (assuming that the promise holds).
This is a sensitive issue because, in the adopted computational setting, one may compute in polynomial-time weights of size exponential
w.r.t.~the input size.
\end{remark}

To state our main result, recall first that, whenever $\Q'$ is a subformula of $\Q$, i.e., $\atoms(\Q')\subseteq \atoms(\Q)$, we say that $\Q'$
is \emph{homomorphically equivalent} to $\Q$, denoted by $\Q'\homEquiv \Q$, if there is a homomorphism from $\Q$ to $\Q'$, i.e.,  mapping $h:
\vars(Q) \mapsto \vars(Q')$ such that for each $r_{i}({\bf u_i})\in \atoms(Q)$, it holds that $r_{i}(h({\bf u_i}))\in \atoms(Q')$.
For any set $O$ of variables, denote by $\atom(O)$ a fresh atom over the variables in $O$.

\begin{theorem}\label{thm:promiseWeakMax}
Algorithm~\algMax\ runs in polynomial time. It outputs {\em \texttt{NO SOLUTION}}, only if $\PHI^\onDB=\emptyset$.
Moreover, it computes 
an answer (if any) to {\em \anymaxref($\PHI_\mathcal{F},O,\mathrm{DB},\V,\DBW$)}, with $\mathcal{F}$ being a structured valuation function, if
$(\mathcal{F},O)$ can be embedded in $(\Q',\V)$ for some subformula $\Q'$ of $\Q\wedge\atom(O)$ such that $\Q'\homEquiv \Q\wedge\atom(O)$. It
outputs {\em \texttt{FAIL}}, only if this condition does not hold.
\end{theorem}

Interestingly, \algMax\ can be used as  a subroutine for an algorithm that incrementally builds a solution in $\Q^\onDB$, hence yielding a
promise-free algorithm, i.e., an algorithm that either computes a correct solution or disproves some given promise (which is $\NP$-hard to be
checked), in our case the existence of an embedding. The full proof of the following result is given in the Appendix, but a proof idea is
discussed below.

\begin{theorem}\label{thm:nopromiseMAX}
There is a polynomial-time algorithm for structured valuation functions that either solves {\em
\anymaxref($\PHI_\mathcal{F},O,\mathrm{DB},\V,\DBW$)}, 
or disproves that $\mathcal{F}$ can be embedded in~$(\Q,\V)$.
\end{theorem}
\begin{proof}[Proof Idea]
Given a variable $X\in \vars(\Q)$, we invoke \algMax\ with $O=\{X\}$ and with a modified set of views $\V_X$, which are obtained by augmenting
each original view in $\V$ with the variable $X$ (and by modifying the original legal database $\DBW$ accordingly). Note that
$(\mathcal{F},\{X\})$ can be embedded in $(\Q,\V_X)$ if and only if $\mathcal{F}$ can be embedded in $(\Q,\V)$.
By the application of Theorem~\ref{thm:promiseWeakMax} on the modified instance, if \algMax\ returns \texttt{NO SOLUTION} (resp.,
\texttt{FAIL}), then we can terminate the computation, by returning that there is no solution, i.e., $\PHI^\onDB=\emptyset$ (resp., the promise
that $\mathcal{F}$ can be embedded in $(\Q,\V)$ is disproved---observe that the equivalent promise that  $(\mathcal{F},\{X\})$ can be embedded
in $(\Q,\V_X)$  is indeed more stringent than the one in Theorem~\ref{thm:promiseWeakMax} for $\V_X$ and $O=\{X\}$). Therefore, let us assume
that we get an assignment $\theta_X$ with an associated weight $z$.
The variable $X$ is then deleted from the constraint formula and from the views, and the original constraint database is modified so as to keep
only assignments where $X$ is fixed to $\theta_X$.

The process is then iterated over all the variables.
It can be shown that the promise is disproved if in some subsequent step \algMax\ does not return the same weight $z$, or if at the end of the
computation the assignment we have computed, say $\theta$, is not a solution or $\mathcal{F}(\theta)\neq z$. Otherwise, $\theta$ can be
returned as a certified solution to \anymaxref.
\end{proof}

The above is the basis for getting the corresponding tractability result for \topk. Note that, since $\PHI^\onDB$ may have an exponential
number of assignments, tractability of enumerating such assignments means here having algorithms that list them \emph{with polynomial delay}
(WPD): An algorithm $\rm M$ solves WPD a computation problem $\rm P$ if there is a polynomial $p(\cdot)$ such that, for every instance of $\rm
P$ of size $n$,  $\rm M$ discovers whether there are no solutions in time $O(p(n))$; otherwise, it outputs all desired solutions in such a way
that a new solution is computed within $O(p(n))$ time after the previous one. Note that, in general, an algorithm running WPD may well use
exponential time and space.

Note, moreover, that in the result below, when the algorithm discovers that the promise does not hold, i.e., when $\mathcal{F}$ cannot be
embedded in $(\Q,\V)$, then it stops the computation but we are still guaranteed that all solutions returned so far (which might be even
exponentially many) constitute a solution to \topk${'}$, for some $K'\leq K$. That is, the algorithm  computes a (possibly empty) certified
prefix of a solution to {\em \topk}. Again, the proof of the following result is elaborated in the Appendix.

\begin{theorem}\label{thm:nopromiseTOPK}
There is a polynomial-delay algorithm for structured valuation functions that either solves {\em
\topk($\PHI_\mathcal{F},O,\mathrm{DB},\V,\DBW$)}, or disproves that $\mathcal{F}$ can be embedded in $(\Q,\V)$; in the latter case, before
terminating, it computes a (possibly empty) certified prefix of a solution.
\end{theorem}
\begin{proof}[Proof Idea]
The result can be established by exploiting a method proposed by Lawler~\cite{L72} for ranking solutions to discrete optimization problems. In
fact, the method has been already discussed in the context of inference in graphical models~\cite{FD10} and in conjunctive query
evaluation~\cite{KS06}.
Reformulated in the CSP context, for a CSP instance over $n$ variables, the idea is to first compute the optimal solution (w.r.t.~the functions
specified by the user), and then recursively process $n$ constraint databases, obtained as suitable variations of the database at hand where
the current optimal solution is no longer a solution (and no relevant solution is missed). By computing the optimal solution over each of these
new constraint databases, we get $n$ candidate solutions that are progressively accumulated in a priority queue over which operations (e.g.,
retrieving any minimal element) take logarithmic time w.r.t.~its size.
Therefore, even when this structure stores data up to \emph{exponential space} (so that its construction required overall exponential time),
basic operations on it are still feasible in polynomial time. The procedure is repeated until $K$ (or all) solutions are returned.
Thus, whenever the \anymaxref\ problem of computing the optimal solution is feasible in polynomial time (over the 
instances 
generated via this process), we can solve with polynomial delay the \topk\ problem of returning the best $K$-ranked ones.
In fact, we can show that the constraint database can always be updated according to the approach by Lawler~\cite{L72}, while being still in
the position of applying Theorem~\ref{thm:nopromiseMAX}, and hence by iteratively solving \topk.
\end{proof}

\subsection{Results for Evaluation Functions}\label{sec:aggregation}

Our analysis has been conducted so far over \emph{structured} valuation functions, whose syntactic form plays a crucial role with respect to
the existence of an embedding in some tree projection. However, when we focus instead on valuation functions only (built over a single binary
operator $\oplus$) the specific form of the constraint formula should not matter because $\oplus$ is by definition a commutative and
associative operator. Therefore, to deal properly with this setting, we adopt a more semantic approach in which the only sensitive issue is the
existence of a tree projection.
To formalize the result, recall that, for a valuation function $\mathcal{F}$, $\mathit{svf}(\mathcal{F})$ denotes the set of all equivalent
structured valuation functions.

\begin{theorem}\label{thm:aggregation}
Let $\Q$ be a constraint formula, let $O$ be a set of variables, let $\mathcal{F}$ be a valuation function, and let $\V$ be a view set for
$\Q$. Then, the following statements are equivalent:
\begin{itemize}
  \item[(1)] There is a function $\mathcal{F}'\in \mathit{svf}(\mathcal{F})$ such that $(\mathcal{F'},\emptyset)$ can be embedded in
      $(\Q,\V)$;

  \item[(2)] There is a tree projection $\HG_a$ of $\HG_\Q$ w.r.t.~$\HG_\V$ such that:
  \begin{itemize}
  \item[(a)] there is a hyperedge $h\in \edges(\HG_a)$ with $h\supseteq O$;
  \item[(b)] for each weight function $f\in\ranks(\mathcal{F})$, there is a hyperedge $h_f\in \edges(\HG_a)$ such that $h_f\supseteq
      \vars(f)$.
  \end{itemize}
\end{itemize}
\noindent Moreover whenever (2) holds and  the tree projection $\HG_a$ is given, then a function $\mathcal{F'}$ as in (1) can be built in
polynomial time.
\end{theorem}
\begin{proof}
The fact that $(1)\Rightarrow(2)$ is immediate by the definition of embedding. Therefore, let us focus on showing that $(2)\Rightarrow(1)$
holds, too.

Assume that $\HG_a$ is a tree projection of $\HG_\Q$ w.r.t.~$\HG_\V$ satisfying the conditions stated in $(2)$. Consider a join tree
$\JT=(V,E,\chi)$ of $\HG_a$ and let $p_O$ be a vertex in $V$ such that $\chi(p_O)\supseteq h$, which exists by (2).(a). Let us root (to
simplify the exposition below) the tree $\JT$ at $p_O$.
Recall that $\ranks(\mathcal{F})$ is the set of all weight functions occurring in $\mathcal{F}$, and for each $f\in \ranks(\mathcal{F})$, let
$p_f\in V$ be any vertex such that $\chi(p_f)= h_f$, which exists by (2).(b).
Based on $\JT$, we build a novel join tree $\bJT=(\bar V,\bar E,\bar \chi)$ as follows: for each function $f\in \ranks(\mathcal{F})$, we create
a new node $\bar p_f$ such that $\bar \chi(\bar p_f)=\chi(p_f)$ and we add this node in $\bJT$ as a child of $p_f$. Moreover, we create a new
node $\bar p_O$ whose only child is $p_O$ and such that $\bar \chi(\bar p_O)=\chi(p_O)$. This new node will act as the root of $\bJT$. All the
other nodes, edges, and labeling remain the same as in $\JT$.
Finally, we process $\bJT$ in order to make it binary. To this end, if a vertex $p$ in $\bJT$ has children $c_1,...,c_n$ with $n>2$, then we
modify $\bJT$ by removing the edges connecting $p$ and $c_i$, for each $i\in \{2,...,n\}$, by adding a novel vertex $p'$ as a child of $p$ and
by appending $c_2,...,c_n$ as children of $p'$. The label of $p'$ is defined as the label of $p$, so that the connectedness condition still
holds on the modified join tree. The transformation is repeated till $\bJT$ is made binary.

Given the join tree $\bJT$, consider the following algorithm that recursively builds $\mathcal{F}(\bJT)$. Let $p$ be the vertex that is the
closest to the root of $\bJT$ and such that $p$ has two children $c_1$ and $c_2$ and the subtrees rooted at them each contains a vertex of the
form $\bar p_f$, for some weight function $f\in \ranks(\mathcal{F})$. Note that, since $\bJT$ is binary, either the vertex $p$ is univocally
determined or it does not exist at all. In particular, in this latter case, let $f$ be the only weight function such that $\bar p_f$ occurs in
$\bJT$ (w.l.o.g., the function contains at least one weight function, and this will be recursively guaranteed) and define $\mathcal{F}(\bar
\JT):=f$. In the former case, define $\mathcal{F}(\bJT):=\mathcal{F}(\bJT_1) \oplus \mathcal{F}(\bJT_2)$, where $\JT_1$ and $\JT_2$ are the
trees rooted at $c_1$ and $c_2$, respectively.

Note that $\mathcal{F}(\bJT)$ clearly belongs to $\mathit{svf}(\mathcal{F})$. Moreover, consider the function $\xi$ such that $\xi(O)=\bar
p_O$; $\xi(f)=\bar p_f$, for each $f\in \ranks(\mathcal{F})$; and, for each internal node $v$ of the parse tree, $\xi(v)$ is mapped to the node
$p$ selected in the above algorithm when processing the subexpression corresponding to the subtree rooted at $v$. Note that $\xi$ is injective,
and by the recursive construction, for each pair $v,v'$ of distinct vertices in the image of $\xi$, $v$ is a descendant of $v'$ in $\bJT$ if
and only if $\xi^{-1}(v)$ is a descendant of $\xi^{-1}(v')$ in ${\it tree}(\mathcal{F}(\bJT),O)$. Then, we can apply
Theorem~\ref{thm:injective} and conclude that an embedding $\xi'$ of $(\mathcal{F}(\bJT),O)$ in $(\Q,\V)$ can be built in polynomial time from
$\xi$.
\end{proof}

Note that, since the ``$(2)\Rightarrow(1)$'' of the above result is constructive, we immediately get the following by
Theorem~\ref{thm:nopromiseMAX} and Theorem~\ref{thm:nopromiseTOPK}.

\begin{cor}
Whenever a tree projection of $\HG_\Q$ w.r.t.~$\HG_\V$ is given, the problems {\em \anymaxref($\PHI_\mathcal{F},O,\mathrm{DB},\V,\DBW$)} and
{\em \topk($\PHI_\mathcal{F},O,\mathrm{DB},\V,\DBW$)} are tractable on classes of constraint optimization formulas $\PHI_{\mathcal{F}}$ where
$\mathcal{F}$ is any valuation function with $|\vars(f)|=1$, for each $f\in \ranks(\mathcal{F})$.
\end{cor}
\begin{proof}
Assume that $|\vars(f)|=1$, for each $f\in \ranks(\mathcal{F})$.  By Theorem~\ref{thm:aggregation}, we can use the given tree projection to
build in polynomial time a function $\mathcal{F'}\in \mathit{ef}(\mathcal{F})$ such that $(\mathcal{F'},\emptyset)$ can be embedded in
$(\Q,\V)$. The result follows by Theorem~\ref{thm:nopromiseMAX} and Theorem~\ref{thm:nopromiseTOPK}.
\end{proof}

Interestingly, if a tree projection is not given, then we are still able to end up with a useful result. Indeed, we can provide a
fixed-parameter polynomial-time algorithm, where \emph{the parameter is the size of the valuation function}, measured as the number of
occurrences of weight functions. This algorithm may be useful in those applications where the number of weight functions is small, while the
number of constraints is large (as it is often the case in CSPs).
Note that this fixed-parameter tractability result should not be confused with tractability results where the parameter is the size of the
whole constraint formula (equivalently, the size of the hypergraph to be decomposed), which may be useful only when the instance consists of a
few constraints only. In such cases, however, the results presented in this paper are not needed, because if the parameter is the size of the
constraint formula, then one can compute in fixed-parameter polynomial-time a tree projection~\cite{GS10}, and then use known techniques for
computing optimal solutions on acyclic instances.

\begin{theorem}\label{thm:nopromiseAggrMAX}
Consider the problem {\em \anymaxref($\PHI_\mathcal{F},O,\mathrm{DB},\V,\DBW$)} over valuation functions and parameterized by the size of such
valuation functions. Then, there is a fixed-parameter polynomial-time algorithm that either solves the problem, 
or disproves that there exists some $\mathcal{F}'\in \mathit{svf}(\mathcal{F})$ that can be  embedded in $(\PHI,\V)$.
\end{theorem}
\begin{proof} Consider the following algorithm: For any $\mathcal{F'}\in
\mathit{svf}(\mathcal{F})$, call the algorithm of Theorem~\ref{thm:nopromiseMAX}. As soon as some invocation does not disprove the promise that
$\mathcal{F}'$ can be  embedded in $(\PHI,\V)$, then we can return the answer we have obtained.
To conclude, observe that we perform at most $|\mathit{svf}(\mathcal{F})|$ iterations, and that $|\mathit{svf}(\mathcal{F})|$ depends only on
the number of weight functions occurring in $\mathcal{F}$.
\end{proof}

From the above theorem, we obtain the corresponding tractability result for \topk, as in the proof of Theorem~\ref{thm:nopromiseTOPK}. In
particular, because we may ask for an exponential number of solutions, fixed-parameter tractability means here having a promise-free algorithm
that computes the desired output {\em with fixed-parameter polynomial-delay}: The first solution is computed  in fixed-parameter
polynomial-time, and any other solution is computed within fixed-parameter polynomial-time from the previous one.

\begin{theorem}
Consider the problem {\em \topk($\PHI_\mathcal{F},O,\mathrm{DB},\V,\DBW$)} over valuation functions and parameterized by the size of such
valuation functions. Then, there is a fixed-parameter polynomial-delay algorithm that either solves the problem, or disproves that there exists
some $\mathcal{F}'\in \mathit{svf}(\mathcal{F})$ that can be  embedded in $(\PHI,\V)$; in the latter case, before terminating, it computes a
(possibly empty) certified prefix of a solution.
\end{theorem}

\section{Parallel Algorithms}\label{sec:parallel}

In the previous sections we have seen that the desired best solutions can be computed by just enforcing local consistency, without any explicit
computation of the equivalent sandwich acyclic instance associated with some tree projection of the given CSP. However, in most practical
applications, in particular when constraint relations contain a large number of tuples of allowed values (with respect to the number of
constraints), having such an acyclic instance allows us to solve the given instance much more efficiently.

In particular, without using acyclicity, we need a quadratic number of operations for enforcing local consistency between each pair of
constraints, contrasted with a linear number of such operations if the procedure is guided by any join tree of the given instance. From a
computational complexity point of view,~\cite{kasif1990} proved that even establishing arc-consistency is $\Pol$-complete and hence not
parallelizable, while~\cite{gott-etal-01} showed that evaluating Boolean acyclic instances is $\LCFL$-complete, hence inside $\Pol$. Combining
the latter result with the techniques of~\cite{gott-etal-02}, it can be seen easily that even establishing global consistency in acyclic
instances is in (functional) $\LCFL$. It is known that all problems in this class are highly parallelizable. Indeed, any problem in \LCFL\ is
solvable in logarithmic time by a concurrent-read concurrent-write {\em parallel random access machine} (CRCW PRAM) with a polynomial number of
processors, or in $\log^2$-time by an exclusive-read exclusive-write (EREW) PRAM with a polynomial number of processors.

In this section we provide parallel algorithms for the computation of the best solutions over a set of output variables $O$ of a given
constraint optimization formula $\PHI'_\mathcal{F}$ over a constraint database $\DB'$, where $\mathcal{F}$ is a valuation function.
These algorithms are significant extensions of a parallel algorithm originally given in~\cite{gott-etal-98tr}.

We assume that a tree projection $\HG_a$ of  $(\Q',\V)$ (for some set of subproblems $\V$) is given in input, too. By the results
in~\cite{GrecoS14}, we assume w.l.o.g. that the number of hyperedges of $\HG_a$ is at most the number of variables occurring in $\Q'$, so that
$\HG_a$ cannot be much larger than the original hypergraph of $\Q'$. Because the valuation function is built with a single operator, say
$\oplus$, the embedding problem is trivial and any tree projection, in particular $\HG_a$, can be used to solve the optimization problem.
Therefore, a sequential algorithm can be obtained easily by using Lemma~\ref{lem:distribution} and a dynamic programming algorithm as
in~\cite{GGS09}. However, it is not a-priori clear how to compute such solutions in parallel with a guaranteed performance that is independent
of the shape of the tree projection at hand, and this is precisely the goal of this section. The algorithms proposed in this section generalize
the DB-SHUNT parallel algorithm for evaluating Boolean conjunctive queries to relational databases presented in~\cite{gott-etal-01}, which is
in its turn based on a {\em tree contraction} technique which closely resembles the method of Karp and Ramachandran for the {\em expression
evaluation problem}~\cite{karp-rama-90}.

We refer the interested reader to~\cite{Gent2011}, for  a nice review of the literature on parallel approaches to solving CSPs.  In particular,
the importance of enforcing consistency (at different levels) in CSP solvers is pointed out. There are different models of parallelism, some
more suited to distributed computations (useful when nodes are associated with difficult subproblems), and others to specialized parallel
machines (see, e.g., Valiant's bulk synchronous parallel (BSP) and multi-BSP computational models~\cite{Valiant2011}, which can simulate the
PRAM model, or the Immediate Concurrent Execution (ICE) abstraction, described in~\cite{Vishkin2011} for a general-purpose many-core explicit
multi-threaded (XMT) computer architecture).

In this paper, we abstract from low-level details of the model and, following~\cite{gott-etal-01}, we assume a shared-memory parallel EREW
DB-machine, where relational algebra operations are machine primitives. In one parallel step of the DB-machine, each machine's processor can
perform a constant number of relational algebra operations on the database relations at hand. Transformations of constraint scopes and input
different from constraint relations are considered costless as long as they are polynomial.\footnote{In fact, such further operations occurring
in the  proposed algorithms are feasible in logarithmic space and thus are parallelizable, in their turn.}
%
The efficiency of a computation of the machine will be measured according to the following cost parameters: (a) the number of parallel steps;
(b) the number of relational processors employed in each step, i.e., processors able to perform operations of relational algebra and related
operations such as relational assignment statements (e.g., $r:=s$ where $r$ and $s$ are relation variables); (c) the size of the working
constraint database. We refer the interested reader to~\cite{AfratiJRSU14} for a connection of our approach (called ACQ there) with Valiant's
BSP and in particular to Map-Reduce models.

\subsection{Enforcing Global Consistency in Parallel Given a Tree Projection}\label{sec:fullreducer}

We first focus on the problem, of independent interest, of computing  the (maximal) globally consistent sub-instance of the given constraint
optimization formula. Recall that a CSP instance $\cJ$ with $m$ constraints is said {\em globally consistent}, or equivalently $m$-wise
consistent, or R$(*,m)$C-consistent, if, for every constraint $C$ occurring in $\cJ$, every tuple $t$ of values in its constraint relation can
be extended to a full solution of $\cJ$. More formally, there exists a satisfying assignment $\theta$ for $\cJ$ such that $t$ is given by
$\theta$ applied (or restricted) to the variables in the scope of $C$.

First observe that, by using the given tree projection $\HG_a$ and the DB-SHUNT algorithm described in~\cite{gott-etal-01}, we can compute
easily with a logarithmic number of parallel steps an equivalent acyclic instance having the same solutions of the original one. Just consider
every hyperedge $h$ of $\HG_a$ and the constraints having scopes $s\subseteq h$ as a single acyclic instance (think of a join tree with $h$ as
its root label and all such constraints as direct children). We thus assume in the following that such an acyclic instance $\cJ=(\PHI,O,\DB)$
equivalent to the original instance has been already computed from $\PHI'$, with $\PHI$  being an acyclic constraint formula, and $\DB$ its
constraint database.
We will omit the specification of the output variables $O$ if it is understood that we are interested in solutions over all variables.
W.l.o.g., every atom occurring in $\PHI$ is constant-free and does not contain any pair of variables with the same name.

Moreover, for an instance $\cJ$ as above, the algorithms use an additional tree $T$, called an {\em e-Join Tree} for $\cJ$, defined as follows:
each vertex $p$ of $T$ is a constraint occurring in $\cJ$, whose scope and constraint relation are denoted, respectively, by $\sch(p)$ and
$\rel(p)$; each constraint in $\PHI$ occurs as some vertex in $T$; for each pair of vertices $p,p'$ in $T$, the variables in
$\sch(p)\cap\sch(p')$ occur in the scopes of all vertices in the path connecting $p$ and $p'$ in $T$. Moreover, the variables of $\sch(p)$ are
partitioned in two distinguished sets $\ret(p)$ and $\iet(p)$.

Note that the above tree is essentially a join tree of the given acyclic instance. By using known results on join trees~\cite{gott-etal-01}, it
easily follows that, given $\cJ$, we can compute in linear time or in logspace (and hence in parallel) an e-Join Tree $T$ for $\cJ$, whose
number of vertices is linear in the number of constraints occurring in $\cJ$, such that: $T$ is a strictly-binary tree, i.e., every non-leaf
vertex has precisely two children; and for every vertex $p$,  $\iet(p) = \emptyset$ (hence, $\ret(p)$ is equal to the original constraint scope
occurring in $\cJ$).

We assume that such a computation has already been done and we next focus on the evaluation of the CSP instance $\cJ$ with such an e-Join Tree
$T$ at hands. We remark that, from the resulting e-Join Tree $T$, we can immediately get, with at most one additional parallel step, the
desired globally consistent constraint occurring in the original (possibly non acyclic) instance $(\Q',\DB')$.

We next describe an algorithm that computes with (at most) a logarithmic number of parallel relational operations the (unique) maximal
sub-instance of $\cJ$ that is globally consistent, which is hereafter called the {\em globally-consistent reduct} of $\cJ$. Note that the naive
parallel algorithm uses a linear number of parallel operations that depends on the tree-shape and that, in general, there is no guarantee that
any balanced or similar good-shaped join tree exists for the given instance. Think, e.g., of a given CSP instance whose associated hypergraph
is just a line.

\begin{figure}[t]
\center
\input{shunt2.eepic}
\caption{{\em Shunt} operation applied to leaf $\bl$}\label{f-shunt}
\end{figure}
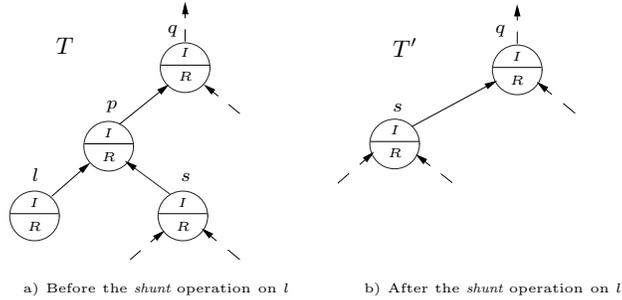

\begin{figure}[t]

\center
\input{genshunt.eepic.tex}
\caption{{Shunt} operation applied to leaves 1,3 and 7 in parallel}\label{genshunt}
\end{figure}
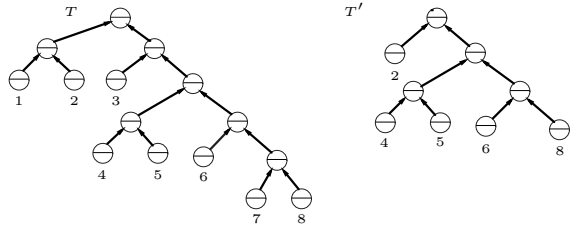

Our algorithm transforms the e-join tree in stages, in such a way that the $n$-vertices tree $T$ is contracted into a $3$-vertices one in
$\lceil log\ n \rceil - 1$ stages. At each stage, a local operation, called {\em shunt}, is applied in parallel to half of the leaves of $T$.
Let $\bl$ be a leaf of an e-join tree $T$, $p$ the parent of $\bl$, $s$ the other child of $p$, and $q$ the parent of $p$ (see Figure
\ref{f-shunt}). The shunt operation applied to $\bl$ results in a new contracted tree $T'$ in which $\bl$ and $p$ are deleted, $s$ is suitably
transformed into a fresh constraint and takes the place of vertex $p$ (i.e., it becomes child of $q$).
Intuitively, as $p$ is deleted, the variables occurring in both $p$ and $q$ must be kept in the new e-join tree, in order to guarantee the
soundness of the procedure. To this end, whenever such variables do not occur in $s$, they are added to the scope of $s$ (precisely, in
$\ietp(s)$). This way, after the application of shunt, by means of $\ietp(s)$, the new constraint at $s$ stores the ``witnesses'' of the
constraint tuples from $p$ that are consistent with both $s$ and $l$, and that are relevant to extend solutions to $q$ (and up in the tree).

The leaves of $T$  are numbered from left to right. At each iteration, the shunt operation is applied to the odd numbered leaves of $T$ in a
parallel fashion; to avoid concurrent changes on the same constraint, left and right leaves are processed in two distinct steps (Figure
\ref{genshunt}). Thus, after each iteration the number of leaves is halved, and the tree-contraction ends within the desired bound.

To compute the globally-consistent reduct in parallel, we employ a two-phase technique. In the ascending phase, a shunt based procedure
contracts the e-join tree; in the descending phase, the original tree shape is rebuilt by applying a sort of reverse shunt operation.

In the ascending phase, unlike the classical tree-contraction algorithm~\cite{gott-etal-01}, no vertex is deleted but all processed vertices
are modified and marked, for a subsequent use in the descending phase. A marked vertex is logically deleted for the ascending procedure and no
shunt operation can be applied on it any more (during this phase). The mark is just a natural number encoding the step of the procedure that
marked that vertex.

\begin{figure}[t]
\center
\input{shuntw.eepic.tex}
\caption{Execution of the {\em shunt}$(w)$ operation on $l$}\label{f-shuntw}
\end{figure}
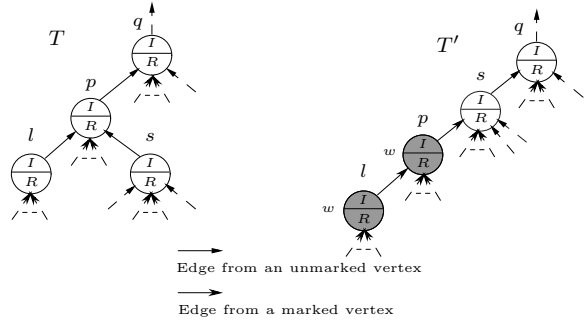

In our algorithm, a shunt operation performed at step $w$ on the (current) e-join tree $T$, denoted {\em shunt(w)}, proceeds as follows. We say
that a leaf of $T$ is unmarked (at step $w$) if it is an unmarked vertex having no unmarked child (hence, such a vertex it is not an actual
leaf of the current tree, though it is a leaf of the sub-tree  of $T$ induced by its unmarked vertices). Let $l$ be an unmarked leaf of $T$,
$p$ the parent of $l$, $s$ the other unmarked child of $p$, and $q$ the parent of $p$ (as in Figure \ref{f-shuntw}). The shunt operation
applied to $l$ results in a new e-join tree $T'$ in which $l$ and $p$ are marked with $w$, $s$ takes the place of vertex $p$ (i.e., it becomes
a child of $q$ in $T'$), and $p$ becomes a child of $s$. The scopes of $p$ and $l$ remain unchanged in $T'$; while the scope of $s$ is
transformed as follows:
\[
\retp({s})=\ret({s})\quad\quad\quad\quad
\ietp({s})= (\ret({q}) \cap \sch({p})) \setminus \ret({s})
\]
The constraint relations of $p$, $s$, and $l$ in $T'$ are just the projections over their scopes of the solutions of the subproblem comprising
the constraints $l$, $p$, and $s$, i.e., in the relational framework, the projection of the relation
 $C =  \rel(l) \join (\rel(p) \ltimes \rel(q)) \join \rel(s)$.

Additionally, the shunt operation produces and stores an updated version of vertex $s$ of $T$, named $s_w$, that will be used in the descending
phase to rebuild possible $p$-$s$ relationships coming from attributes in $\iet(s)$. This new constraint $s_w$ is stored into an additional
storage (it is not in the e-join tree), its scope is $\sch(s_w)=\sch(s)$ (the same as the old $s$) and its constraint relation $\rel(s_w)$ is
the projection over this scope of the same relation $C$.

As soon as the ascending phase is terminated and the tree is contracted to a tree of depth 1, a descending phase gets started which re-expands
the e-join tree by applying in parallel a reverse shunt operation.

The {\em reverse shunt} ({\em r-shunt}) operation unmarks (and updates) the vertices having the highest mark, say, $w$ in the e-join tree. Let
$l$ and $p$ be two vertices with mark $w$ of an e-join tree $T$, such that $p$ is the parent of $l$ in $T$ (note that $l$ and $p$ have been
necessarily marked at the same step of the ascending procedure). Moreover, let $s$ be the (unmarked) parent of $p$, and $q$, in turn, the
(unmarked) parent of $s$ (see Figure \ref{f-rshuntw}). The {\em r-shunt(w)} operation applied on $l$ and $p$ results in a new e-join tree $T'$
in which the marks of $l$ and $p$ have been removed, $p$ takes the place of $s$ as a child of $q$, and $s$ becomes a child of $p$. The scopes
and the constraint relations associated to the vertices $p$ and $l$ in $T'$ are the following:
\[
\retp(p) = \ret(p);\quad\quad \ietp(p) = \emptyset;
\quad\quad(\schp(p) = \ret(p);)
\]
\[
\retp(l) = \ret(l);\quad\quad \ietp(l) = \emptyset;
\quad\quad(\schp(l) = \ret(l);)
\]
\[ \relp(p) = \prod_{\ret(p)}\left ( \rel(q) \rtimes \rel(p) \ltimes
(Rel(s_w) \ltimes \rel(s)) \right ) \]
\[ \relp(l) = \prod_{\ret(l)} \left ( \relp(p)\rtimes \rel(l) \right )
\]

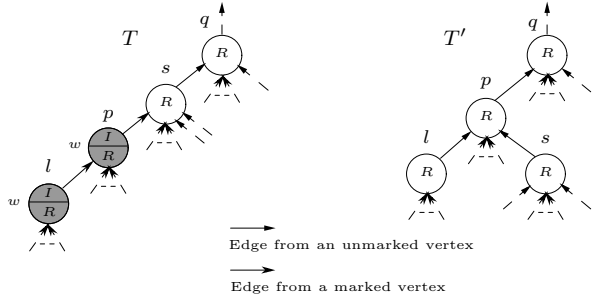
\begin{figure}
\center
\input{rshuntw.eepic.tex}
\caption{Execution of the {\em r-shunt}$(w)$ operation on vertices $l$ and $p$}
\label{f-rshuntw}
\end{figure}

Intuitively, the new constraint for $p$ is computed from its ``frozen'' version (which was marked at step $w$ of the ascending phase) by
enforcing consistency with both $q$ and $s$. Actually, concerning $s$, $p$ is made consistent with the constraint relation $\rel(s_w) \ltimes
\rel(s)$, because $s_w$ stores the variables shared by $p$ and $s$ in the ascending phase that were in $\iet(s)$ and are not present in the
current scope of $s$. The relation for $l$ is then obtained by enforcing consistency with the new $p$ relation (note that the scope of $p$
contains all variables that $l$ shares with other unmarked vertices of the e-join tree).

\begin{figure}[t!]
\centering
\fbox{\parbox{\textwidth}{
\begin{tabbing}
{\bf Input}: An acyclic CSP instance $(\PHI,{\DB})$ and an e-Join Tree $T$.\\
{\bf Output}: The globally-consistent reduct of $(\PHI,{\DB})$.\\
\vspace{0.2 cm}
{\bf begin}\=\\
(1)\> Let $\lambda$ be the number of leaves in $T$; \\
(2) \> Label the leaves of $T$ in order from left to right as
$1,\ldots,\lambda$;\\
(3) \> $mark := 0;$\\
\> (* Ascending Phase *)\\
(4) \> {\bf while}\= \ {\em depth}($T$)$>1$ {\bf do}\\
\> (a) \> $mark := mark+1;$\\
\> (b) \> {\bf in parallel} apply the {\em shunt(mark)} operation
      to all unmarked odd leaves  \\
\> \> that are the left children of their parent, and \\
\> \> that have depth
greater than $1$; \\
\> (c) \> $mark := mark +1;$\\
\> (d) \> {\bf in parallel} apply the {\em shunt(mark)} operation
      to all unmarked odd leaves  \\
\> \> that are the right children of their parent, and \\
\> \> that have depth
greater than $1$; \\
\> (e) \> shift out the rightmost bit in the labels of all
remaining unmarked leaves;\\
   \>  {\bf end} (* while *) \\
\> Let $p$ be the root of $T$, and let $p'$ and $p''$ be the
   children of $p$ in $T$ \\
(5) \> $\rel(p) := \rel(p')\rtimes \rel(p) \ltimes \rel(p'');$\\
\> (* Descending Phase *)\\
(6)  \> $\rel(p') := \prod_{\ret(p')}(\rel(p)\rtimes \rel(p'));
        \ \ \ \ \iet(p') := \emptyset;$\\
    \> $\rel(p''):= \prod_{\ret(p'')} (\rel(p)\rtimes \rel(p''));
        \ \ \ \iet(p''):= \emptyset;$\\
(7) \> {\bf for} \ $w := mark$ {\bf downto} 1 {\bf do}\\
\> (a) \> {\bf in parallel} apply {\em r-shunt(w)}
 to all pairs of adjacent vertices with mark $w$ \\
   \>  {\bf end} (* for *) \\
(8) \> {\bf output} all constraints stored in the vertices of $T$.\\
{\bf end}.
\end{tabbing}}}
\caption{GLOBAL-CONSISTENCY ALGORITHM.}
\label{f-fr-algo}
\end{figure}

The GLOBAL-CONSISTENCY algorithm is shown in Figure \ref{f-fr-algo}. The algorithm consists of two main phases. The ascending phase is similar
to DB-SHUNT~\cite{gott-etal-01}: at each stage of the {\bf while} loop performed in this phase, a shunt operation is applied to half of the
unmarked leaves which get marked. After $\lceil log\ n \rceil - 1$ {\bf while} iterations only three vertices remain unmarked. Then,
instruction (5) computes the relation associated to the root of the tree, which is the final relation for this vertex and will be returned in
the output database. Instruction (6) begins the descending phase, by fixing the values for the relations associated to the two children of the
root. The descending phase ``propagates'' the correct relation values from the unmarked vertices to the marked vertices. The {\bf for} loop
unmarks the vertices marked during the ascending phase. In particular, an execution of instruction (7.a) unmarks all vertices marked by an
execution of either instruction (4.b) or (4.d) of the ascending phase. Once a vertex is unmarked, its relation is fixed to the final value.
Upon unmarking, the vertex is located exactly in the same tree position it had in the ascending phase (at the time when it was marked).
Unmarking is performed by the r-shunt operations; each of them applies to a vertex $l$ and to its parent $p$, both having the highest mark $w$
in the current e-join tree. The relation for $p$ is computed by eliminating from its old value all tuples which do not agree to its (unmarked)
adjacent vertices. The relation for $l$ is then obtained by a semi-join to its parent $p$. At the end of the {\bf for} loop all vertices are
unmarked and the associated constraints are returned.

Given an e-join tree $T$, we denote by $verts(T)$ the set of the vertices of $T$, and by $verts(T,s)$ the subtree rooted at vertex $s$ of $T$.
Moreover, define 
${\it Mverts}(T,s)$ as the set of vertices in the subtrees of $T$ rooted at the marked children of $s$.
%
Given a subproblem of the given CSP instance encoded as a set of vertices/constraints $V$ of the e-join tree $T$, we will denote by $\join V$
(or just $V$, if no confusion arises) the set of all its satisfying assignments,  obtained as the (natural) join of all the constraint
relations occurring in $V$.

\begin{theorem}\label{t:fr-sound}
The GLOBAL-CONSISTENCY algorithm is correct.
\end{theorem}
\begin{proof}
Given an acyclic  CSP instance $\cJ$, we denote by $cr(\cJ)$ the globally-consistent reduct of  $\cJ$. Moreover, given a constraint $p$,
$cr(p,\cJ)$ is the constraint occurring in the globally-consistent reduct, also called the consistent reduct of $p$ w.r.t. $\cJ$.
Importantly, note that, by construction, the e-join trees keep the peculiar connectedness property of ordinary join trees during the
computation.
By the same arguments used for proving the soundness of DB-shunt~\cite{gott-etal-01}, it turns out  that at the end of the ascending phase the
constraint associated with the root $r$ of the e-join tree $T$ is precisely $cr(r,\cJ)$.

Observe that, at each step of the computation, $cr(p,\cJ)$ is equal to the projection on $\ret(p)$ of the join of the constraint relations of
all (marked and unmarked) vertices of the e-join tree at hand, i.e. $\join verts(T)$. Indeed, this property clearly holds at the beginning of
the computation by definition of globally-consistent reduct, and it is easy to verify that the join of all constraint relations corresponding
to the nodes of the e-join trees is an invariant of the algorithm. In fact, a tuple is discarded from a constraint relation at some step {\em
only if} it is inconsistent with some other constraint.

By using an inductive argument we prove that,  at any step of the descending phase, for each unmarked vertex $p$ of the current e-join tree
$T$, its constraint relation is equal to $cr(p,\cJ)$. ({\em Basis.}) We already pointed out that, at the end of the ascending phase, the root
$r$ holds $cr(r,\cJ)$. Moreover, it can be shown that, after the execution of instruction (6), $\rel(p) = cr(p,\cJ)$ holds for each child $p$
of the root. That is, after the execution of instruction (6) (the basis step of the descending phase), the root and its two children, which are
the only unmarked vertices in the current e-join tree, contain precisely their respective consistent reducts for $\cJ$.

({\em Induction step.}) Assume the property holds for the e-join tree $T$ computed after some executions of the loop at step 7. Consider one of
the parallel steps executed on $T$ by instruction (7.a), say, the execution of {\em r-shunt}$(w)$ on vertices $l$ and $p$, where $p$ is the
parent of $l$, and both vertices have the highest mark $w$ of $T$. Let $s$ be the parent of $p$, and $q$ the parent of $s$, as shown in
Figure~\ref{f-rshuntw}. By the induction hypothesis, for the unmarked vertices $s$ and $q$, $\rel(s)=cr(s,\cJ)$ and $\rel(q)=cr(q,\cJ)$. By
applying {\em r-shunt}$(w)$ on $l$, we get a new e-join tree $T'$ where $p$ replaces $s$ as a child of $q$ and $s$ becomes a child of $p$ (the
other one is $l$).
First, it can be shown that $\relp(p)=cr(p,\cJ)$. To this end, note that the consistent reduct of $p$ can be obtained as $cr(p,\cJ) =
\prod_{\ret(p)} (UT_q \join \rel(q)\join T_s)$, where $UT_q= verts(T)\setminus verts(T_s)$. Moreover, we use the fact that all variables
occurring in both $UT_q$ and $T_{s}$ belong to $\sch(q)$, and that
 $\rel(q)$ is the consistent reduct of $q$ and hence it is clearly consistent with $\join UT_q$. It follows that $cr(p,\cJ) = \prod_{\ret(p)}
(\rel(q)\join T_s)$. By means of a further elaboration on this expression, using a similar (though more involved) argument,
 we get  $cr(p,\cJ) = \prod_{\ret(p)} (\rel(q)\join \rel(p) \join SP\join\rel(s) )$, where $SP=verts(MT_s)\setminus verts(T_p)$. Recall that,
 at the execution of the {\em shunt}$(w)$, we stored a relation $Rel(s_w)$, which can be shown to be consistent with the full
subproblem encoded by $SP$. Moreover, $Sch(s_w)\subseteq (\sch(p)\cup\sch(s))$ by construction, so that we can prove $cr(p,\cJ) =
\prod_{\ret(p)} (\rel(q)\join \rel(p) \join\rel(s)\join \rel(s_w) )$. Because $\sch(s)\subseteq Sch(s_w)$ and $(\sch(q)\cap
Sch(s_w))\subseteq\sch(p)$,
 this constraint relation is equal to the result of the semi-join based computation in {\em r-shunt}$(w)$. By similar arguments, we get that
 {\em r-shunt}$(w)$ gives to us the constraint reduct for node $l$, too.
\end{proof}

\begin{theorem}\label{t-compl-fr}
On a parallel DB-machine with $c$ processors, given a CSP instance $\cJ$ with an e-join tree for it having $n$ vertices, the GLOBAL-CONSISTENCY
algorithm performs
\begin{itemize}
\item[(a)] a sequence of at most $4(\lceil \log\ c \rceil + 2\lceil n/4c \rceil )$ parallel shunt operations;
\item[(b)] by using $O(n)$ intermediate relations, having size $O(d^2)$, where $d$ is the size of the largest constraint relation.
\end{itemize}
\end{theorem}
\begin{proof}
Assume first that the number of processors $c$ is half the number of leaves of the e-join tree, which is known to be $\lceil n/2 \rceil$,
because it is a strictly-binary tree. Then each parallel operation can be executed by a different processor, and the ascending phase of
GLOBAL-CONSISTENCY performs $2 \lceil \log n/4 \rceil$ shunt operations, because after each run of instructions 4a and 4b the number of leaves
is reduced by half.

For the general case with $c < \lceil n/4 \rceil$: at the first iteration of the loop, we can equally divide the parallel shunt operations
among the $c$ processors, which means $2\lceil n/4c \rceil$ operations. In the subsequent steps, the number of leaves (hence of needed
processors) halves at each step. Then, to complete the loop, we need $2(2\lceil n/4c \rceil -1)$ operations until  $c = \lceil n/4 \rceil$,
followed by further $2 \lceil \log c \rceil$ operations, as above.

Moreover, we note that this phase uses $O(n)$ intermediate relations stored in the e-join tree, and we claim that the size of each such
relation requires at most $O(d^2)$ space.

The claim is proven by induction, assuming that for each vertex $v$, at each step, both $\size{\prod_{\ret(v)} \rel(v) }\leq d$ and
$\size{\prod_{\iet(v)} \rel(v) }\leq d$ hold. The claim trivially holds when the ascending phase starts. Consider the generic step depicted in
Figure~\ref{f-shunt} where, after the shunt operation, the vertex $s$, which is a child of vertex $p$ in the current tree $T$, becomes a child
of vertex $q$ in the new tree $T'$. Recall that, by definition, $\ietp(s)\subseteq \ret(q)$. Because of the semijoin operation in the
computation of the relation $C$ and the fact that, by induction, $\size{\prod_{\ret(q)} \rel(q) }\leq d$, we get $ d \geq \size{\prod_{\ret(q)}
C } \geq \size{\prod_{\ietp(s)} C }$. Moreover, $\size{\prod_{\ret(s)} \rel(s) }$ is monotonically decreasing during the procedure. Then, the
size of the relation $\rel(s)$ in the e-join tree $T'$ is at most $d^2$, as it is at most the cartesian product $\size{\prod_{\ietp(s)} C
\times \prod_{\retp(s)} \rel(s) } \leq d^2$. Note that all other vertices in $T'$ have the same schema they have in $T$, and thus their stored
relations, after the additional join operations computed in $C$, may only lose tuples with respect to $T$.
For the sake of completeness, we note that the computation of the whole relation $C$ is in fact not necessary to obtain the new relations in
$T'$. Indeed, note that $T$ encodes an acyclic query, so that the new relations for the vertices $l, p$, and $s_w$ can be computed easily in
linear space by using the classical semijoin algorithm for enforcing global consistency over the acyclic subquery induced by vertices $l, p$,
$q$, and $s$. After this procedure, the new relation for vertex $s$ in $T'$ can be computed by evaluating $\size{\prod_{\schp(s)} \rel(p) \join
\rel(s) }$.

The descending phase is even more efficient than the ascending phase. Indeed, the relation sizes decrease monotonically in the descending phase
(only semi-joins and projections are applied). With respect to the constraints appearing in the given e-join tree, the algorithm uses
additional intermediate relations to store the $s_w$ vertices; however, since there is only one additional relation for each shunt operation
performed in the ascending phase, the total number of intermediate relations is still linear in the size of the e-join tree (and hence in the
number of constraints of the given instance).
Moreover, the number of (parallel) steps performed in the descending phase is exactly the same as in the ascending phase. The same number of
processors used in the ascending phase is clearly sufficient to perform the descending phase, and also the number of relational operations is
the same (apart from a constant factor).
\end{proof}

\subsection{Computing $\anymaxref$}\label{sec:ACQ}

\begin{figure}\centering
\fbox{\parbox{\textwidth}{%
\begin{tabbing}
{\bf Input}: An acyclic CSP instance $(\PHI_\mathcal{F},O,{\DB})$ and an e-Join Tree $T$.\\
{\bf Output}: The best solutions of $(\PHI_\mathcal{F},O,{\DB})$.\\
\vspace{0.2 cm}
{\bf begin}\=\\
\> Let $\lambda$ be the number of leaves in $T$; \\
(1) \> Label the leaves of $T$ in order from left to right as $1,\ldots,\lambda$;\\
(2) \> {\bf while}\= \ {\em depth}($T$)$>1$ {\bf do}\\
\> (a) \> {\bf in parallel} apply the {\em shunt} operation
      to all odd numbered leaves that \\
\> \> are the left children of their parent, and have depth greater than $1$; \\
\> (b) \> {\bf in parallel} apply the {\em shunt} operation
      to all odd numbered leaves that \\
\> \> are the right children of their parent, and have depth greater than $1$; \\
\> (c) \> shift out the rightmost bit in the labels of all remaining leaves;\\
   \>  {\bf end} (* while *) \\
\> Let $p$ be the root of $T$, and let $p'$ and $p''$ be
   the children of $p$ in $T$ \\
(3) \> {\bf output}
       $\prod_{O}^\oplus (\rel(p)\join^\oplus \rel(p')\join^\oplus \rel(p''))$\\
{\bf end}.
\end{tabbing}}}
\caption{The parallel algorithm \ACQ.}
\label{f:algo-ACQ}
\end{figure}

Figure~\ref{f:algo-ACQ} shows the parallel algorithm  \ACQ\ for computing the best solutions of a given (already made) acyclic instance
$(\PHI_\mathcal{F},O,\mathrm{DB})$, given an e-join tree $T$ of a tree projection for the given instance, and where  $\mathcal{F}$ is a
valuation function for $\Q$ over $\oplus$. The e-join tree $T$ is assumed to hold a global consistent instance, possibly computed by using the
algorithm described in the previous section. Algorithm  \ACQ\ outputs a relation with the solutions $\theta$ over $O$, each one annotated with
a value $\val(\theta) \in \mathbb{D}$, which is the best possible value that can be obtained by extending  the partial substitution $\theta$ to
a full solution of the given instance. All tuples leading to the maximum value can thus be obtained immediately from this output relation.
Clearly enough, such best solutions will be also the best solutions (w.r.t. $\mathcal{F}$) of the original (possibly non-acyclic) constraint
formula $(\Q',\DB',O)$.

Algorithm \ACQ\ proceeds similarly to the ascending phase of the previous algorithm, but the output variables are suitably preserved during all
the computation, by propagating them towards the root of $T$. To this end, the scope $\sch (p)$ of any vertex $p$ is now partitioned in three
(instead of two) distinguished set of variables $\ret(p)$, $\iet(p)$, and $\oet(p)$, where the new set $\oet(p)$  contains output variables.
For each vertex $p$ of $T$, $\oet(p)$ is initialized to $\emptyset$.

The major novelty is due to the necessity of dealing with the optimization problem, which requires an extension of the classical relational
operators so to manage the values provided by the weighting functions in $\ranks(\mathcal{F})$. Every tuple $\theta$ of every vertex of the
e-join tree $T$ is associated with a value $\val(\theta) \in \mathbb{D}$. Without loss of generality, we assume that $\mathbb{D}$ contains a
neutral value for $\oplus$, denoted by $0$, such that  $a\oplus 0 = 0\oplus a = a$, for all $a\in \mathbb{D}$.\footnote{Indeed, the neutral
element is just used in the proposed algorithm to manage operations involving tuples with no assigned value. It can be easily simulated by
dealing explicitly with such case.}

The  e-join tree $T$ is initialized as follows: Every vertex $s$ whose associated constraint occurs in the constraint formula with some weigh
function $f$, is such that every tuple $\theta\in\rel(s)$ has value  $\val(\theta) = f(\theta)$; for every other relation in the e-join tree,
all tuples have value $0$.

Given two vertices $R_1$ and $R_2$, define the extended join operation $R_1 \bowtie^\oplus R_2$ as the set of tuples
\[   \{ \theta\in R_1 \bowtie R_2, \mbox{ with value }
\val(\theta) =  \max \{\val (\theta')\oplus \val(\theta'') \mid \theta = \theta' \cup \theta'' ,
\theta'\in R_1, \theta''\in R_2\} \},\]

and the extended projection operation $\prod_X^\oplus R$ of a relation $R$ over a set of variables $X$ as
\[ \{ \theta\in \prod_X R,  \mbox{ with value } \val(\theta) =
\max \{\val (\theta')\in R \mid  \theta'[X] = \theta \}. \]

Then, the shunt operation is redefined as follows. Let $\bl$ be a leaf of an e-join tree $T$, $p$ the parent of $\bl$, $s$ the other child of
$p$, and $q$ the parent of $p$. The {\em shunt operation} applied to $\bl$ results in a new contracted e-join tree $T'$ in which $\bl$ and $p$
are deleted, $s$ is transformed (as specified below) and takes the place of vertex $p$ (i.e., it becomes child of $q$). The scope and the
constraint relation of the transformed version of $s$ in $T'$ are specified below.
\[ \retp({s})=\ret({s}) \]
\[ \ietp({s})= (\ret({q}) \cap \sch({p})) \setminus \ret({s}) \]
\[ \oetp({s})= ( (\sch(l)\cup\sch(p)\cup\sch(s))\cap O )
\setminus (\retp(s)\cup\ietp(s)) \]
\[
\relp(s) = \prod^\oplus_{\schp({s})}
( \rel(\bl) \join^\oplus \rel(p) \join^\oplus \rel(s))
\]

Thus, each output variable occurring in a deleted constraint ($p$ or $l$) is kept in $s$ (if it does not belong to $\retp(s)\cup\ietp(s)$, then
it is added to $\oetp({s})$).

\begin{theorem}\label{t-algo-ACQ-sound}
The algorithm \ACQ\ is correct.
\end{theorem}
\begin{proof}
The proof can be derived by a similar line of reasoning as for the previous algorithm; only a few remarks are needed.
By the new definition of shunt, no output variable disappears from the e-join tree during the computation: whenever the last constraint
containing an output variable $X$ is deleted, $X$ is stored in an $\oet$ attribute of another vertex. Thus, when the last iteration of
instruction (3) is executed, all variables in $O$ are still present in the tree (i.e., it holds that $O \subseteq \sch(p) \cup \sch(p') \cup
\sch(p'')$). Moreover, if a variable $X$ occurs in $\oet(p)$ for some vertex $p$ of an e-join tree $T$ generated during the computation of
Algorithm \ACQ, then $X$ does not appear in the scope $\sch(q)$ of any other vertex $q$ of $T$. Thus, $\oet$ variables are not playing as join
attributes in any shunt operation performed in the algorithm. They are used only for storing and preserving, during the computation, the
assignments over output variables that will  eventually be output as solutions.

Finally, at the end of the algorithm, the root holds a relation where  the value  $\val(\theta) \in \mathbb{D}$ of each tuple $\theta$ is the
best possible value that can be obtained by extending  it to a full solution of the given instance. Indeed recall that $\mathcal{F}$ is a
valuation function for $\Q$ over a single operator $\oplus$ so that there is always a natural embedding for any tree projection, with every
vertex that can play the role of a separator.  Then, the statement easily follows by the definition of the extended operators $\bowtie^\oplus$
and $\prod^\oplus$, and by Lemma~\ref{lem:distribution}.
\end{proof}

\begin{theorem}\label{t-compl-acq}
On a parallel DB-machine with $c$ processors, given a globally consistent CSP instance $(\PHI,O,{\DB})$ with an e-join tree for it having $n$
vertices, the best solutions (over the variables $O$) can be computed by performing
\begin{itemize}
\item[(a)] a sequence of at most $2(\lceil \log\ c \rceil + 2\lceil n/4c \rceil )$ parallel shunt operations;
\item[(b)] by using $O(n)$ intermediate relations, having size $O(v d^2)$, where $d$ is the size of the largest constraint relation, and
    $v$ the size of the solutions.
\end{itemize}
\end{theorem}
\begin{proof}
Use Algorithm \ACQ\ on $(\PHI,O,{\DB})$ and its e-join tree: the while loop of instruction $2$ is very similar to the ascending phase of the
previous algorithm, so that it is easy to see that property (a) of the theorem stems from the proof of Theorem~\ref{t-compl-fr}.

Concerning property (b), compare the algorithm for enforcing global consistency in Figure~\ref{f-fr-algo} with \ACQ, where the extra output
variables that are kept in the e-join tree $T$ may increase the size of the intermediate relations. If the cardinality of the set $O$ is
bounded by a fixed constant, such a size is bounded by a polynomial of the input size. However, in general the set $O$ is arbitrary, possibly
the whole set of variables occurring in the CSP instance. Therefore, to get the desired output-polynomial bound (property b), we use the
crucial property that the given instance is global consistent, so that each tuple of values which is stored in the output variables during the
while loop of \ACQ\ will eventually be part of the solutions (no tuple can be deleted during the while loop). It follows that the size of the
(partial) assignments over the output variables in the intermediate relations cannot exceed the size $v$ of the result.
\end{proof}

Because the number of solutions over the output variables $O$ can be exponential with respect to the size of the input, if we are interested in
computing just one best solution, the complexity can be made smaller by using $O=\emptyset$ in \ACQ\ and then building the desired solution
backward with a final top-down step.

\begin{cor}\label{t-compl-max-parallel}
On a parallel DB-machine with $c$ processors, given a globally consistent CSP instance $(\PHI,O,{\DB})$ with an e-join tree for it having $n$
vertices, \underline{\anymaxref}{($\PHI_\mathcal{F},O,\mathrm{DB}$)} can be computed by performing
\begin{itemize}
\item[(a)] a sequence of at most $4(\lceil \log\ c \rceil + 2\lceil n/4c \rceil )$ parallel shunt operations;
\item[(b)] by using $O(n)$ intermediate relations, having size $O(d^2)$, where $d$ is the size of the largest constraint relation.
\end{itemize}
\end{cor}

\section{Further Related Work}\label{sec:related}

In this section, in addition to the references already given, we discuss further literature that is closely related to our research, by
emphasizing our specific modeling choices.

\smallskip

\paragraph{Soft Constraints} Soft constraints are classical constraints~\cite{D03} enriched with the ability of associating either with the entire constraint or with each assignment of its variables a weight (meant to encode, for instance, a level of preference or
a cost).
The use of soft constraints leads to generalizations of the basic CSP setting, such as \emph{fuzzy}~\cite{DFP93},
\emph{probabilistic}~\cite{FL93}, \emph{possibilistic}~\cite{S92}, \emph{partial}~\cite{FW92}, and \emph{lexicographic}~\cite{FLS93} CSPs.
These extensions can be viewed as special instances of the general setting of \emph{semiring}-based CSPs~\cite{BMR97}, where each assignment in
a constraint is associated with a value taken from a domain $\mathbb{D}$ over which a \emph{constraint-semiring}
$\tuple{\mathbb{D},\oplus,\otimes,0,1}$ is defined\footnote{In particular, $\oplus$ is closed, commutative, associative, idempotent, $0$ is its
unit element, and $1$ is its absorbing element; $\otimes$ is closed, commutative, associative, distributes over $\oplus$, $1$ is its unit
element, and $0$ is its  absorbing element.}.
Intuitively, $\otimes$ is a binary operator combining the values associated with the various constraints, while $\oplus$ is a binary operator
inducing a partial order $\succeq_{\oplus}$ over $\mathbb{D}$ such that $a\succeq_{\oplus} b$ if and only if $a=a\oplus b$ holds. The goal is
to find an assignment whose total value is minimal w.r.t.~this order.

In many cases of practical interest, the domain $\mathbb{D}$ is already associated with a total order and, therefore, the semiring-based model
can be reduced to the setting of \emph{valued} CSPs~\cite{SFV95}---see~\cite{BFMRSV96}, for a formal comparison of the two settings.
In a valued CSP, the domain is part of a valuation structure $\tuple{\mathbb{D},\circledast,\geq}$, where $\geq$ is a total order over
$\mathbb{D}$ and where $\circledast$ is a commutative, associative, and monotonic binary operator used (as usual) to combine the values.
For instance, let $k\in \mathbb{N}\cup\{\infty\}$ be an element taken from the set of the natural numbers extended with the positive
infinity\footnote{As usual, it holds that $a+\infty=\infty$ and $\infty\geq a$, for each $a\in \mathbb{N}\cup\{\infty\}$.} and consider the
valuation structure $\tuple{\{0,1,...,k\},\oplus,\geq}$ where $a\oplus b=\min \{k, a+b\}$.
This structure gives rise to the well-known \emph{weighted} CSP setting: For each constraint $(S_v,r_v)$ and for each assignment $\theta_v$
over the variables in $S_v$, if $\theta_v$ is not in $r_v$, then its value is $k$, which basically means that the given partial assignment is
forbidden. Instead, if $\theta_v$ is in $r_v$, then its value is a non-negative cost. The goal is to check whether there is any $\geq$-minimal
assignment whose associated value is different from $\infty$, and to compute one if any.
Despite their simplicity, weighted CSPs are expressive enough to model all valued CSPs over discrete valuation structures, provided that
$\circledast$ has a partial inverse~\cite{C05}.

\medskip

\paragraph{Encodings}
From the above discussion, it is easily seen that \emph{hard} and soft constraints do not need to be explicitly distinguished in the
formalization, since hard constraints can be enforced by just associating an infinite cost to the assignments that are not admissible (and then
looking for solutions with minimum cost).
However, algorithms exploited to process hard constraints often assume a relational representation where only the allowed tuples of values for
each hard constraint are listed, while soft constraints algorithms assume a tabular representation where all possible assignments are listed
 together with their values. Therefore, encoding a hard constraint in terms of a soft one might lead to an exponential blow-up of the size of
its representation.
For an extreme example, consider a constraint $(S_v,r_v)$ such that $r_v$ does not contain any assignment. To encode this (hard) constraint in
terms of a weighted CSP, we have to associate with all possible assignments over $S_v$ the value $\infty$. The number of these assignments is
clearly exponential in the \emph{arity} of $S_v$, i.e., $|S_v|$, and hence representing them in tabular form might quickly become unfeasible.

In fact, most of the works in the literature on soft CSPs do not care about the issue, because a \emph{bounded-arity} setting is (implicitly)
considered, i.e., the size of the largest constraint scope is assumed to be bounded by a fixed constant---when this constant is $2$, then we
get the classical setting of \emph{binary} CSPs~\cite{BCVBW02}.
In the present paper, instead, we have not posed any arity bound on the given constraints, so that it was natural to avoid listing all possible
assignments (cf.~\cite{LD03,KDLD05}). Accordingly, we assumed that the input to our reasoning problems is given by a \emph{standard} CSP
instance $\mathcal{I}$, where only the assignments that are allowed are explicitly represented, plus an optimization function built on top of a
valuation structure (over a set of binary operators) associating a value \emph{only} with the assignments that are contained in some of the
constraint relations of $\mathcal{I}$.

\medskip

\paragraph{Decomposition Methods}
One of the most important and deeply studied island of tractability for standard CSPs is the class of instances whose associated hypergraphs
are {\em acyclic}~\cite{Y81,FMU82,F83}.
Structural decomposition methods are approaches for extending the good results about this class to relevant classes of {\em nearly acyclic}
structures (see, e.g.,~\cite{GLS00,CJG08,GF10} and the references therein). On CSP instances having bounded arity, the \emph{tree
decomposition}~\cite{RS86,DP89,FFG02} emerged to be the most powerful decomposition method~\cite{G07}. Its natural counterpart over arbitrary
instances is the \emph{(generalized) hypertree} decomposition method\cite{GLS02,GMS09}, but it has been observed that this method does not
chart the frontier of tractability---for instance, further classes of tractable instances can be identified via the \emph{fractional hypertree}
method~\cite{GM14,M10} (see also the more recent results in~\cite{DBLP:journals/corr/FischlGP16}).
All these methods (including fractional hypertree decompositions) fit into the framework of the {tree projections}~\cite{GS84,SS93,GMS09},
within which the results derived in the paper have been positioned.

Structural decomposition methods have been shown to play a crucial role even in presence of soft constraints, as introduced above. Indeed, the
tractability of semiring-based and valued CSPs over structures having bounded treewidth has been shown in the literature~\cite{BMR97,TJ03}, and
the effectiveness of the solution approaches has been practically validated, too  (see, e.g.,~\cite{JT03}). Actually, in a variety of automated
reasoning areas, similar solution algorithms over structures having bounded treewidth have been proposed over the years. A unifying perspective
of all of them has been provided by~\cite{KDLD05}, where the concept of \emph{graphical model} is introduced (essentially capturing a valuation
structure) and where the bucket-tree elimination algorithm has been introduced to solve instances of this model having a tree-like structure.
In~\cite{KDLD05} it has been also shown how these results can be extended if the notion of hypertree decomposition is used in place of the
notion of tree decomposition. However, for this extension, the given valuation structure has to satisfy certain technical
conditions\footnote{Values are real numbers, with 0 meant to encode that an assignment is not allowed/desirable. The binary operator
$\circledast$ must be absorbing relative to 0, i.e., $a\circledast 0=0$. For example, multiplication has this property while summation has
not.}, which reduces its range of applicability.

Efficient solution algorithms that work on instances having bounded hypertree width and with arbitrary valuation structures have been more
recently proposed by~\cite{GS11,GG13}.
In the present paper, we further generalized these algorithms to deal with complex optimization functions where more than just one aggregation
operator is allowed. Moreover, unlike all the references reported above, we have not assumed that a decomposition is given (except for
Section~\ref{sec:parallel}), which is a useful assumption when computing a decomposition is an $\NP$-hard problem~\cite{GMS09}. This has been
obtained by designing promise-free algorithms that either compute correct certified solutions, or disprove the (promised) existence of a
decomposition, in the spirit of~\cite{CD05}. Both generalizations are non trivial, and novel technical machineries have been required.

\medskip

\paragraph{Constraint Propagation}
From a technical viewpoint, our algorithms are based on procedures enforcing pairwise consistency~\cite{BFMY83}, also  known in the CSP
community as {\em relational arc consistency} (or arc consistency on the dual graph)~\cite{D03}, {\em 2-wise consistency}~\cite{G86}, and
$R(*,2)C$~\cite{KWRCB10}, which are suitably adapted to deal with the propagation of the weights (in addition to the propagation of the
information in the allowed assignments).

Constraint propagation is a fundamental technique in the context of CSPs (see, e.g.,~\cite{D03}), and a number of different propagation
strategies have been proposed over the years for standard CSPs.
Moreover, it is known since the very introduction of the setting of valued and semiring-based CSPs that suitable notions of consistency can be
used for propagation in the constraint optimization framework, too.
It has been observed in~\cite{BMR97} that when the aggregation operator is idempotent, then \emph{soft local consistency} terminates by
producing an instance equivalent to the original one and in a way that the result does not depend on the order of application of the
propagation rules---see also~\cite{BFMR04}, for a general environment supporting different forms of soft propagations founded on these
properties.
The result excludes non-idempotent operators, such as the very basic summation, and hence it has a limited scope of applicability.
However, for operators that are not necessarily idempotent but admit a partial inverse, a notion of \emph{soft arc consistency} can be defined
whose enforcement preserves at least the equivalence with the original instance~\cite{CS04}.
The goal of soft arc consistency is to transform a problem into an equivalent one, by providing incrementally maintainable bounds which are
crucial for branch and bound search~\cite{LT93}. Motivated by the fact that the fixpoint of the computation is not unique and may lead to
different lower bounds, the problem of finding optimal sequences of soft arc consistency operations has been recently addressed
by~\cite{CGSSZW10}.

While our algorithms might be abstractly viewed as methods enforcing (suitable kinds) of soft consistency, it must be pointed out that they are
completely orthogonal to the research illustrated above. The correctness of our propagation strategies (except for Section~\ref{sec:parallel})
only requires that the instance has a tree-like structure (without any further knowledge about it, but its existence) to find the optimal
sequence of consistency operations leading not only to a bound, but in fact to the exact solution.
There is no obvious way to extend our results to design algorithms for efficiently solving (or just enforcing bounds in) arbitrary instances.

\section{Conclusion}\label{sec:conclusion}

A formal framework for constraint optimization has been proposed and analyzed. The computational complexity of reasoning problems related to
computing the best solutions have been studied. In particular, structural tractability results have been derived within the general setting of
{tree projections}. Transferring our theoretical findings into the design of a practical platform for constraint optimization is a natural
avenue of further research. With respect to foundational analysis and theoretical contributions, instead, efforts might be spent to study
extensions of the framework supporting, for instance, forms of multi-criteria optimization.

\section*{Acknowledgments}

Georg Gottlob's work was supported by the EPSRC Programme Grant EP/M025268/ ``VADA: Value Added Data Systems -- Principles and Architecture''.
The work of Gianluigi Greco was supported by the Italian Ministry for Economic Development under PON project ``Smarter Solutions in the Big
Data World'' and by the Regione Calabria under POR project ``Explora Process''.

\bibliographystyle{plain}
\bibliography{biblio}

\def\appendixname{}

\appendix

\section{Proof of Theorem~\ref{thm:nopromiseMAX}}\label{sec:proof}

In this section, we analyze the correctness of \algMax. The proof is rather involved and is discussed incrementally.

Hereinafter, recall that $\Q$ is a constraint formula, $\DB$ a constraint database, $O\subseteq\vars(\Q)$ a set of variables, $\mathcal{F}$ a
valuation function, $\V$ a view set for $\Q$, and $\DB\mbox{\rm '}$\, a legal database instance for $\V$ w.r.t.~$\Q$ and $\DB$. Moreover,
recall the following.

\begin{proposition}[\cite{GS10}]\label{fact:equiv2}
Let $\Q$ be a constraint formula over {\em $\DB$}, $\V$ a view set for $\Q$, and {\em $\DBW$} a legal database for $\V$ w.r.t.~$\Q$ and {\em
$\DB$}.
Let $\Q_a$ be a {sandwich formula} of $\Q'$ w.r.t.~$\V$, where $\Q'$ is a subformula of $\Q$ such that $\Q'\homEquiv \Q$. Then,

\begin{enumerate}
\item[(1)] A database {\em $\DB_a\mbox{\rm \hspace{-1.5mm}'}$} legal for $\atoms(\Q_a)$ w.r.t.~$\Q$ and {\em $\DB$} can be computed in
    polynomial time.

\item[(2)] If {\em $\DBW$} is locally consistent w.r.t.~$\V$, then for every $w\in \V$ and every $h\subseteq \vars(w)$ such that there is
    atom $q_a\in\ \atoms(\Q_a)$ with $h\subseteq \vars(q_a)$, {\em $w^{\onDBW}[h]=\Q^\onDB[h]$}.

\item[(3)] If $\Q_a$ is a sandwich formula of $\Q$ w.r.t.~$\V$ (and not just of $\Q'$), then {\em $\Q_a^{\onDB_a\mbox{\small \rm
    \hspace{-1.5mm}'}}=\Q^\onDB$};
\end{enumerate}
\end{proposition}

\subsection{Promise Tractability}



For any set of variables $O$, define $\atom(O)$ to be a fresh atom (with a fresh relation symbol) over these variables, i.e., such that
$O=\vars(\atom(O))$.
We start the analysis of \algMax\ by considering the following promise:

\begin{itemize}
\item[\textbf{(P1)}] {\em $(\mathcal{F},O)$ can be embedded in $(\Q',\V)$ for some subformula $\Q'$ of $\Q\wedge\atom(O)$ such that
    $\Q'\homEquiv \Q\wedge\atom(O)$}.
\end{itemize}

Note that \textbf{P1} is less stringent than assuming that $(\mathcal{F},O)$ can be embedded in $(\Q,\V)$.
In fact, we will find it convenient to show that \algMax\ is correct (as a promise algorithm), even under \textbf{P1}.
To this end, we first prove the following technical lemma, stating some crucial properties of separator views.

\begin{lemma}\label{lem:propagation}
Assume that \emph{\textbf{P1}} holds. After the last execution of {\it evaluate} in \algMax, $\newsep_s$ contains a view $[w_O X]$ with
$\vars(w_O)=O$ such that, for all $\theta\in [w_O X]^{\onDB'_{s}}$, $\theta[X] = \max_{\mathcal{F}} (\theta[O])$.
\end{lemma}
\begin{proof}
Since $(\mathcal{F},O)$ can be \emph{embedded in $(\Q',\V)$} for some subformula $\Q' \homEquiv \Q\wedge\atom(O)$, there exists a sandwich
(acyclic) formula $\Q_a$ equivalent to $\Q'$ and $\Q\wedge\atom(O)$ and an embedding $\xi$ that satisfies the conditions for being an embedding
for some join tree $\JT=(V,E,\chi)$ of the tree projection $\HG_{\Q_a}$. In particular, we may assume w.l.o.g that, for every leaf $p_j$,
$\chi(\xi(p_j))= \vars(w_{f_j})$ for some function view $w_{f_j}$, and that the root $p_s$ is mapped via $\xi$ to $\atom(O)$, and hence
$\chi(\xi(p_s))=O$. Recall that $(p_1,\dots,p_s)$ is a topological ordering of the vertices of ${\it tree}(\mathcal{F},O)$, and hence $p_s$ is
the root of the tree.

Note that after local consistency, from Proposition~\ref{fact:equiv2}, for every $q\in \Q_a$, $q^{\onDB_1} = \Q^\onDB[q]$. Thus, all
assignments in the relations for the output view and the function views are partial homomorphisms that can be extended to full answers of the
given constraint formula.

The proof is by induction on $i=1\ldots s$. For any view $v\in\newsep_i$, any set of variables $h\subseteq \vars(v)$, and any assignment
$\theta\in v^{\onDB_{i'}}$, let $\max(\theta,X,h,v^{\onDB_{i}'})=\max\{\theta'[X] \mid \theta'\in v^{\onDB_{i}'} \wedge
\theta'[h]=\theta[h]\}$. We show that the following properties hold for views in any set $\newsep_i$, where the first of the two entails the
statement of the lemma:
\begin{enumerate}
\item $\newsep_i$ contains some view $[w X]$ with $h_i\subseteq\vars(w)$, where $h_i$ are the variables of the $p_i$-separator, and for all
    $\theta\in [w X]^{\onDB_{i}'}$, $\max_{\mathcal{F}_{p_i}} (\theta[h_i])= \max(\theta,X,h_i,[w X]^{\onDB_{i}'})$;

\item for each view $[w' X]\in\newsep_i$, for all $\theta\in [w' X]^{\onDB_{i}'}$, $\theta[X]\geq \max_{\mathcal{F}_{p_i}}
    (\theta[w'])$.\footnote{Recall that $\max_{\mathcal{F}_{p_i}} (\theta[w'])=\bot$, whenever $\theta[w']$ cannot be extended to any full
    answer of $\Q$.}
\end{enumerate}

(Basis: $i=1$.) The vertex $p_1$ is a leaf of ${\it tree}(\mathcal{F},O)$ and $\newsep_1$ contains one view $[w_{f_1} X^{(w_{f_1})}_1]$, whose
database relation is $\{ \theta\cup \{ X^{(w_{f_1})}_1/f_1(t)\} \mid \theta\in w_{f_1}^{\onDB_1} \}$. Clearly, in this case, every assignment
$\theta[w_{f_1}]$ has only one possible weight for $f_1(\theta)$, because $f_1=\mathcal{F}_{p_1}$ is evaluated precisely on its variables
$\vars(w_{f_1})$. Therefore,  $\theta[X^{(w_{f_1})}_1] = \max_{\mathcal{F}_{p_1}} (\theta)$ for every assignment $\theta$. Indeed, $\theta$ is
an assignment of a function view and thus it is part of some answer of $\Q$. As a consequence, there are full answers in $\Q^\onDB$ that extend
$\theta$, and all of them will get $\theta[X^{(w_{f_1})}_1]$ as their weight according to ${\mathcal{F}_{p_1}}$. Finally, recall that, being
$p_1$ a leaf, we have $\chi(\xi(p_1))=\vars(f_1)=\vars(w_{f_1})$.

\medskip

(Inductive Step: $i=j+1$.) Assume the statement holds until some $j\geq 1$. That is, we have executed function~{\it evaluate} at step $j$, and
we consider the execution of function~{\it propagate}, in order to obtain the new constraint database $\DB_{j+1}$. This means that we would
like to propagate the weights of $\mathcal{F}_{p_j}$ to the candidate separators in $\newsep_r$, where $p_r$ is the parent of $p_j$ in ${\it
tree}(\mathcal{F},O)$. By the inductive hypothesis, there exists some view $[w X^{(w)}_j]\in\newsep_j$ such that $h_j\subseteq \vars(w)$ are
the variables of the separator $\xi(p_j)$, and where $\forall \theta\in [w X^{(w)}_j]^{\onDB_{j}}$, $\max(\theta,X^{(w)}_j,h_j,[w
X^{(w)}_j]^{\onDB_{j}}) = \max_{\mathcal{F}_{p_j}} (\theta[h_j])$. Let $h_r=\chi(\xi(p_r))$ be set of variables of the $p_r$-separator
occurring in the tree projection and thus included in the variables of some atom $a_r$ of the sandwich formula $\Q_a$. Since all views in $\V$
are considered as candidates to be $p_r$-separators, $\newsep_r$ contains some views of the form $[w_r X' X'']$, such that $h_r \subseteq
\vars(w_r)$.
By construction---see the initialization step---among them there is also an augmented view having the form  $[w_r X^{(w)}_j X']$. Then, there
is a step in function~{\it propagate} where we enforce local consistency on the pair $(\V_j,\DB_j')$ where $\V_j$ includes  $[w X^{(w)}_j]$ and
$[w_r X^{(w)}_j X']$, together with all augmented views $[w_b X^{(w)}_j]$, for each $w_b\in\V$. Observe that the corresponding constraint
database $\DB_j'$ is legal for $\V_j$ w.r.t.~$\Q_+=\Q''\wedge [w X^{(w)}_j]$ and $\DB''$, where $\Q'' = \bigwedge_{q\in \Q'} [q X^{(w)}_j]$;
and the relations in $\DB''$ have the form $\rel(q,\DB)\times\dom(X^{(w)}_j)$, for any $q\in \Q'$, and $\rel([w X^{(w)}_j],\DB_j')$ for $[w
X^{(w)}_j]$. Indeed, $\DB_1$ was legal w.r.t.~$\Q'$ and $\DB$, and the new variable $X^{(w)}_j$ occurs in the view $[w X^{(w)}_j]$ with the
right weights, and in all other relations with all possible weights in its active domain (for every assignment of the original relation).

Consider the acyclic hypergraph $\HG_a^+$ such that $\nodes(\HG_a^+)=\nodes(\HG_{\Q_a})\cup \{X^{(w)}_j\}$, and $\edges(\HG_a^+)=\{ h\cup
\{X^{(w)}_j\} \mid h\in \edges(\HG_{\Q_a})\}$. Since the new variable occurs in all available views, and $h_j$ is covered in $\HG_{\Q_a}$,
$\HG_a^+$ is clearly a tree projection of $\HG_{\Q_+}$ w.r.t.~$\HG_{\V_j}$. In particular, because $\HG_{\Q_a}$ covers $h_r$, the tree
projection $\HG_a^+$ covers $h_r\cup \{X^{(w)}_j\} $, too. From Proposition~\ref{fact:equiv2}, after $\DB_{j+1}$ is obtained by enforcing local
consistency on $(\V_j,\DB_j')$, $[w_r X^{(w)}_j X']^{\onDB_{j+1}}[h_r\cup \{X^{(w)}_j\}]= \Q_+^{\onDB''}[h_r\cup \{X^{(w)}_j\}]$. Recall that
all assignments in $\rel([w X^{(w)}_j],\DB'')[h_j\cup \{X^{(w)}_j\}]$ are correct by the inductive hypothesis. Therefore, the last statement
means that $\rel([w_r X^{(w)}_j X'],{DB_{j+1}})$ is such that each assignment $\theta$ in this relation holds in the variable $X^{(w)}_j$ the
weight of any assignment $\theta'$ of the $p_j$ separator that can be extended to a same full answer of $\Q'$ as $\theta$, i.e., such that
there exists $\theta''\in \Q'^\onDB$ with $\theta''[h_j]= \theta'[h_j]$ and $\theta''[h_r] = \theta[h_r]$. In particular some of these
assignments will hold the correct maximal weight $\max_{\mathcal{F}_{p_j}}(\theta[h_r])$.
Note that all variables $C$ that determine the evaluation of $\mathcal{F}_{p_j}$ are covered in the subtree rooted at the separator
$\xi(p_j)$ (cf. Theorem~\ref{thm:property}). 
From the connectedness condition of $\JT$, $(C\cap h_r)\subseteq h_j$, and thus the above possible extensions of $\theta$ are precisely those
relevant for the evaluation of $\mathcal{F}_{p_j}$. It follows that the maximum weight $\max(\theta,X^{(w)}_j,h_r,[w_r X^{(w)}_j
X']^{\onDB_{j+1}})$ over all assignments $[w_r X^{(w)}_j X']^{\onDB_{j+1}}[h_r]$ that agree with $\theta[h_r]$ is the maximum weight for
$\mathcal{F}_{p_j}$ over all possible extensions of $\theta[h_r]$ to answers in $\Q'^\onDB$. Finally, recall that $\Q'$ is a subformula of
$\Q\wedge \atom(O)$, and clearly considering further constraints in possible extensions encoded by the other atoms in $\Q$ cannot improve such
a weight. Therefore, $\max(\theta,X^{(w)}_j,h_r,[w_r X^{(w)}_j X']^{\onDB_{j+1}})$ is indeed equal to $\max_{\mathcal{F}_{p_j}}(\theta[h_r])$.

\medskip

Now, consider instead the propagation from any generic view $[w' X^{(w')}_j]\in\newsep_j$, not necessarily including the variables of some
$p_j$-separator, to a generic view $[w'_r X^{(w')}_j X']\in\newsep_r$. From the inductive hypothesis, for all $\theta\in [w'
X^{(w')}_j]^{\onDB_{j'}}$, $\theta[X^{(w')}_j]\geq \max_{\mathcal{F}_{p_j}} (\theta[h_j])$, where $h_j=\vars(w')$.
We next show that for all $\theta'\in [w_r' X^{(w')}_j X']^{\onDB_{j'}}$, $\theta'[X^{(w')}_j]\geq \max_{\mathcal{F}_{p_j}} (\theta'[h])$,
where $h=\vars(w_r')$. To this end, consider the same construction of $\V_j$, $\DB'_j$, $\Q_+$, and $\DB''$ as above, but where we use
everywhere $w'$ instead of $w$, and $h$ instead of $h_j$. Note that in this general case we do not know whether there exists a tree projection
of $\HG_{\Q_+}$ w.r.t.~$\HG_{\V_j}$ that covers both $h$ and $h'_r$, thus guaranteeing the correct propagation of weights stored in
$X^{(w')}_j$. However we can always add a suitable set of views $\V_+$ to $\V$ having this property, i.e., such that there exists a tree
projection $\HG_a'$ of $\HG_{\Q_+}$ w.r.t.~$\HG_{\V_j\cup \V_+}$ that covers both $h$ and $h'_r$. For each view $w_+\in\V_+$, we add to
$\DB'_j$ the most liberal relation $\dom(X_1)\times\dots\times\dom(X_z)$, if $X_1,\dots,X_z$ are the variables occurring in $w_+$. Clearly, the
resulting constraint database, say $\DB''_j$ is legal for $\V_j\cup \V_+$ w.r.t.~$\Q_+$ and $\DB''$. Assume we enforce local consistency on
$(\V_j\cup \V_+,\DB''_j)$ and let $\DB_+$ be the resulting constraint database. Consider any assignment $\theta'\in w_r'^{\onDB'_j}$. If this
assignment cannot be extended to any full solution (e.g., this is definitely the case if it does not belong to $w_r'^{\onDB_+}$), then
$\mathcal{F}_{p_j}(\theta')=\bot$ and the statement trivially holds. Then, assume by contradiction that $\theta'\in w_r'^{\onDB_+}\cap
\Q'^\onDB[h'_r]$ and that $\theta'[X^{(w')}_j] < \max_{\mathcal{F}_{p_j}}(\theta'[h'_r])$. Thus, there exists some assignment $\theta_m\in
\Q'^\onDB$ such that $\theta_m[h'_r]= \theta'[h'_r]$ and $\mathcal{F}_{p_j}(\theta_m) > \theta'[X^{(w')}_j]$. Note that there cannot exist any
assignment $\theta\in w'^{\onDB_+}\cap \Q'^\onDB[h]$ with $\theta_m[h] = \theta[h]$ and $\theta[X^{(w')}_j] \geq
\max_{\mathcal{F}_{p_j}}(\theta[h])\geq \mathcal{F}_{p_j}(\theta_m)$, otherwise such a weight would be propagated to the view $w_r$ after the
local consistency procedure, thanks to the tree projection $\HG_{\Q_+}$. Indeed, from Proposition~\ref{fact:equiv2}, we get $w'^{\onDB_+}[h\cup
X^{(w')}_j] = \Q_+^{\onDB''}[h\cup X^{(w')}_j]$ and $w_r'^{\onDB_+}[h'_r\cup X^{(w')}_j] = \Q_+^{\onDB''}[h'_r\cup X^{(w')}_j]$.
However, some assignment $\theta''\in w'^{\onDB_+}\cap \Q'^\onDB[h]$ such that  $\theta_m[h] = \theta''[h]$ must exist since $\theta_m$ is a
full solution and thus must match some assignment in every view, from the view-consistency property of the legal database $\DB_+$. Moreover,
from the inductive hypothesis, we know that $\theta''[X^{(w')}_j] \geq \max_{\mathcal{F}_{p_j}}(\theta[h])$, which thus leads to a
contradiction. To conclude the analysis of this step, recall that we considered an additional set of views $\V_+$, to get the desired
decomposition. However, if we remove such views and consider only the set $\V_j$, we get more combinations available in
$w_r'^{\onDB'_j}[h'_r\cup X^{(w')}_j]$ and more weights (possibly wrong) for assignments $\theta'\in w_r'^{\onDB'_j}$ among which to select an
even larger weight for $\theta'[X^{(w')}_j]$.

\medskip

To conclude the proof of the inductive step, consider now the execution of function~{\it evaluate} at step $i=j+1$. If $p_i$ is a leaf, then
the proof is the same as the base case $p_1$. Therefore, assume $p_i$ is not a leaf, and let $p_b$ and $p_c$ its children, with $b<c$, without
loss of generality. Since we are proceeding according to a topological ordering, both $b< c <i$  holds, and thus function~{\it Propagate} has
been already executed for views in $\newsep_b$ and $\newsep_c$. Let $h_i$ be the set of variables of the separator $\xi(p_i)$ that is covered
in the tree projection $\HG_{\Q_a}$, and let $w_i\in\V$ be a view with $h_i\subseteq\vars(w_i)$. From the inductive hypothesis, there exists
two views $[w_b X^{(w_b)}_b]\in \newsep_b$ and $[w_c X^{(w_c)}_c]\in \newsep_c$ that propagated the correct weights for $\mathcal{F}_{p_b}$ and
$\mathcal{F}_{p_c}$ to their parent $p_i$. More precisely, after the above discussion on function~{\it Propagate}, it follows that the view
$[w_i X^{(w_b)}_b X^{(w_c)}_c]$ is such that, $\forall \theta\in [w_i X^{(w_b)}_b X^{(w_c)}_c]^{\onDB_i}$, $\max(\theta,X^{(w_b)}_b,h_i,[w_i
X^{(w_b)}_b X^{(w_c)}_c]^{\onDB_i})=\max_{\mathcal{F}_{p_b}}(\theta[h_i])$ and $\max(\theta,X^{(w_c)}_c,h_i,[w_i X^{(w_b)}_b
X^{(w_c)}_c]^{\onDB_i})=\max_{\mathcal{F}_{p_c}}(\theta[h_i])$. However, since local consistency holds and $h_i$ is covered by the tree
projection at hand, such assignments $\theta[h_i]$ are precisely the assignments of the relation $\rel(w_{p_i},\overline \DB)$ of the atom
$w_{p_i}$ in the sandwich formula covering the $p_i$-separator, as in the statement of Lemma~\ref{lem:distribution}. Therefore, if $\oplus_i$
is the operator labeling $p_i$ in the parse tree, from this lemma get $\max_{\mathcal{F}_{p_i}}(\theta[h_i])=
\max_{\mathcal{F}_{p_b}}(\theta[h_i]) \oplus_i \max_{\mathcal{F}_{p_c}}(\theta[h_i]) =$ $\theta[X^{(w_b)}_b]\oplus_i \theta[X^{(w_b)}_c]$. It
follows that the desired combined maximum is achieved, for any projected assignment $\theta[h_i]$, on some assignment $\theta'\in  [w_i
X^{(w_b)}_b X^{(w_c)}_c]^{\onDB_i}$, with $\theta'[h_i]=\theta[h_i]$ and $t'[w_i]\in \Q^{\overline \onDB}$, that gets maximum weights according
to both $\mathcal{F}_{p_b}$ and $\mathcal{F}_{p_b}$. Of course such a correct assignment is preserved under any step of the algorithm, and we
know that the weights are correctly propagated via the tree-projection. Then, the marginalization step correctly selects the right maximal
weights for every assignment according to $\mathcal{F}_{p_i}$.

Finally, observe that, from the inductive property proved for the two variables $X^{(w_b)}_b$ and $X^{(w_c)}_c$ in the discussion about ${\it
propagate}$ and from Lemma~\ref{lem:bound}, it follows that the inductive statements hold for $\newsep_i$, too. Now consider the step where the
selection of a ``minimum view'' (if any) is executed. Observe that, by the inductive hypothesis, wrong views may only get better maxima for
their assignments. Therefore, for any set of augmented views $\V_w\subseteq \newsep_i$ with the same ``base'' view $w$, we may safely select
for each assignment $\theta\in w^{\overline\onDB}$, its version with the lowest marginalized weight $X^{(w')}_i$ among all views in $\V_w$.
More precisely, observe that such a selection does not alter the validity of the two inductive statements. Moreover, as far as augmented views
over different base views are concerned, we may safely consider the absolute maximum. Clearly, the right one (if any) should be the smallest
one, from the inductive statements. Therefore, looking at this maximum, we are able to discard augmented views that cannot contain the
variables of an actual $\newsep_i$ separator, and more importantly we enforce the actual maximum to be an upper bound for all weights computed
in the algorithm for $\mathcal{F}_{p_i}$ and, at the end, for $\mathcal{F}$.
\end{proof}

We can now show the correctness of \algMax\ under \textbf{P1}.

\medskip \noindent \textbf{Theorem~\ref{thm:promiseWeakMax}.}
{\em Algorithm~\algMax\ runs in polynomial time.
It outputs {\em \texttt{NO SOLUTION}}, only if $\PHI^\onDB=\emptyset$.
Moreover, it computes 
an answer to {\em \anymaxref($\PHI_\mathcal{F},O,\mathrm{DB},\V,\DBW$)}, with $\mathcal{F}$ being a structured evaluation function, if
\emph{\textbf{P1}} holds. It outputs {\em \texttt{FAIL}}, only if \emph{\textbf{P1}} does not hold.}

\smallskip

\begin{proof}
Note first that, whenever the algorithm outputs \texttt{NO SOLUTION}, then the constraint database obtained by enforcing local
pairwise-consistency is empty. Hence, we are guaranteed that there is no solution at all in $\PHI^\onDB$.

Consider then the case where the algorithm does not output  \texttt{NO SOLUTION}. In this case,  correctness follows from
Lemma~\ref{lem:propagation} applied to the root $p_s$, which contains a unique augmented view, after the selection of the ``minimum'' view in
function~$\mathit{evaluate}$. Such a view contains only corrected weights if {$(\mathcal{F},O)$ can be embedded in $(\Q',\V)$ for some
subformula $\Q'$ of $\Q\wedge\atom(O)$ such that $\Q'\homEquiv \Q\wedge\atom(O)$}. On the other hand, the algorithm outputs \texttt{FAIL} only
when it recognizes that this condition does not hold, because there are no feasible separators at some step, or when it turns out that the
formula has no answers at all.
Observe that Algorithm~\algMax\ performs $s-1$ iterations of the {\bf for} loop plus some further operations, each one feasible in polynomial
time (where the mathematical operations cost 1). In particular, local consistency requires polynomial time w.r.t.~the given set of views and
the given constraint database, and note that we deal with at most $O(|\V|^2 \size{\mathcal{F}})$ views in the algorithm, where
$\size{\mathcal{F}}$ denotes the size of $\mathcal{F}$, which is an upper bound to the total number of vertices of its parse trees.

Finally, note that the the size of the weights computed during the execution of the algorithm is polynomially bounded w.r.t.~the size of the
input. Indeed, recall from Lemma~\ref{lem:propagation} that we avoid the computation of wrong weights larger than the actual maximum, if the
promise is true. Otherwise, we cannot guarantee anything: just think that the formula may be empty if there are no tree-projections at all (no
matter of formula embeddings).
\end{proof}

\subsection{Larger Islands of Tractability}

We now show how to solve \anymaxref\ under a less stringent promise than \textbf{P1}, by using \algMax\ as an oracle.
For any variable $X$, define $\atom_X(\{X\})$ to be a fresh atom (with a fresh relation symbol) such that $\{X\}=\vars(\atom_X(\{X\}))$.

Consider the following promise:

\begin{itemize}
\item[\textbf{(P2)}] \emph{$\mathcal{F}$ can be embedded in $(\Q',\V)$ for some subformula $\Q'$ of $\Q\wedge\bigwedge_{X\in
    O}\atom_X(\{X\})$ such that $\Q'\homEquiv \Q\wedge\bigwedge_{X\in O}\atom_X(\{X\})$}.
\end{itemize}

It is easy to see that \textbf{P1} entails \textbf{P2}, as every tree projection covering $\atom(O)$ also covers all atoms of the form
$\atom_X(\{X\})$, for any $X \in O$.

\begin{theorem}\label{thm:promiseMax}
There is a polynomial-time algorithm that solves {\em \anymaxref($\PHI_\mathcal{F},O,\mathrm{DB},\V,\DBW$)}  
if \emph{\bf P2} holds. It outputs {\em \texttt{FAIL}}, only if \emph{\bf P2} does not hold.
\end{theorem}

\begin{proof} Let $\bar X$ be a variable in $O$. Let $\V_{\bar X}$ be the set of views $\{ [w \bar X] \mid w\in \V, X\not\in
\vars(w)\}\cup \{ w \mid w\in V, \bar X\in \vars(w)\}$.
Let $\DBW_{\bar X}$ be the constraint database including all relations of $\DBW$ for the views $w\in \V_{\bar X}\cap \V$, and the cartesian
product $w^\onDBW\times \dom(\bar X)$, for each other view $[w \bar X]$.
Observe that $\mathcal{F}$ can be embedded in $(\Q',\V_{\bar X})$, too. Then, since each view in $\V_{\bar X}$ includes $\bar X$ in its sets of
variables, it trivially follows that $(\mathcal{F},\{\bar X\})$ can be embedded in $(\Q',\V_{\bar X})$.
Note that $\Q'$ has the form $\Q'' \wedge \bigwedge_{X\in O}\atom_X(\{X\})$, where $\Q''$ is a subformula of $\Q$ with $\Q''\homEquiv \Q$.
Thus, $\Q'' \wedge \atom_{\bar X}(\{\bar X\})$ is homomorphically equivalent to $\Q \wedge \atom_{\bar X}(\{\bar X\})$. And, of course,
$(\mathcal{F},\{\bar X\})$ can be embedded in $(\Q'',\V_{\bar X})$.

Now, note that $\DBW_{\bar X}$ is legal for $\V_{\bar X}$ w.r.t.~$\Q$ and $\DB$. Thus, we are in the position to apply \algMax\ over $\Q$,
$\mathcal{F}$, $\V_{\bar X}$, $\DBW_{\bar X}$, and with $O=\{\bar X\}$. By the application of Theorem~\ref{thm:promiseWeakMax} on the modified
instance, if \algMax\ returns \texttt{NO SOLUTION} (resp., \texttt{FAIL}), then we can terminate the computation, by returning that there is no
solution, i.e., $\PHI^\onDB=\emptyset$ (resp., the promise \textbf{P1}---and, hence \textbf{P2}---is disproved).
Therefore, let us assume that we get an assignment $\theta_{\bar X}$.

After that $\theta_{\bar X}$ is computed, we can update $\DBW$ in such a way that $\dom(\bar X)=\{\theta_{\bar X}\}$, and repeat the process
with another variable in $O$. Eventually, all the values obtained this way for variables in $O$ form an assignment that is returned as output
as a solution to {\anymaxref}.
Note that, at each step, \algMax\ can disprove the promise \textbf{P1} (and, hence \textbf{P2}), by returning \texttt{FAIL}. Moreover, our
procedure may additionally disprove the property \textbf{P2}, whenever there is a pair of values $\theta_{\bar X}$ and $\theta_{\bar X'}$
associated with different maximum weights by \algMax. Finally, if some subsequent invocation returns \texttt{NO SOLUTION} (after that the first
invocation did not return \texttt{NO SOLUTION} or \texttt{FAIL}), then we can conclude that there is no sandwich formula of $\Q$ w.r.t.~$\V$
(by the results in~\cite{GS10,GS13}), and hence the promise is again disproved.
\end{proof}

By the above result and the approach by Lawler~\cite{L72}, we obtain the corresponding tractability result for \topk.

\begin{theorem}\label{thm:promiseTOPK} There is an algorithm that solves {\em \topk($\PHI_\mathcal{F},O,\mathrm{DB},\V,\DBW$)}  with
polynomial delay, if \emph{\bf P2} holds. It outputs {\em \texttt{FAIL}}, only if \emph{\bf P2} does not hold.
\end{theorem}

\begin{proof} Let $\bar \Q=\Q\wedge\bigwedge_{X\in O}\atom_X(\{X\})$. Update $\V$ by adding the base views associated with
such novel atoms. Moreover, update $\DB$ and $\DBW$ by adding the relations associated with these views, each one containing the whole active
domain $\dom(X)$ of any variable $X\in O$.
Note that $\Q^\onDB=\bar \Q^\onDB$ and that $\DBW$ is legal. By \textbf{P2}, $\mathcal{F}$ can be embedded in $(\Q',\V)$ for some subformula
$\Q'$ of $\bar \Q$ with $\Q'\homEquiv \bar \Q$. Then, \textbf{P2} clearly holds for $\bar \Q$, too. From Theorem~\ref{thm:promiseMax},
\anymaxref($\bar \Q_{\mathcal{F}},O,\DB$) can be computed in polynomial time. Therefore, we can apply Lawler's procedure~\cite{L72} on $\bar
\Q$, by using \algMax\ as an oracle. In particular, at any step, the oracle will be called with a different constraint database obtained by
changing the $\atom_X$ relations (as in~\cite{KS06}), and with the base views in the constraint database $\DBW$ changed accordingly. Note that,
at each step, \anymaxref\ is still feasible in polynomial time, as we just change the database instance over which it has to be solved, without
affecting the structural promise.
\end{proof}

\subsection{Putting It All Together: Computing Certified Solutions}

We can now complete the picture, by showing that the above results may be exploited to design an algorithm that always output reliable answers.
Indeed, if it is not possible to compute a correct solution, the promise is disproved. In particular, the price of having certified solutions
is now to consider the following more stringent promise: 

\begin{itemize}
\item[\textbf{(P0)}] \emph{$\mathcal{F}$ can be embedded in $(\Q,\V)$}.
\end{itemize}

\medskip

\noindent \textbf{Theorem~\ref{thm:nopromiseMAX}.} {\em There is a polynomial-time algorithm for structured valuation functions that either
computes a solution to {\em \anymaxref($\PHI_\mathcal{F},O,\mathrm{DB},\V,\DBW$)},
or disproves \emph{\bf P0}.}

\smallskip

\begin{proof}
Assume that \textbf{P0} holds. Then, we derive that $\mathcal{F}$ can be embedded in $(\Q\wedge\bigwedge_{X\in \vars(\Q)}\atom_X(\{X\}),\V)$,
because any tree projection of $\Q$ w.r.t.~$\V$ clearly covers every single variable in $\vars(\Q)$. Then, the promise \textbf{P2} holds on
$\Q\wedge\bigwedge_{X\in \vars(\Q)}\atom_X(\{X\})$ and, thus, we can exploit Theorem~\ref{thm:promiseMax} to compute a solution to
{\anymaxref}{($\Q_{\mathcal{F}},\vars(\Q),\DB$)}.

Assume that the procedure in the proof of Theorem~\ref{thm:promiseMax} succeeds (otherwise \textbf{P2} and thus \textbf{P0} are disproved).
If the result is \texttt{NO SOLUTION}, then we are in fact guaranteed that there is no solution at all. Otherwise, let $\theta$ be its output.
Let $v_{max}$ be the one weight associated with all partial solutions at any invocation of \algMax. Then, in polynomial time we check that
$\theta$ is in fact an answer in $\Q^\onDB$, and that $\mathcal{F}(t)=v_{max}$. If the former check fails, then we conclude that no tree
projections exist (and thus no embedding may exist, as well). Otherwise, from the proof of Lemma~\ref{lem:propagation},
$\max_{\mathcal{F}}(\theta[X])\leq v_{max}$, for each variable $X\in \vars(\Q)$. Therefore, if $\mathcal{F}(\theta)=v_{max}$ we are sure that
the result is correct, and else that the promise does not hold.
\end{proof}

By following the same line of reasoning as in the proof of Theorem~\ref{thm:promiseTOPK}, we get the corresponding result for \topk.

\medskip

\noindent \textbf{Theorem~\ref{thm:nopromiseTOPK}.} {\em There is a polynomial-delay algorithm for structured valuation functions that either
solves {\em \topk($\PHI_\mathcal{F},O,\mathrm{DB},\V,\DBW$)}, or disproves that $\mathcal{F}$ can be embedded in $(\Q,\V)$; in the latter case,
before terminating, it computes a (possibly empty) certified prefix of a solution.}
\end{document}

%% file: shunt2.eepic.tex
\setlength{\unitlength}{0.00015000in}
\begingroup\makeatletter\ifx\SetFigFont\undefined
\def\x#1#2#3#4#5#6#7\relax{\def\x{#1#2#3#4#5#6}}%
\expandafter\x\fmtname xxxxxx\relax \def\y{splain}%
\ifx\x\y   
\gdef\SetFigFont#1#2#3{%
  \ifnum #1<17\tiny\else \ifnum #1<20\small\else
  \ifnum #1<24\normalsize\else \ifnum #1<29\large\else
  \ifnum #1<34\Large\else \ifnum #1<41\LARGE\else
     \huge\fi\fi\fi\fi\fi\fi
  \csname #3\endcsname}%
\else
\gdef\SetFigFont#1#2#3{\begingroup
  \count@#1\relax \ifnum 25<\count@\count@25\fi
  \def\x{\endgroup\@setsize\SetFigFont{#2pt}}%
  \expandafter\x
    \csname \romannumeral\the\count@ pt\expandafter\endcsname
    \csname @\romannumeral\the\count@ pt\endcsname
  \csname #3\endcsname}%
\fi
\fi\endgroup
{\renewcommand{\dashlinestretch}{30}
\begin{picture}(19910,9942)(0,-10)
\put(863,2715){\ellipse{1710}{1710}}
\put(3413,5115){\ellipse{1710}{1710}}
\put(6038,7815){\ellipse{1710}{1710}}
\put(5963,2715){\ellipse{1710}{1710}}
\put(17438,7740){\ellipse{1710}{1710}}
\put(13238,5190){\ellipse{1710}{1710}}
\path(5513,3465)(4088,4515)
\blacken\thicklines
\path(4316.805,4420.936)(4088.000,4515.000)(4245.621,4324.329)(4316.805,4420.936)
\thinlines
\path(3788,5865)(5363,7140)
\blacken\thicklines
\path(5214.213,6942.358)(5363.000,7140.000)(5138.709,7035.627)(5214.213,6942.358)
\thinlines
\path(38,2790)(1688,2790)
\path(2588,5190)(4238,5190)
\path(5213,7890)(6863,7890)
\path(5138,2790)(6788,2790)
\path(1463,3390)(2663,4440)
\blacken\thicklines
\path(2521.892,4236.804)(2663.000,4440.000)(2442.871,4327.113)(2521.892,4236.804)
\thinlines
\dashline{4.500}(7913,6240)(6788,7140)
\blacken\thicklines
\path(7012.890,7036.925)(6788.000,7140.000)(6937.927,6943.221)(7012.890,7036.925)
\thinlines
\dashline{4.500}(7913,1215)(6713,2190)
\blacken\thicklines
\path(6937.103,2085.225)(6713.000,2190.000)(6861.432,1992.091)(6937.103,2085.225)
\thinlines
\path(13838,5790)(16613,7290)
\blacken\thicklines
\path(16430.401,7123.094)(16613.000,7290.000)(16373.339,7228.658)(16430.401,7123.094)
\thinlines
\path(12413,5265)(14063,5265)
\dashline{4.500}(11288,3840)(12413,4815)
\blacken\thicklines
\path(12270.930,4612.475)(12413.000,4815.000)(12192.339,4703.158)(12270.930,4612.475)
\thinlines
\path(16613,7815)(18263,7815)
\dashline{4.500}(6038,8715)(6038,9915)
\blacken\thicklines
\path(6098.000,9675.000)(6038.000,9915.000)(5978.000,9675.000)(6098.000,9675.000)
\thinlines
\dashline{4.500}(4163,1215)(5213,2115)
\blacken\thicklines
\path(5069.826,1913.255)(5213.000,2115.000)(4991.731,2004.365)(5069.826,1913.255)
\thinlines
\dashline{4.500}(17438,8640)(17438,9840)
\blacken\thicklines
\path(17498.000,9600.000)(17438.000,9840.000)(17378.000,9600.000)(17498.000,9600.000)
\thinlines
\dashline{4.500}(15204,3860)(14063,4740)
\blacken\thicklines
\path(14289.687,4640.939)(14063.000,4740.000)(14216.401,4545.917)(14289.687,4640.939)
\thinlines
\dashline{4.500}(19424,6244)(18224,7219)
\blacken\thicklines
\path(18448.103,7114.225)(18224.000,7219.000)(18372.432,7021.091)(18448.103,7114.225)
\put(3338,6390){\makebox(0,0)[lb]{\smash{{{\SetFigFont{7}{8.4}{rm}$\bp$}}}}}
\put(5888,3915){\makebox(0,0)[lb]{\smash{{{\SetFigFont{7}{8.4}{rm}$\bs$}}}}}
\put(788,3915){\makebox(0,0)[lb]{\smash{{{\SetFigFont{7}{8.4}{rm}$\bl$}}}}}
\put(488,90){\makebox(0,0)[lb]{\smash{{{\SetFigFont{5}{6.0}{rm}a) Before the {\em shunt} operation on  $\bl$}}}}}
\put(12188,90){\makebox(0,0)[lb]{\smash{{{\SetFigFont{5}{6.0}{rm}b) After the {\em shunt} operation on $\bl$}}}}}
\put(5438,9015){\makebox(0,0)[lb]{\smash{{{\SetFigFont{7}{8.4}{rm}$\bq$}}}}}
\put(16688,9015){\makebox(0,0)[lb]{\smash{{{\SetFigFont{7}{8.4}{rm}$\bq$}}}}}
\put(17438,8040){\makebox(0,0)[b]{\smash{{{\SetFigFont{5}{6.0}{rm}$I$}}}}}
\put(17438,7215){\makebox(0,0)[b]{\smash{{{\SetFigFont{5}{6.0}{rm}$R$}}}}}
\put(13238,4740){\makebox(0,0)[b]{\smash{{{\SetFigFont{5}{6.0}{rm}$R$}}}}}
\put(863,2190){\makebox(0,0)[b]{\smash{{{\SetFigFont{5}{6.0}{rm}$R$}}}}}
\put(3413,4590){\makebox(0,0)[b]{\smash{{{\SetFigFont{5}{6.0}{rm}$R$}}}}}
\put(5963,3015){\makebox(0,0)[b]{\smash{{{\SetFigFont{5}{6.0}{rm}$I$}}}}}
\put(5963,2190){\makebox(0,0)[b]{\smash{{{\SetFigFont{5}{6.0}{rm}$R$}}}}}
\put(6038,8190){\makebox(0,0)[b]{\smash{{{\SetFigFont{5}{6.0}{rm}$I$}}}}}
\put(6038,7365){\makebox(0,0)[b]{\smash{{{\SetFigFont{5}{6.0}{rm}$R$}}}}}
\put(863,3015){\makebox(0,0)[b]{\smash{{{\SetFigFont{5}{6.0}{rm}$I$}}}}}
\put(13246,5505){\makebox(0,0)[b]{\smash{{{\SetFigFont{5}{6.0}{rm}$I$}}}}}
\put(3413,5415){\makebox(0,0)[b]{\smash{{{\SetFigFont{5}{6.0}{rm}$I$}}}}}
\put(13163,6315){\makebox(0,0)[lb]{\smash{{{\SetFigFont{7}{8.4}{rm}$\bs$}}}}}
\put(13163,8190){\makebox(0,0)[lb]{\smash{{{\SetFigFont{9}{10.8}{rm}$T'$}}}}}
\put(1613,8265){\makebox(0,0)[lb]{\smash{{{\SetFigFont{9}{10.8}{rm}$T$}}}}}
\end{picture}
}

%% file: genshunt.eepic.tex
\setlength{\unitlength}{0.00006667in}
\begingroup\makeatletter\ifx\SetFigFont\undefined
\def\x#1#2#3#4#5#6#7\relax{\def\x{#1#2#3#4#5#6}}%
\expandafter\x\fmtname xxxxxx\relax \def\y{splain}%
\ifx\x\y   
\gdef\SetFigFont#1#2#3{%
  \ifnum #1<17\tiny\else \ifnum #1<20\small\else
  \ifnum #1<24\normalsize\else \ifnum #1<29\large\else
  \ifnum #1<34\Large\else \ifnum #1<41\LARGE\else
     \huge\fi\fi\fi\fi\fi\fi
  \csname #3\endcsname}%
\else
\gdef\SetFigFont#1#2#3{\begingroup
  \count@#1\relax \ifnum 25<\count@\count@25\fi
  \def\x{\endgroup\@setsize\SetFigFont{#2pt}}%
  \expandafter\x
    \csname \romannumeral\the\count@ pt\expandafter\endcsname
    \csname @\romannumeral\the\count@ pt\endcsname
  \csname #3\endcsname}%
\fi
\fi\endgroup
{\renewcommand{\dashlinestretch}{30}
\begin{picture}(48076,18608)(0,-10)
\put(9563,17730){\ellipse{1710}{1710}}
\put(3263,15105){\ellipse{1710}{1710}}
\put(863,12405){\ellipse{1710}{1710}}
\put(5588,12405){\ellipse{1710}{1710}}
\put(9188,12405){\ellipse{1710}{1710}}
\put(12488,15105){\ellipse{1710}{1710}}
\put(15788,12105){\ellipse{1710}{1710}}
\put(10463,8805){\ellipse{1710}{1710}}
\put(8063,6105){\ellipse{1710}{1710}}
\put(12788,6105){\ellipse{1710}{1710}}
\put(16688,5805){\ellipse{1710}{1710}}
\put(19613,8805){\ellipse{1710}{1710}}
\put(22988,5505){\ellipse{1710}{1710}}
\put(21188,2205){\ellipse{1710}{1710}}
\put(25088,2205){\ellipse{1710}{1710}}
\put(36713,17730){\ellipse{1710}{1710}}
\put(33113,14655){\ellipse{1710}{1710}}
\put(40013,14730){\ellipse{1710}{1710}}
\put(34688,11430){\ellipse{1710}{1710}}
\put(32288,8730){\ellipse{1710}{1710}}
\put(37013,8730){\ellipse{1710}{1710}}
\put(40913,8430){\ellipse{1710}{1710}}
\put(43838,11430){\ellipse{1710}{1710}}
\put(47213,8130){\ellipse{1710}{1710}}
\path(8738,17805)(10388,17805)
\thicklines
\path(17288,6480)(18863,8130)
\blacken\path(18679.531,7807.449)(18863.000,8130.000)(18549.327,7931.735)(18679.531,7807.449)
\path(22538,6180)(20363,8055)
\blacken\path(20694.432,7888.109)(20363.000,8055.000)(20576.903,7751.775)(20694.432,7888.109)
\path(21563,3030)(22463,4605)
\blacken\path(22362.532,4247.780)(22463.000,4605.000)(22206.248,4337.085)(22362.532,4247.780)
\path(24863,3030)(23588,4605)
\blacken\path(23884.463,4381.820)(23588.000,4605.000)(23744.559,4268.564)(23884.463,4381.820)
\path(19088,9480)(16538,11355)
\blacken\path(16881.350,11214.248)(16538.000,11355.000)(16774.719,11069.231)(16881.350,11214.248)
\path(15263,12705)(13238,14430)
\blacken\path(13570.410,14265.064)(13238.000,14430.000)(13453.686,14128.040)(13570.410,14265.064)
\drawline(12038,15780)(12038,15780)
\path(12038,15780)(10313,17055)
\blacken\path(10655.999,16913.395)(10313.000,17055.000)(10549.008,16768.643)(10655.999,16913.395)
\path(3863,15705)(8588,17355)
\blacken\path(8277.798,17151.346)(8588.000,17355.000)(8218.456,17321.282)(8277.798,17151.346)
\path(1238,13230)(2513,14505)
\blacken\path(2322.081,14186.802)(2513.000,14505.000)(2194.802,14314.081)(2322.081,14186.802)
\path(5213,13155)(3938,14355)
\blacken\path(4261.835,14173.807)(3938.000,14355.000)(4138.469,14042.731)(4261.835,14173.807)
\path(9788,13080)(11663,14505)
\blacken\path(11430.839,14215.515)(11663.000,14505.000)(11321.924,14358.825)(11430.839,14215.515)
\path(12338,6780)(11213,8055)
\blacken\path(11518.669,7844.604)(11213.000,8055.000)(11383.698,7725.512)(11518.669,7844.604)
\path(8438,6930)(9713,8130)
\blacken\path(9512.531,7817.731)(9713.000,8130.000)(9389.165,7948.807)(9512.531,7817.731)
\path(11138,9405)(14888,11580)
\blacken\path(14621.743,11321.529)(14888.000,11580.000)(14531.434,11477.234)(14621.743,11321.529)
\path(41513,9105)(43088,10680)
\blacken\path(42897.081,10361.802)(43088.000,10680.000)(42769.802,10489.081)(42897.081,10361.802)
\path(46763,8805)(44588,10680)
\blacken\path(44919.432,10513.109)(44588.000,10680.000)(44801.903,10376.775)(44919.432,10513.109)
\path(43313,12105)(40763,13980)
\blacken\path(41106.350,13839.248)(40763.000,13980.000)(40999.719,13694.231)(41106.350,13839.248)
\path(39488,15330)(37538,16980)
\blacken\path(37870.954,16816.166)(37538.000,16980.000)(37754.684,16678.756)(37870.954,16816.166)
\drawline(36263,18405)(36263,18405)
\path(33713,15330)(35813,17130)
\blacken\path(35598.239,16827.382)(35813.000,17130.000)(35481.096,16964.048)(35598.239,16827.382)
\path(36563,9405)(35438,10680)
\blacken\path(35743.669,10469.604)(35438.000,10680.000)(35608.698,10350.512)(35743.669,10469.604)
\path(32663,9555)(33863,10755)
\blacken\path(33672.081,10436.802)(33863.000,10755.000)(33544.802,10564.081)(33672.081,10436.802)
\path(35369,11869)(39113,14055)
\blacken\path(38847.491,13795.761)(39113.000,14055.000)(38756.733,13951.205)(38847.491,13795.761)
\thinlines
\path(2438,15180)(4088,15180)
\path(38,12480)(1688,12480)
\path(4763,12480)(6413,12480)
\path(8363,12480)(10013,12480)
\path(11663,15180)(13313,15180)
\path(14963,12180)(16613,12180)
\path(9638,8880)(11288,8880)
\path(7238,6180)(8888,6180)
\path(11963,6180)(13613,6180)
\path(15863,5880)(17513,5880)
\path(18788,8880)(20438,8880)
\path(22163,5580)(23813,5580)
\path(20363,2280)(22013,2280)
\path(24263,2280)(25913,2280)
\path(35888,17805)(37538,17805)
\path(32288,14730)(33938,14730)
\path(39188,14805)(40838,14805)
\path(33863,11505)(35513,11505)
\path(31463,8805)(33113,8805)
\path(36188,8805)(37838,8805)
\path(40088,8505)(41738,8505)
\path(43013,11505)(44663,11505)
\path(46388,8205)(48038,8205)
\put(863,10230){\makebox(0,0)[b]{\smash{{{\SetFigFont{5}{6.0}{rm}1}}}}}
\put(5588,10305){\makebox(0,0)[b]{\smash{{{\SetFigFont{5}{6.0}{rm}2}}}}}
\put(9188,10305){\makebox(0,0)[b]{\smash{{{\SetFigFont{5}{6.0}{rm}3}}}}}
\put(7988,3930){\makebox(0,0)[b]{\smash{{{\SetFigFont{5}{6.0}{rm}4}}}}}
\put(12788,3930){\makebox(0,0)[b]{\smash{{{\SetFigFont{5}{6.0}{rm}5}}}}}
\put(16688,3705){\makebox(0,0)[b]{\smash{{{\SetFigFont{5}{6.0}{rm}6}}}}}
\put(21263,105){\makebox(0,0)[b]{\smash{{{\SetFigFont{5}{6.0}{rm}7}}}}}
\put(25088,105){\makebox(0,0)[b]{\smash{{{\SetFigFont{5}{6.0}{rm}8}}}}}
\put(33113,12405){\makebox(0,0)[b]{\smash{{{\SetFigFont{5}{6.0}{rm}2}}}}}
\put(32213,6630){\makebox(0,0)[b]{\smash{{{\SetFigFont{5}{6.0}{rm}4}}}}}
\put(37013,6705){\makebox(0,0)[b]{\smash{{{\SetFigFont{5}{6.0}{rm}5}}}}}
\put(40913,6105){\makebox(0,0)[b]{\smash{{{\SetFigFont{5}{6.0}{rm}6}}}}}
\put(47288,5955){\makebox(0,0)[b]{\smash{{{\SetFigFont{5}{6.0}{rm}8}}}}}
\put(28763,17955){\makebox(0,0)[lb]{\smash{{{\SetFigFont{5}{6.0}{rm}$T'$}}}}}
\put(4763,17955){\makebox(0,0)[lb]{\smash{{{\SetFigFont{5}{6.0}{rm}$T$}}}}}
\end{picture}
}

%% file: shuntw.eepic.tex
\setlength{\unitlength}{0.00012500in}
\begingroup\makeatletter\ifx\SetFigFont\undefined
\def\x#1#2#3#4#5#6#7\relax{\def\x{#1#2#3#4#5#6}}%
\expandafter\x\fmtname xxxxxx\relax \def\y{splain}%
\ifx\x\y   
\gdef\SetFigFont#1#2#3{%
  \ifnum #1<17\tiny\else \ifnum #1<20\small\else
  \ifnum #1<24\normalsize\else \ifnum #1<29\large\else
  \ifnum #1<34\Large\else \ifnum #1<41\LARGE\else
     \huge\fi\fi\fi\fi\fi\fi
  \csname #3\endcsname}%
\else
\gdef\SetFigFont#1#2#3{\begingroup
  \count@#1\relax \ifnum 25<\count@\count@25\fi
  \def\x{\endgroup\@setsize\SetFigFont{#2pt}}%
  \expandafter\x
    \csname \romannumeral\the\count@ pt\expandafter\endcsname
    \csname @\romannumeral\the\count@ pt\endcsname
  \csname #3\endcsname}%
\fi
\fi\endgroup
{\renewcommand{\dashlinestretch}{30}
\begin{picture}(24536,13320)(0,-10)
\put(20138,8793){\ellipse{1710}{1710}}
\path(19313,8868)(20963,8868)
\dashline{4.500}(22104,7463)(20963,8343)
\blacken\thicklines
\path(21189.687,8243.939)(20963.000,8343.000)(21116.401,8148.917)(21189.687,8243.939)
\put(20138,8343){\makebox(0,0)[b]{\smash{{{\SetFigFont{5}{6.0}{rm}$R$}}}}}
\put(20146,9108){\makebox(0,0)[b]{\smash{{{\SetFigFont{5}{6.0}{rm}$I$}}}}}
\thinlines
\put(6038,11193){\ellipse{1710}{1710}}
\put(5963,6093){\ellipse{1710}{1710}}
\put(863,6093){\ellipse{1710}{1710}}
\put(3413,8493){\ellipse{1710}{1710}}
\put(22538,10893){\ellipse{1710}{1710}}
\texture{55888888 88555555 5522a222 a2555555 55888888 88555555 552a2a2a 2a555555 
	55888888 88555555 55a222a2 22555555 55888888 88555555 552a2a2a 2a555555 
	55888888 88555555 5522a222 a2555555 55888888 88555555 552a2a2a 2a555555 
	55888888 88555555 55a222a2 22555555 55888888 88555555 552a2a2a 2a555555 }
\put(15113,4518){\shade\ellipse{1710}{1710}}
\put(15113,4518){\ellipse{1710}{1710}}
\put(17663,6918){\shade\ellipse{1710}{1710}}
\put(17663,6918){\ellipse{1710}{1710}}
\path(5513,6843)(4088,7893)
\blacken\thicklines
\path(4316.805,7798.936)(4088.000,7893.000)(4245.621,7702.329)(4316.805,7798.936)
\thinlines
\path(3788,9243)(5363,10518)
\blacken\thicklines
\path(5214.213,10320.358)(5363.000,10518.000)(5138.709,10413.627)(5214.213,10320.358)
\thinlines
\path(5213,11268)(6863,11268)
\path(5138,6168)(6788,6168)
\dashline{4.500}(7913,9618)(6788,10518)
\blacken\thicklines
\path(7012.890,10414.925)(6788.000,10518.000)(6937.927,10321.221)(7012.890,10414.925)
\thinlines
\dashline{4.500}(7913,4593)(6713,5568)
\blacken\thicklines
\path(6937.103,5463.225)(6713.000,5568.000)(6861.432,5370.091)(6937.103,5463.225)
\thinlines
\dashline{4.500}(6038,12093)(6038,13293)
\blacken\thicklines
\path(6098.000,13053.000)(6038.000,13293.000)(5978.000,13053.000)(6098.000,13053.000)
\thinlines
\path(38,6168)(1688,6168)
\path(2588,8568)(4238,8568)
\path(1463,6768)(2663,7818)
\blacken\thicklines
\path(2521.892,7614.804)(2663.000,7818.000)(2442.871,7705.113)(2521.892,7614.804)
\thinlines
\path(21713,10968)(23363,10968)
\dashline{4.500}(22538,11793)(22538,12993)
\blacken\thicklines
\path(22598.000,12753.000)(22538.000,12993.000)(22478.000,12753.000)(22598.000,12753.000)
\thinlines
\dashline{4.500}(24524,9397)(23324,10372)
\blacken\thicklines
\path(23548.103,10267.225)(23324.000,10372.000)(23472.432,10174.091)(23548.103,10267.225)
\thinlines
\path(20588,9543)(21713,10443)
\blacken\thicklines
\path(21563.073,10246.221)(21713.000,10443.000)(21488.110,10339.925)(21563.073,10246.221)
\thinlines
\path(18275,7503)(19313,8343)
\whiten\thicklines
\path(19164.180,8145.382)(19313.000,8343.000)(19088.692,8238.664)(19182.405,8237.316)(19164.180,8145.382)
\thinlines
\path(14288,4593)(15938,4593)
\path(16838,6993)(18488,6993)
\path(15713,5193)(16913,6243)
\whiten\thicklines
\path(16771.892,6039.804)(16913.000,6243.000)(16692.871,6130.113)(16786.567,6132.371)(16771.892,6039.804)
\thinlines
\dashline{4.500}(4163,4593)(5213,5493)
\blacken\thicklines
\path(5069.826,5291.255)(5213.000,5493.000)(4991.731,5382.365)(5069.826,5291.255)
\thinlines
\dashline{4.500}(21563,6993)(20663,8043)
\blacken\thicklines
\path(20864.745,7899.826)(20663.000,8043.000)(20773.635,7821.731)(20864.745,7899.826)
\thinlines
\dashline{4.500}(1538,4143)(863,5193)
\whiten\thicklines
\path(1043.253,5023.563)(863.000,5193.000)(942.311,4958.672)(953.847,5051.682)(1043.253,5023.563)
\thinlines
\dashline{4.500}(1088,4443)(488,4443)
\dashline{4.500}(38,4143)(788,5193)
\whiten\thicklines
\path(697.327,4962.830)(788.000,5193.000)(599.679,5032.578)(690.352,5056.293)(697.327,4962.830)
\thinlines
\dashline{4.500}(6713,4143)(6038,5193)
\whiten\thicklines
\path(6218.253,5023.563)(6038.000,5193.000)(6117.311,4958.672)(6128.847,5051.682)(6218.253,5023.563)
\thinlines
\dashline{4.500}(6263,4443)(5663,4443)
\dashline{4.500}(5213,4143)(5963,5193)
\whiten\thicklines
\path(5872.327,4962.830)(5963.000,5193.000)(5774.679,5032.578)(5865.352,5056.293)(5872.327,4962.830)
\thinlines
\dashline{4.500}(6788,9168)(6113,10218)
\whiten\thicklines
\path(6293.253,10048.563)(6113.000,10218.000)(6192.311,9983.672)(6203.847,10076.682)(6293.253,10048.563)
\thinlines
\dashline{4.500}(6338,9468)(5738,9468)
\dashline{4.500}(5288,9168)(6038,10218)
\whiten\thicklines
\path(5947.327,9987.830)(6038.000,10218.000)(5849.679,10057.578)(5940.352,10081.293)(5947.327,9987.830)
\thinlines
\dashline{4.500}(4088,6468)(3413,7518)
\whiten\thicklines
\path(3593.253,7348.563)(3413.000,7518.000)(3492.311,7283.672)(3503.847,7376.682)(3593.253,7348.563)
\thinlines
\dashline{4.500}(3638,6768)(3038,6768)
\dashline{4.500}(2588,6468)(3338,7518)
\whiten\thicklines
\path(3247.327,7287.830)(3338.000,7518.000)(3149.679,7357.578)(3240.352,7381.293)(3247.327,7287.830)
\thinlines
\dashline{4.500}(15863,2568)(15188,3618)
\whiten\thicklines
\path(15368.253,3448.563)(15188.000,3618.000)(15267.311,3383.672)(15278.847,3476.682)(15368.253,3448.563)
\thinlines
\dashline{4.500}(15413,2868)(14813,2868)
\dashline{4.500}(14363,2568)(15113,3618)
\whiten\thicklines
\path(15022.327,3387.830)(15113.000,3618.000)(14924.679,3457.578)(15015.352,3481.293)(15022.327,3387.830)
\thinlines
\dashline{4.500}(20738,6843)(20063,7893)
\whiten\thicklines
\path(20243.253,7723.563)(20063.000,7893.000)(20142.311,7658.672)(20153.847,7751.682)(20243.253,7723.563)
\thinlines
\dashline{4.500}(20288,7143)(19688,7143)
\dashline{4.500}(19238,6843)(19988,7893)
\whiten\thicklines
\path(19897.327,7662.830)(19988.000,7893.000)(19799.679,7732.578)(19890.352,7756.293)(19897.327,7662.830)
\thinlines
\dashline{4.500}(18413,4968)(17738,6018)
\whiten\thicklines
\path(17918.253,5848.563)(17738.000,6018.000)(17817.311,5783.672)(17828.847,5876.682)(17918.253,5848.563)
\thinlines
\dashline{4.500}(17963,5268)(17363,5268)
\dashline{4.500}(16913,4968)(17663,6018)
\whiten\thicklines
\path(17572.327,5787.830)(17663.000,6018.000)(17474.679,5857.578)(17565.352,5881.293)(17572.327,5787.830)
\thinlines
\dashline{4.500}(23288,8943)(22613,9993)
\whiten\thicklines
\path(22793.253,9823.563)(22613.000,9993.000)(22692.311,9758.672)(22703.847,9851.682)(22793.253,9823.563)
\thinlines
\dashline{4.500}(22838,9243)(22238,9243)
\dashline{4.500}(21788,8943)(22538,9993)
\whiten\thicklines
\path(22447.327,9762.830)(22538.000,9993.000)(22349.679,9832.578)(22440.352,9856.293)(22447.327,9762.830)
\thinlines
\path(7163,2793)(8888,2793)
\blacken\thicklines
\path(8648.000,2733.000)(8888.000,2793.000)(8648.000,2853.000)(8648.000,2733.000)
\thinlines
\path(7163,993)(8888,993)
\whiten\thicklines
\path(8648.000,933.000)(8888.000,993.000)(8648.000,1053.000)(8720.000,993.000)(8648.000,933.000)
\put(5963,6393){\makebox(0,0)[b]{\smash{{{\SetFigFont{5}{6.0}{rm}$I$}}}}}
\put(5963,5568){\makebox(0,0)[b]{\smash{{{\SetFigFont{5}{6.0}{rm}$R$}}}}}
\put(6038,11568){\makebox(0,0)[b]{\smash{{{\SetFigFont{5}{6.0}{rm}$I$}}}}}
\put(6038,10743){\makebox(0,0)[b]{\smash{{{\SetFigFont{5}{6.0}{rm}$R$}}}}}
\put(1613,11643){\makebox(0,0)[lb]{\smash{{{\SetFigFont{8}{9.6}{rm}$T$}}}}}
\put(863,5568){\makebox(0,0)[b]{\smash{{{\SetFigFont{5}{6.0}{rm}$R$}}}}}
\put(3413,7968){\makebox(0,0)[b]{\smash{{{\SetFigFont{5}{6.0}{rm}$R$}}}}}
\put(863,6393){\makebox(0,0)[b]{\smash{{{\SetFigFont{5}{6.0}{rm}$I$}}}}}
\put(3413,8793){\makebox(0,0)[b]{\smash{{{\SetFigFont{5}{6.0}{rm}$I$}}}}}
\put(22538,11193){\makebox(0,0)[b]{\smash{{{\SetFigFont{5}{6.0}{rm}$I$}}}}}
\put(22538,10368){\makebox(0,0)[b]{\smash{{{\SetFigFont{5}{6.0}{rm}$R$}}}}}
\put(18263,11343){\makebox(0,0)[lb]{\smash{{{\SetFigFont{8}{9.6}{rm}$T'$}}}}}
\put(15113,3993){\makebox(0,0)[b]{\smash{{{\SetFigFont{5}{6.0}{rm}$R$}}}}}
\put(17663,6393){\makebox(0,0)[b]{\smash{{{\SetFigFont{5}{6.0}{rm}$R$}}}}}
\put(15113,4818){\makebox(0,0)[b]{\smash{{{\SetFigFont{5}{6.0}{rm}$I$}}}}}
\put(17663,7218){\makebox(0,0)[b]{\smash{{{\SetFigFont{5}{6.0}{rm}$I$}}}}}
\put(16238,6918){\makebox(0,0)[b]{\smash{{{\SetFigFont{5}{6.0}{rm}$w$}}}}}
\put(13613,4443){\makebox(0,0)[b]{\smash{{{\SetFigFont{5}{6.0}{rm}$w$}}}}}
\put(7088,1893){\makebox(0,0)[lb]{\smash{{{\SetFigFont{5}{6.0}{rm}Edge from an unmarked vertex}}}}}
\put(7163,93){\makebox(0,0)[lb]{\smash{{{\SetFigFont{5}{6.0}{rm}Edge from a marked vertex}}}}}
\put(15113,5943){\makebox(0,0)[b]{\smash{{{\SetFigFont{7}{8.4}{rm}$\bl$}}}}}
\put(17663,8268){\makebox(0,0)[b]{\smash{{{\SetFigFont{7}{8.4}{rm}$\bp$}}}}}
\put(20138,10143){\makebox(0,0)[b]{\smash{{{\SetFigFont{7}{8.4}{rm}$\bs$}}}}}
\put(21788,12243){\makebox(0,0)[b]{\smash{{{\SetFigFont{7}{8.4}{rm}$\bq$}}}}}
\put(3413,9843){\makebox(0,0)[b]{\smash{{{\SetFigFont{7}{8.4}{rm}$\bp$}}}}}
\put(5438,12468){\makebox(0,0)[b]{\smash{{{\SetFigFont{7}{8.4}{rm}$\bq$}}}}}
\put(5963,7518){\makebox(0,0)[b]{\smash{{{\SetFigFont{7}{8.4}{rm}$\bs$}}}}}
\put(863,7518){\makebox(0,0)[b]{\smash{{{\SetFigFont{7}{8.4}{rm}$\bl$}}}}}
\end{picture}
}

%% file: rshuntw.eepic.tex
\setlength{\unitlength}{0.00012500in}
\begingroup\makeatletter\ifx\SetFigFont\undefined
\def\x#1#2#3#4#5#6#7\relax{\def\x{#1#2#3#4#5#6}}%
\expandafter\x\fmtname xxxxxx\relax \def\y{splain}%
\ifx\x\y   
\gdef\SetFigFont#1#2#3{%
  \ifnum #1<17\tiny\else \ifnum #1<20\small\else
  \ifnum #1<24\normalsize\else \ifnum #1<29\large\else
  \ifnum #1<34\Large\else \ifnum #1<41\LARGE\else
     \huge\fi\fi\fi\fi\fi\fi
  \csname #3\endcsname}%
\else
\gdef\SetFigFont#1#2#3{\begingroup
  \count@#1\relax \ifnum 25<\count@\count@25\fi
  \def\x{\endgroup\@setsize\SetFigFont{#2pt}}%
  \expandafter\x
    \csname \romannumeral\the\count@ pt\expandafter\endcsname
    \csname @\romannumeral\the\count@ pt\endcsname
  \csname #3\endcsname}%
\fi
\fi\endgroup
{\renewcommand{\dashlinestretch}{30}
\begin{picture}(25216,12345)(0,-10)
\texture{55888888 88555555 5522a222 a2555555 55888888 88555555 552a2a2a 2a555555 
	55888888 88555555 55a222a2 22555555 55888888 88555555 552a2a2a 2a555555 
	55888888 88555555 5522a222 a2555555 55888888 88555555 552a2a2a 2a555555 
	55888888 88555555 55a222a2 22555555 55888888 88555555 552a2a2a 2a555555 }
\put(1954,3843){\shade\ellipse{1710}{1710}}
\put(1954,3843){\ellipse{1710}{1710}}
\put(4504,6243){\shade\ellipse{1710}{1710}}
\put(4504,6243){\ellipse{1710}{1710}}
\put(23329,10218){\ellipse{1710}{1710}}
\put(23254,5118){\ellipse{1710}{1710}}
\put(18154,5118){\ellipse{1710}{1710}}
\put(20704,7518){\ellipse{1710}{1710}}
\put(9379,10218){\ellipse{1710}{1710}}
\put(6979,8118){\ellipse{1710}{1710}}
\dashline{4.500}(9379,11118)(9379,12318)
\blacken\thicklines
\path(9439.000,12078.000)(9379.000,12318.000)(9319.000,12078.000)(9439.000,12078.000)
\thinlines
\dashline{4.500}(11365,8722)(10165,9697)
\blacken\thicklines
\path(10389.103,9592.225)(10165.000,9697.000)(10313.432,9499.091)(10389.103,9592.225)
\thinlines
\path(7429,8868)(8554,9768)
\blacken\thicklines
\path(8404.073,9571.221)(8554.000,9768.000)(8329.110,9664.925)(8404.073,9571.221)
\thinlines
\path(5116,6828)(6154,7668)
\whiten\thicklines
\path(6005.180,7470.382)(6154.000,7668.000)(5929.692,7563.664)(6023.405,7562.316)(6005.180,7470.382)
\thinlines
\path(1129,3918)(2779,3918)
\path(3679,6318)(5329,6318)
\path(2554,4518)(3754,5568)
\whiten\thicklines
\path(3612.892,5364.804)(3754.000,5568.000)(3533.871,5455.113)(3627.567,5457.371)(3612.892,5364.804)
\thinlines
\path(22804,5868)(21379,6918)
\blacken\thicklines
\path(21607.805,6823.936)(21379.000,6918.000)(21536.621,6727.329)(21607.805,6823.936)
\thinlines
\path(21079,8268)(22654,9543)
\blacken\thicklines
\path(22505.213,9345.358)(22654.000,9543.000)(22429.709,9438.627)(22505.213,9345.358)
\thinlines
\dashline{4.500}(25204,8643)(24079,9543)
\blacken\thicklines
\path(24303.890,9439.925)(24079.000,9543.000)(24228.927,9346.221)(24303.890,9439.925)
\thinlines
\dashline{4.500}(25204,3618)(24004,4593)
\blacken\thicklines
\path(24228.103,4488.225)(24004.000,4593.000)(24152.432,4395.091)(24228.103,4488.225)
\thinlines
\dashline{4.500}(23329,11118)(23329,12318)
\blacken\thicklines
\path(23389.000,12078.000)(23329.000,12318.000)(23269.000,12078.000)(23389.000,12078.000)
\thinlines
\dashline{4.500}(21454,3618)(22504,4518)
\blacken\thicklines
\path(22360.826,4316.255)(22504.000,4518.000)(22282.731,4407.365)(22360.826,4316.255)
\thinlines
\path(18754,5793)(19954,6843)
\blacken\thicklines
\path(19812.892,6639.804)(19954.000,6843.000)(19733.871,6730.113)(19812.892,6639.804)
\thinlines
\dashline{4.500}(8945,6788)(7804,7668)
\blacken\thicklines
\path(8030.687,7568.939)(7804.000,7668.000)(7957.401,7473.917)(8030.687,7568.939)
\thinlines
\dashline{4.500}(8614,6450)(7489,7350)
\blacken\thicklines
\path(7713.890,7246.925)(7489.000,7350.000)(7638.927,7153.221)(7713.890,7246.925)
\thinlines
\dashline{4.500}(24079,8268)(23404,9318)
\whiten\thicklines
\path(23584.253,9148.563)(23404.000,9318.000)(23483.311,9083.672)(23494.847,9176.682)(23584.253,9148.563)
\thinlines
\dashline{4.500}(23629,8568)(23029,8568)
\dashline{4.500}(22579,8268)(23329,9318)
\whiten\thicklines
\path(23238.327,9087.830)(23329.000,9318.000)(23140.679,9157.578)(23231.352,9181.293)(23238.327,9087.830)
\thinlines
\dashline{4.500}(21454,5568)(20779,6618)
\whiten\thicklines
\path(20959.253,6448.563)(20779.000,6618.000)(20858.311,6383.672)(20869.847,6476.682)(20959.253,6448.563)
\thinlines
\dashline{4.500}(21004,5868)(20404,5868)
\dashline{4.500}(19954,5568)(20704,6618)
\whiten\thicklines
\path(20613.327,6387.830)(20704.000,6618.000)(20515.679,6457.578)(20606.352,6481.293)(20613.327,6387.830)
\thinlines
\dashline{4.500}(24004,3168)(23329,4218)
\whiten\thicklines
\path(23509.253,4048.563)(23329.000,4218.000)(23408.311,3983.672)(23419.847,4076.682)(23509.253,4048.563)
\thinlines
\dashline{4.500}(23554,3468)(22954,3468)
\dashline{4.500}(22504,3168)(23254,4218)
\whiten\thicklines
\path(23163.327,3987.830)(23254.000,4218.000)(23065.679,4057.578)(23156.352,4081.293)(23163.327,3987.830)
\thinlines
\dashline{4.500}(18904,3168)(18229,4218)
\whiten\thicklines
\path(18409.253,4048.563)(18229.000,4218.000)(18308.311,3983.672)(18319.847,4076.682)(18409.253,4048.563)
\thinlines
\dashline{4.500}(18454,3468)(17854,3468)
\dashline{4.500}(17404,3168)(18154,4218)
\whiten\thicklines
\path(18063.327,3987.830)(18154.000,4218.000)(17965.679,4057.578)(18056.352,4081.293)(18063.327,3987.830)
\thinlines
\dashline{4.500}(10129,8193)(9454,9243)
\whiten\thicklines
\path(9634.253,9073.563)(9454.000,9243.000)(9533.311,9008.672)(9544.847,9101.682)(9634.253,9073.563)
\thinlines
\dashline{4.500}(9679,8493)(9079,8493)
\dashline{4.500}(8629,8193)(9379,9243)
\whiten\thicklines
\path(9288.327,9012.830)(9379.000,9243.000)(9190.679,9082.578)(9281.352,9106.293)(9288.327,9012.830)
\thinlines
\dashline{4.500}(7654,6168)(6979,7218)
\whiten\thicklines
\path(7159.253,7048.563)(6979.000,7218.000)(7058.311,6983.672)(7069.847,7076.682)(7159.253,7048.563)
\thinlines
\dashline{4.500}(7204,6468)(6604,6468)
\dashline{4.500}(6154,6168)(6904,7218)
\whiten\thicklines
\path(6813.327,6987.830)(6904.000,7218.000)(6715.679,7057.578)(6806.352,7081.293)(6813.327,6987.830)
\thinlines
\dashline{4.500}(5254,4293)(4579,5343)
\whiten\thicklines
\path(4759.253,5173.563)(4579.000,5343.000)(4658.311,5108.672)(4669.847,5201.682)(4759.253,5173.563)
\thinlines
\dashline{4.500}(4804,4593)(4204,4593)
\dashline{4.500}(3754,4293)(4504,5343)
\whiten\thicklines
\path(4413.327,5112.830)(4504.000,5343.000)(4315.679,5182.578)(4406.352,5206.293)(4413.327,5112.830)
\thinlines
\dashline{4.500}(2629,1818)(1954,2868)
\whiten\thicklines
\path(2134.253,2698.563)(1954.000,2868.000)(2033.311,2633.672)(2044.847,2726.682)(2134.253,2698.563)
\thinlines
\dashline{4.500}(2179,2118)(1579,2118)
\dashline{4.500}(1129,1818)(1879,2868)
\whiten\thicklines
\path(1788.327,2637.830)(1879.000,2868.000)(1690.679,2707.578)(1781.352,2731.293)(1788.327,2637.830)
\thinlines
\path(9754,2793)(11479,2793)
\blacken\thicklines
\path(11239.000,2733.000)(11479.000,2793.000)(11239.000,2853.000)(11239.000,2733.000)
\thinlines
\path(9754,993)(11479,993)
\whiten\thicklines
\path(11239.000,933.000)(11479.000,993.000)(11239.000,1053.000)(11311.000,993.000)(11239.000,933.000)
\put(1954,3318){\makebox(0,0)[b]{\smash{{{\SetFigFont{5}{6.0}{rm}$R$}}}}}
\put(4504,5718){\makebox(0,0)[b]{\smash{{{\SetFigFont{5}{6.0}{rm}$R$}}}}}
\put(1954,4143){\makebox(0,0)[b]{\smash{{{\SetFigFont{5}{6.0}{rm}$I$}}}}}
\put(4504,6543){\makebox(0,0)[b]{\smash{{{\SetFigFont{5}{6.0}{rm}$I$}}}}}
\put(3079,6243){\makebox(0,0)[b]{\smash{{{\SetFigFont{5}{6.0}{rm}$w$}}}}}
\put(454,3768){\makebox(0,0)[b]{\smash{{{\SetFigFont{5}{6.0}{rm}$w$}}}}}
\put(22729,11418){\makebox(0,0)[b]{\smash{{{\SetFigFont{7}{8.4}{rm}$\bq$}}}}}
\put(18904,10668){\makebox(0,0)[lb]{\smash{{{\SetFigFont{8}{9.6}{rm}$T'$}}}}}
\put(9379,10143){\makebox(0,0)[b]{\smash{{{\SetFigFont{5}{6.0}{rm}$R$}}}}}
\put(6979,8043){\makebox(0,0)[b]{\smash{{{\SetFigFont{5}{6.0}{rm}$R$}}}}}
\put(23329,10143){\makebox(0,0)[b]{\smash{{{\SetFigFont{5}{6.0}{rm}$R$}}}}}
\put(20704,7443){\makebox(0,0)[b]{\smash{{{\SetFigFont{5}{6.0}{rm}$R$}}}}}
\put(18154,5043){\makebox(0,0)[b]{\smash{{{\SetFigFont{5}{6.0}{rm}$R$}}}}}
\put(23254,5043){\makebox(0,0)[b]{\smash{{{\SetFigFont{5}{6.0}{rm}$R$}}}}}
\put(5104,10668){\makebox(0,0)[lb]{\smash{{{\SetFigFont{8}{9.6}{rm}$T$}}}}}
\put(9679,1893){\makebox(0,0)[lb]{\smash{{{\SetFigFont{5}{6.0}{rm}Edge from an unmarked vertex}}}}}
\put(9754,93){\makebox(0,0)[lb]{\smash{{{\SetFigFont{5}{6.0}{rm}Edge from a marked vertex}}}}}
\put(1954,5193){\makebox(0,0)[b]{\smash{{{\SetFigFont{7}{8.4}{rm}$\bl$}}}}}
\put(4504,7518){\makebox(0,0)[b]{\smash{{{\SetFigFont{7}{8.4}{rm}$\bp$}}}}}
\put(6979,9468){\makebox(0,0)[b]{\smash{{{\SetFigFont{7}{8.4}{rm}$\bs$}}}}}
\put(8629,11568){\makebox(0,0)[b]{\smash{{{\SetFigFont{7}{8.4}{rm}$\bq$}}}}}
\put(20704,8943){\makebox(0,0)[b]{\smash{{{\SetFigFont{7}{8.4}{rm}$\bp$}}}}}
\put(18154,6468){\makebox(0,0)[b]{\smash{{{\SetFigFont{7}{8.4}{rm}$\bl$}}}}}
\put(23254,6468){\makebox(0,0)[b]{\smash{{{\SetFigFont{7}{8.4}{rm}$\bs$}}}}}
\end{picture}
}

%% file: TR.bbl
\begin{thebibliography}{10}

\bibitem{AboKhamis:2016:FQA:2902251.2902280}
Mahmoud Abo~Khamis, Hung~Q. Ngo, and Atri Rudra.
\newblock Faq: Questions asked frequently.
\newblock In {\em Proc of {PODS'16}}, pages 13--28, 2016.

\bibitem{AfratiJRSU14}
Foto~N. Afrati, Manas~R. Joglekar, Christopher R{\'{e}}, Semih Salihoglu, and
  Jeffrey~D. Ullman.
\newblock {GYM:} {A} multiround distributed join algorithm.
\newblock In {\em Proc.\ of {ICDT}'17}, pages 4:1--4:18, 2017.

\bibitem{BCVBW02}
F.~Bacchus, X.~Chen, P.~van Beek, and T.~Walsh.
\newblock Binary vs. non-binary constraints.
\newblock {\em Artificial Intelligence}, 140(1-2):1--37, 2002.

\bibitem{BFMY83}
C.~Beeri, R.~Fagin, D.~Maier, and M.~Yannakakis.
\newblock On the desirability of acyclic database schemes.
\newblock {\em Journal of the ACM}, 30(3):479--513, 1983.

\bibitem{BG81}
P.~Bernstein and N.~Goodman.
\newblock Power of natural semijoins.
\newblock {\em SIAM Journal on Computing}, 10(4):751--771, 1981.

\bibitem{BFMRSV96}
S.~Bistarelli, H.~Fargier, U.~Montanari, F.~Rossi, T.~Schiex, and
  G.~Verfaillie.
\newblock Semiring-based {CSP}s and valued {CSP}s: Basic properties and
  comparison.
\newblock In {\em Over-Constrained Systems}, pages 111--150, 1996.

\bibitem{BFMR04}
S.~Bistarelli, T.~Fr{\"u}hwirth, M.~Marte, and F.~Rossi.
\newblock Soft constraint propagation and solving in constraint handling rules.
\newblock {\em Computational Intelligence}, 20(2):287--307, 2004.

\bibitem{BMR97}
S.~Bistarelli, U.~Montanari, and F.~Rossi.
\newblock Semiring-based constraint satisfaction and optimization.
\newblock {\em Journal of the ACM}, 44(2):201--236, 1997.

\bibitem{BRSVW10}
R.~Brafman, F.~Rossi, D.~Salvagnin, K.B. Venable, and T.~Walsh.
\newblock Finding the next solution in constraint- and preference-based
  knowledge representation formalisms.
\newblock In {\em Proc.\ of KR'10}, pages 425--433, 2010.

\bibitem{BDGM12}
A.A. Bulatov, V.~Dalmau, M.~Grohe, and D.~Marx.
\newblock Enumerating homomorphisms.
\newblock {\em Journal of Computer and System Sciences}, 78(2):638--650, 2012.

\bibitem{CD05}
H.~Chen and V.~Dalmau.
\newblock Beyond hypertree width: Decomposition methods without decompositions.
\newblock In {\em Proc.\ of CP'05}, pages 167--181, 2005.

\bibitem{CJG08}
D.~Cohen, P.~Jeavons, and M.~Gyssens.
\newblock A unified theory of structural tractability for constraint
  satisfaction problems.
\newblock {\em Journal of Computer and System Sciences}, 74(5):721--743, 2008.

\bibitem{CS04}
M.~Cooper and T.~Schiex.
\newblock Arc consistency for soft constraints.
\newblock {\em Artificial Intelligence}, 154(1-2):199--227, 2004.

\bibitem{C05}
M.C. Cooper.
\newblock High-order consistency in valued constraint satisfaction.
\newblock {\em Constraints}, 10(3):283--305, 2005.

\bibitem{CGSSZW10}
M.C. Cooper, S.~de~Givry, M.~Sanchez, T.~Schiex, M.~Zytnicki, and T.~Werner.
\newblock Soft arc consistency revisited.
\newblock {\em Artificial Intelligence}, 174(7-8):449--478, 2010.

\bibitem{D03}
R.~Dechter.
\newblock {\em Constraint Processing}.
\newblock Morgan Kaufmann Publishers Inc., San Francisco, CA, USA, 2003.

\bibitem{DP89}
R.~Dechter and J.~Pearl.
\newblock Tree clustering for constraint networks.
\newblock {\em Artificial Intelligence}, 38(3):353--366, 1989.

\bibitem{down-fell-99}
Rodney~G. Downey and M.~R. Fellows.
\newblock {\em Parameterized Complexity}.
\newblock Springer Publishing Company, Incorporated, 2012.

\bibitem{DFP93}
D.~Dubois, H.~Fargier, and H.~Prade.
\newblock The calculus of fuzzy restrictions as a basis for flexible constraint
  satisfaction.
\newblock In {\em Proc. of FUZZ-IEEE'93}, pages 1131--1136, 1993.

\bibitem{F83}
R.~Fagin.
\newblock Degrees of acyclicity for hypergraphs and relational database
  schemes.
\newblock {\em Journal of the ACM}, 30(3):514--550, 1983.

\bibitem{FMU82}
R.~Fagin, A.O. Mendelzon, and J.D. Ullman.
\newblock A simplied universal relation assumption and its properties.
\newblock {\em ACM Transactions on Database Systems}, 7(3):343--360, 1982.

\bibitem{FL93}
H.~Fargier and J.~Lang.
\newblock Uncertainty in constraint satisfaction problems: A probabilistic
  approach.
\newblock In {\em Proc. of ECSQARU'93}, pages 97--104, 1993.

\bibitem{FLS93}
H.~Fargier, J.~Lang, and T.~Schiex.
\newblock Selecting preferred solutions in fuzzy constraint satisfaction
  problems.
\newblock In {\em Proc.\ of EUFIT'93}, 1993.

\bibitem{DBLP:journals/corr/FischlGP16}
Wolfgang Fischl, Georg Gottlob, and Reinhard Pichler.
\newblock General and fractional hypertree decompositions: Hard and easy cases.
\newblock {\em CoRR}, abs/1611.01090 - short version to appear in Proc. of PODS
  2018, 2016.

\bibitem{FD10}
Natalia Flerova, Radu Marinescu, and Rina Dechter.
\newblock Searching for the {M} best solutions in graphical models.
\newblock {\em Journal of Artificial Intelligence Research}, 55:889--952, 2016.

\bibitem{FFG02}
J.~Flum, M.~Frick, and M.~Grohe.
\newblock Query evaluation via tree-decompositions.
\newblock {\em Journal of the ACM}, 49(6):716--752, 2002.

\bibitem{FW92}
E.C. Freuder and R.J. Wallace.
\newblock Partial constraint satisfaction.
\newblock {\em Artificial Intelligence}, 58:21--70, 1992.

\bibitem{Gent2011}
Ian~P. Gent, Chris Jefferson, Ian Miguel, Neil~C.A. Moore, Peter Nightingale,
  Patrick Prosser, and Chris Unsworth.
\newblock A preliminary review of literature on parallel constraint solving.
\newblock In {\em Proc.\ of PMCS'11}, 2011.

\bibitem{GS84}
N.~Goodman and O.~Shmueli.
\newblock The tree projection theorem and relational query processing.
\newblock {\em Journal of Computer and System Sciences}, 28(1):60--79, 1984.

\bibitem{GG13}
G.~Gottlob and G.~Greco.
\newblock Decomposing combinatorial auctions and set packing problems.
\newblock {\em Journal of the ACM}, 60(4):1--39, 2013.

\bibitem{GLS00}
G.~Gottlob, N.~Leone, and F.~Scarcello.
\newblock A comparison of structural csp decomposition methods.
\newblock {\em Artificial Intelligence}, 124(2):243--282, 2000.

\bibitem{gott-etal-01}
G.~Gottlob, N.~Leone, and F.~Scarcello.
\newblock {The complexity of acyclic conjunctive queries}.
\newblock {\em Journal of the ACM}, 48(3):431--498, 2001.

\bibitem{gott-etal-02}
G.~Gottlob, N.~Leone, and F.~Scarcello.
\newblock {Computing LOGCFL certificates}.
\newblock {\em Theoretical Computer Science}, 270(1-2):761--777, 2002.

\bibitem{GLS02}
G.~Gottlob, N.~Leone, and F.~Scarcello.
\newblock Hypertree decompositions and tractable queries.
\newblock {\em Journal of Computer and System Sciences}, 64(3):579--627, 2002.

\bibitem{GMS09}
G.~Gottlob, Z.~Mikl\'{o}s, and T.~Schwentick.
\newblock Generalized hypertree decompositions: {NP}-hardness and tractable
  variants.
\newblock {\em Journal of the ACM}, 56(6):1--32, 2009.

\bibitem{GGLS16}
Georg Gottlob, Gianluigi Greco, Nicola Leone, and Francesco Scarcello.
\newblock Hypertree decompositions: Questions and answers.
\newblock In {\em Proc of {PODS'16}}, pages 57--74. Association for Computing
  Machinery ({ACM}), 2016.

\bibitem{GGS09}
Georg Gottlob, Gianluigi Greco, and Francesco Scarcello.
\newblock Tractable optimization problems through hypergraph-based structural
  restrictions.
\newblock In {\em Proc. of ICALP'09}, pages 16--30, 2009.

\bibitem{gott-etal-98tr}
Georg Gottlob, Nicola Leone, and Francesco Scarcello.
\newblock Advanced parallel algorithms for acyclic conjunctive queries.
\newblock Technical Report DBAI-TR-98/18, Technical University of Vienna, 1998.

\bibitem{GF10}
G.~Greco and F.~Scarcello.
\newblock On the power of structural decompositions of graph-based
  representations of constraint problems.
\newblock {\em Artificial Intelligence}, 174(5-6):382--409, 2010.

\bibitem{GS11}
G.~Greco and F.~Scarcello.
\newblock Structural tractability of constraint optimization.
\newblock In {\em Proc.\ of CP'11}, pages 340--355, 2011.

\bibitem{GS13}
G.~Greco and F.~Scarcello.
\newblock Structural tractability of enumerating csp solutions.
\newblock {\em Constraints}, 18(1):38--74, 2013.

\bibitem{GS10}
G.~Greco and F.~Scarcello.
\newblock The power of local consistency in conjunctive queries and constraint
  satisfaction problems.
\newblock {\em SIAM Journal on Computing}, pages 1111--1145, 2017.

\bibitem{GrecoS14}
Gianluigi Greco and Francesco Scarcello.
\newblock Tree projections and structural decomposition methods: Minimality and
  game-theoretic characterization.
\newblock {\em Theoretical Computer Science}, 522:95--114, 2014.

\bibitem{GrecoSIC17}
Gianluigi Greco and Francesco Scarcello.
\newblock Greedy strategies and larger islands of tractability for conjunctive
  queries and constraint satisfaction problems.
\newblock {\em Information and Computation}, 252:201--220, 2017.

\bibitem{G07}
M.~Grohe.
\newblock The complexity of homomorphism and constraint satisfaction problems
  seen from the other side.
\newblock {\em Journal of the ACM}, 54(1):1--24, 2007.

\bibitem{GM14}
M.~Grohe and D.~Marx.
\newblock Constraint solving via fractional edge covers.
\newblock {\em ACM Transactions on Algorithms}, 11(1):1--20, 2014.

\bibitem{G86}
Marc Gyssens.
\newblock On the complexity of join dependencies.
\newblock {\em {ACM} Transactions on Database Systems}, 11(1):81--108, 1986.

\bibitem{JT03}
P.~J{\'e}gou and C.~Terrioux.
\newblock Hybrid backtracking bounded by tree-decomposition of constraint
  networks.
\newblock {\em Artificial Intelligence}, 146(1):43--75, 2003.

\bibitem{DBLP:journals/corr/JoglekarPR15}
Manas Joglekar, Rohan Puttagunta, and Christopher R{\'{e}}.
\newblock Aggregations over generalized hypertree decompositions.
\newblock {\em CoRR}, abs/1508.07532, 2015.

\bibitem{KWRCB10}
Shant Karakashian, Robert~J. Woodward, Christopher~G. Reeson, Berthe~Y.
  Choueiry, and Christian Bessiere.
\newblock A first practical algorithm for high levels of relational
  consistency.
\newblock In {\em Proc.\ of AAAI'10}, pages 101--107, 2010.

\bibitem{karp-rama-90}
Richard~M Karp and Vijaya Ramachandran.
\newblock Parallel algorithms for shared-memory machines.
\newblock In Jan van Leeuwen, editor, {\em {Handbook of Theoretical Computer
  Science (Vol. A)}}, pages 869--941. MIT Press, Cambridge, MA, USA, 1990.

\bibitem{kasif1990}
Simon Kasif.
\newblock On the parallel complexity of discrete relaxation in constraint
  satisfaction networks.
\newblock {\em Artificial Intelligence}, 45(3):275--286, oct 1990.

\bibitem{KDLD05}
K.~Kask, R.~Dechter, J.~Larrosa, and A.~Dechter.
\newblock Unifying tree decompositions for reasoning in graphical models.
\newblock {\em Artificial Intelligence}, 166(1-2):165--193, 2005.

\bibitem{KS06}
B.~Kimelfeld and Y.~Sagiv.
\newblock Incrementally computing ordered answers of acyclic conjunctive
  queries.
\newblock In {\em Proc. of NGITS'06}, pages 33--38, 2006.

\bibitem{KV00}
P.G. Kolaitis and M.Y. Vardi.
\newblock Conjunctive-query containment and constraint satisfaction.
\newblock {\em Journal of Computer and System Sciences}, 61(2):302--332, 2000.

\bibitem{LD03}
D.~Larkin and R.~Dechter.
\newblock Bayesian inference in the presence of determinism.
\newblock In {\em Proc.\ of AISTATS'03}, 2003.

\bibitem{LT93}
J.~Larrosa and T.~Schiex.
\newblock In the quest of the best form of local consistency for weighted csp.
\newblock In {\em Proc.\ of IJCAI'03}, pages 239--244, 2003.

\bibitem{LS90}
S.L. Lauritzen and D.J. Spiegelhalter.
\newblock Local computations with probabilities on graphical structures and
  their application to expert systems.
\newblock In {\em Readings in Uncertain Reasoning}, pages 415--448. Morgan
  Kaufmann Publishers Inc., 1990.

\bibitem{L72}
E.L. Lawler.
\newblock A procedure for computing the k best solutions to discrete
  optimization problems and its application to the shortest path problem.
\newblock {\em Management Science}, 18(7):401--405, 1972.

\bibitem{M10}
D.~Marx.
\newblock Approximating fractional hypertree width.
\newblock {\em ACM Transactions on Algorithms}, 6(2):1--17, 2010.

\bibitem{M13}
D\'{a}niel Marx.
\newblock Tractable hypergraph properties for constraint satisfaction and
  conjunctive queries.
\newblock {\em Journal of the ACM}, 60(6):42:1--42:51, 2013.

\bibitem{MONTANARI197495}
Ugo Montanari.
\newblock Networks of constraints: Fundamental properties and applications to
  picture processing.
\newblock {\em Information Sciences}, 7:95--32, 1974.

\bibitem{RS86}
N.~Robertson and P.D. Seymour.
\newblock Graph minors. ii. algorithmic aspects of tree-width.
\newblock {\em Journal of Algorithms}, 7(3):309--322, 1986.

\bibitem{SS93}
Yehoshua Sagiv and Oded Shmueli.
\newblock Solving queries by tree projections.
\newblock {\em ACM Transactions on Database Systems}, 18(3):487--511, 1993.

\bibitem{S92}
T.~Schiex.
\newblock Possibilistic constraint satisfaction problems or ``how to handle
  soft constraints?''.
\newblock In {\em Proc.\ of UAI'92}, pages 268--275, 1992.

\bibitem{SFV95}
T.~Schiex, H.~Fargier, and G.~Verfaillie.
\newblock Valued constraint satisfaction problems: Hard and easy problems.
\newblock In {\em Proc.\ of IJCAI'95}, pages 631--637, 1995.

\bibitem{TJ03}
C.~Terrioux and P.~J{\'e}gou.
\newblock Bounded backtracking for the valued constraint satisfaction problems.
\newblock In {\em Proc.\ of CP'03}, pages 709--723, 2003.

\bibitem{Valiant2011}
Leslie~G. Valiant.
\newblock A bridging model for multi-core computing.
\newblock {\em Journal of Computer and System Sciences}, 77(1):154--166, jan
  2011.

\bibitem{Vishkin2011}
Uzi Vishkin.
\newblock Using simple abstraction to reinvent computing for parallelism.
\newblock {\em Communications of the {ACM}}, 54(1):75, jan 2011.

\bibitem{Y81}
Mihalis Yannakakis.
\newblock Algorithms for acyclic database schemes.
\newblock In {\em Proc.\ of VLDB'81}, pages 82--94, 1981.

\end{thebibliography}
